\documentclass[a4paper]{article}

\usepackage{url}
\usepackage[hidelinks, breaklinks=true]{hyperref}
\usepackage[utf8]{inputenc}
\usepackage[small]{caption}
\usepackage{graphicx}
\usepackage{amsmath}
\usepackage{amsthm}
\usepackage{booktabs}
\usepackage{algorithm}
\usepackage{algorithmic}
\usepackage{breakcites}
\usepackage{todonotes}
\usepackage{complexity}
\newcommand{\SigmaP}[1]{\ComplexityFont{\Sigma}_{#1}^{\P}}
\newcommand{\PiP}[1]{\ComplexityFont{\Pi}_{#1}^{\P}}

\usepackage{amssymb}
\usepackage{url}
\usepackage{caption}
\usepackage{subcaption}
\usepackage{booktabs}
\usepackage{thmtools} 
\usepackage{thm-restate}

\usepackage{tikz,pgf}
\usetikzlibrary{shapes,backgrounds,shapes.geometric,shadows,patterns}
\tikzstyle{arg}=[draw,circle,fill=gray!15,inner sep=1pt,minimum size=.5cm]
\tikzstyle{attack}=[->,left,thick,>=stealth]

\newcommand{\AF}{F}

\newcommand{\cf}{\mathit{cf}}
\newcommand{\naive}{\mathit{naive}}
\newcommand{\stable}{{\mathit{stb}}}
\newcommand{\stb}{\stable}
\newcommand{\adm}{\mathit{adm}}
\newcommand{\pref}{\mathit{prf}}
\newcommand{\stage}{\mathit{stg}}
\newcommand{\semi}{\mathit{sem}}
\newcommand{\comp}{\mathit{com}}

\newcommand{\Args}{A}
\newcommand{\Claims}{\mathit{Claims}}

\newcommand{\wfCAF}{\mathit{wfCAF}}
\newcommand{\CAF}{\mathit{CAF}}

\newcommand{\PCAF}{\mathit{PCAF}}

\newcommand{\Image}[1]{\mathcal{R}_{#1}\text{-}\mathbf{CAF}}
\newcommand{\ImageTrans}[1]{\Image{#1}_{\mathit{tr}}}
\newcommand{\setOfAllwfCAFs}{\mathbf{wfCAF}}
\newcommand{\setOfAllCAFs}{\mathbf{CAF}}
\newcommand{\problematicPart}{\mathit{wfp}}

\newcommand{\U}{U}

\newcommand{\Att}{R} %
\newcommand{\Def}{R'} %
\newcommand{\cl}{\mathit{claim}}

\newcommand{\red}[1]{\mathcal{R}_{{#1}}}

\newcommand{\Cred}{\mathit{Cred}}

\newcommand{\Skept}{\mathit{Skept}}

\newcommand{\Ver}{\mathit{Ver}}

\declaretheorem[name=Definition]{definition}
\declaretheorem[name=Example]{example}
\declaretheorem[numbered=no,name=Remark]{remark}

\declaretheorem[name=Lemma]{lemma}
\declaretheorem[name=Proposition]{proposition}

\newenvironment{proofsketch}{\proof}{\endproof}

\title{The Effect of Preferences in Abstract Argumentation Under a Claim-Centric View}

\author{
	Michael Bernreiter,
	Wolfgang Dvo\v{r}\'{a}k,
	Anna Rapberger, \\
	Stefan Woltran
	\\
	Institute of Logic and Computation, TU Wien, Austria
	\\
	\{mbernrei,dvorak,arapberg,woltran\}@dbai.tuwien.ac.at
}

\begin{document}

\maketitle

\begin{abstract}
  In this paper, we study the effect of preferences in abstract argumentation
  under a claim-centric perspective. Recent work has revealed that semantical
  and computational properties can change when reasoning is performed on 
  claim-level rather than on the argument-level, while under certain 
  natural restrictions (arguments with the same claims have the 
  same outgoing attacks) these properties are conserved. We now investigate
  these effects when, in addition, preferences have to be taken into account and consider four prominent reductions to handle preferences between arguments.
  As we shall see, these reductions give rise to %
  different classes of claim-augmented argumentation frameworks, and behave 
  differently in terms of semantic properties and computational complexity.
  This strengthens the view that the actual choice for handling preferences 
  has to be taken with care.
\end{abstract}

\section{Introduction} \label{sec:introduction}

Arguments vary in their plausibility.
Research in formal argumentation has taken up this %
aspect in both quantitative %
and qualitative terms \cite{AtkinsonB21,LiON11,AtkinsonBGHPRST17}. %
Indeed, preferences 
are nowadays a standard feature of many structured argumentation formalisms \cite{PrakkenS97,ModgilP13,CyrasT16,GarciaS14}.
In abstract argumentation frameworks \cite{Dung95} in which conflicts are expressed as a binary relation between abstract arguments (\textit{attack relation}), the incorporation of preferences typically results in the deletion or reversion of attacks between arguments of different strength---deciding acceptability of arguments via argumentation semantics is thus implicitly reflected in terms of the modified attack relation \cite{KaciTV18}. %

The difference in argument strength and the resulting modification of the attack relation naturally influences the acceptability of the arguments' conclusion (the \textit{claim} of the argument). 
Claim justification in argumentation systems, i.e., the evaluation of commonly acceptable statements while disregarding their particular justifications, is an integral part of many structured argumentation formalisms \cite{ModgilP18,DungKT09}; its relevance is also expressed in the discussion about the advantages and disadvantages of floating conclusions \cite{Horty02}.
Recently, the connection between abstract argumentation and claim-based reasoning has received increasing attention \cite{BaroniR19,DvorakW20}. %
Claim-augmented Argumentation Frameworks (CAFs) \cite{DvorakW20} extend AFs by a function which assigns to each argument a claim, representing its conclusion. CAFs serve as an ideal target formalism for ASPIC+ \cite{ModgilP18} and other structured argumentation formalisms which utilize abstract argumentation semantics whilst also considering the claims of the arguments in the evaluation; they are equally well-suited for AF instantiations of related formalisms like logic programming %
which typically require the translation from the constructed arguments into their claims as a final step before the outcome can be compared with the original instance \cite{CaminadaSAD15a}.
Semantics for CAFs can be obtained by evaluating the underlying AF before inspecting the claims of the acceptable arguments in the final step (thus simulating the process of instantiation procedures as outlined above).

In contrast to Dung AFs where arguments are unique, claims can appear multiple times as conclusions of different arguments within a single argumentation framework. As a consequence, several properties of AF semantics such as \emph{I-maximality}, i.e., $\subseteq$-maximality of extensions, cannot be taken for granted when considered in terms of the arguments' claims \cite{DvorakRW20c}. Furthermore, the additional level of claims causes a rise in the computational complexity of standard decision problems (in particular, verification is one level higher in the polynomial hierarchy as for standard AFs), see \cite{DvorakW20,DvorakGRW21}. 
Luckily, these drawbacks can be alleviated by taking fundamental properties of the attack relation into account: the basic observation that attacks typically depend on the claim of the attacking arguments %
gives rise to the central class of \textit{well-formed CAFs}. 
CAFs from this class require that arguments with the same claim attack the same arguments, thus
modeling -- on the abstract level -- a very natural behavior of arguments that is common to all leading structured argumentation formalisms and instantiations. %
Well-formed CAFs have the main advantage that most of the semantics behave `as expected'. For instance, they retain the fundamental property of I-maximality, and their computational complexity is located at the same level of the polynomial hierarchy as for AFs.
Unfortunately, it turns out that well-formedness cannot be assumed if one deals with preferences in argumentation, as arguments with the same claim are not necessarily equally plausible.
Let us %
have a look on a simple example.

\begin{figure}[ht]
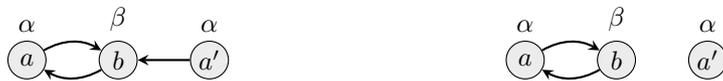

	\captionsetup[subfigure]{labelformat=empty}
	\centering
	\begin{subfigure}{0.45\columnwidth}
		\centering
		\tikz{
			\node[arg, label={above}:$\alpha$] (a) at (0,0) {$a$};
			\node[arg, label={above}:$\beta$] (b) at (1.2,0) {$b$};
			\node[arg, label={above}:$\alpha$] (c) at (2.4,0) {$a'$};
			\draw[attack] 
			(a) edge [bend left] (b)
			(b) edge [bend left] (a)
			(c) edge (b);
		}
	\end{subfigure}
	\hfill
	\begin{subfigure}{0.45\columnwidth}
		\centering
		\tikz{
			\node[arg, label={above}:$\alpha$] (a) at (0,0) {$a$};
			\node[arg, label={above}:$\beta$] (b) at (1.2,0) {$b$};
			\node[arg, label={above}:$\alpha$] (c) at (2.4,0) {$a'$};
			\draw[attack] 
			(a) edge [bend left] (b)
			(b) edge [bend left] (a)
		}
	\end{subfigure}
	\hfill
	\caption{$F$ (left) and $F'$ (right) from Example~\ref{ex:intro}.}
	\label{fig:intro}
\end{figure}
\begin{example} \label{ex:intro}
	Consider two arguments $a,a'$ with claim $\alpha$, and another argument $b$ having claim $\beta$.
	Without specifying the particular formalism, we may assume that each argument has a support which entails its claim:
	the claim $\alpha$ is entailed by $\neg \beta$ (in argument $a$) and by $\gamma$ (in argument $a'$); and $\beta$ is entailed by $\neg \alpha$. The corresponding CAF $F$ 
	is depicted in Figure~\ref{fig:intro} (left); as usual, an argument $a$ attacks another argument $b$ if the claim of $a$ contradicts the support of $b$. Note that $F$ is well-formed. The unique acceptable claim-set w.r.t.\ stable semantics (cf.~Definition~\ref{def:semantics})
	is $\{\alpha\}$ (obtained from the stable extension $\{a,a'\}$ of the underlying AF).
	
	Now assume that $b$ is preferred over $a'$ (for example, if assumptions in the support of $b$ are stronger than assumptions made by $a'$). A common method to integrate such information on argument rankings is to \emph{delete} attacks from arguments that attack preferred arguments. In this case, we delete the attack from $a'$ to $b$ (cf.\ CAF $F'$ in Figure~\ref{fig:intro}). Observe that the resulting CAF is not well-formed. Moreover, we obtain that I-maximality is violated since both claim-sets $\{\alpha\}$, $\{\alpha,\beta\}$ are acceptable w.r.t.\ stable semantics.
\end{example}
As preferences or argument rankings are a standard feature of many argumentation systems it is often the case that the resulting CAF violates the property of well-formedness.
Nevertheless, this does not imply arbitrary behavior of the resulting CAF:
on the one hand, preferences conform to a certain type of ordering (e.g., strict, partial, or total orders) over the set of arguments; on the other hand, it is evident that the deletion, reversion, and other types of attack manipulation impose certain restrictions on the structure of the resulting CAF. Combining both aspects, we obtain that, assuming well-formedness of the initial framework, it is unlikely that preference incorporation results in arbitrary behavior. An easy example which supports this hypothesis is given by self-attacking arguments: assuming a reasonable preference ordering, it is clear that self-attacking arguments are not affected by commonly used attack modifications since an argument cannot be preferred to itself.
The key motivation of this paper is thus to identify and exploit structural properties of preferential argumentation in the scope of claim justification. The aforementioned restrictions suggest beneficial impact on both the computational complexity as well as in terms of desired semantical properties.

In this paper, we tackle this question by considering four commonly used methods, so-called reductions, to integrate preference orderings into the attack relation: 
the most common modification is the deletion of attacks in case the attacking argument is less preferred than its target (cf.\ Example~\ref{ex:intro}). This method is typically utilized to transform preference-based argumentation frameworks (PAFs) \cite{AmgoudC98} into AFs but is also used in many structured argumentation formalisms such as ASPIC+. 
This reduction has been criticized due to several problematic side-effects, e.g., it can be the case that two conflicting arguments are jointly acceptable, and has been accordingly adapted in \cite{AmgoudV14};
two other reductions have been introduced in \cite{KaciTV18}.
We apply these four preference reductions to well-formed CAFs with preferences.
In particular, our main contributions are as follows:
\begin{itemize}
	\item For each of the four preference reductions,
	we characterize the possible structure of CAFs that are obtained by applying the reduction to a well-formed CAF and a preference relation. This results in four novel CAF classes, each of which constitutes a proper extension of well-formed CAFs but does not retain the full expressiveness of general CAFs. Moreover, we investigate the relationship between these classes.
	
	\item We study I-maximality of stable, preferred, semi-stable, stage, and naive semantics of the novel CAF classes. Our results highlight a significant advantage of a particular reduction: we show that, for admissible-based semantics, this modification preserves I-maximality. The other reductions fail to preserve I-maximality; moreover, for %
	naive and stage semantics, I-maximality cannot be guaranteed for any of the four reductions. 
	
	\item Finally, we investigate the 
	complexity of reasoning for %
		CAFs with preferences
	with respect to the five semantics mentioned above plus conflict-free, admissible, and complete semantics.
	We show that for three of the four reductions, the verification problem drops by one level in the polynomial hierarchy for almost all
	considered semantics and is thus not harder than for well-formed CAFs (which in turn has the same complexity as the corresponding problems for AFs). Complete semantics are the only exception, remaining hard for all but one preference reduction.
	On the other hand, it turns out that verification for the most commonly used reduction, i.e., deleting attacks from weaker arguments, %
	remains as hard as for general CAFs.
\end{itemize}

\section{Preliminaries} \label{sec:preliminaries}

We first define (abstract) 
argumentation frameworks~\cite{Dung95}. $U$ denotes a countable infinite domain of arguments.

\begin{definition} \label{def:AF}
	An argumentation framework (AF) is a tuple $\AF = (\Args,\Att)$ where $\Args \subseteq \U$ is a finite set of arguments and $\Att \subseteq \Args \times \Args$ is an attack relation between arguments. 
	Let $E \subseteq \Args$. We say $E$ \emph{attacks} $b$ (in $\AF$) if $(a,b)\in \Att$ for some $a\in E$; 
	$E^+_\AF=\{b\in \Args \mid\exists a\in E: (a,b)\in \Att\}$ denotes the set of arguments attacked by $E$.
	$E^\oplus_\AF=E\cup E^+_\AF$ is the \emph{range} of $E$ in $\AF$.
	An argument $a \in \Args$ is \emph{defended} (in $\AF$) by $E$ if $b\in E^+_F$ for each $b$ with $(b,a)\in \Att$.
\end{definition}

Given an AF $F = (\Args,\Att)$ it can be convenient to write $a \in F$ for $a \in \Args$ and $(a,b) \in F$ for $(a,b) \in \Att$. 
Semantics for AFs are defined as 
functions $\sigma$ which assign to each AF
$\AF=(\Args,\Att)$ 
a set
$\sigma(\AF)\subseteq 2^{\Args}$
of extensions. 
We consider for 
$\sigma$ the functions
$\cf$,
$\adm$, 
$\comp$,   
$\naive$,
$\stable$, 
$\pref$,  
$\semi$, 
and 
$\stage$ 
which 
stand for
conflict-free,
admissible, 
complete,  
naive, 
stable,
preferred, 
semi-stable,
and 
stage, 
respectively \cite{BaroniCG18}.

\begin{definition}\label{def:semantics}
	Let $F = (\Args,\Att)$ be an AF.  A set $S\subseteq \Args$ is 
	{\em conflict-free (in $F$)}, 
	if there are no 
	$a, b \in S$, such that $(a,b) \in \Att$.
	$\cf(F)$ denotes the collection of conflict-free sets of $F$.
	For a conflict-free set $S \in \cf(F)$, it holds that
	\begin{itemize}
		\item $S\in\adm(F)$ 
		if  each $a\in S$ is defended by $S$ in $F$;
		\item $S\in\comp(F)$ 
		if $S\in\adm(F)$ and each $a\in \Args$ defended by $S$ in $F$ is contained in $S$;
		\item $S\in\naive(F)$
		if there is no $T\in\cf(F)$ with $S\subset T$;
		\item $S\in\stable(F)$ 
		if each $a\in \Args \setminus S$ is attacked by $S$ in $F$;
		\item $S\in\pref(F)$ 
		if $S\in\adm(F)$ and 
		there is no $T\in\adm(F)$ with $S \subset T$;
		\item $S\in\semi(F)$ if $S\in\adm(F)$
		and there is no $T\in\adm(F)$ with $S^\oplus_F\subset T^\oplus_F$;
		\item $S\in\stage(F)$ if there is no $T\in\cf(F)$ 
		with $S^\oplus_F\subset T^\oplus_F$.
	\end{itemize}
\end{definition}

\begin{example} \label{ex:AF}
	Consider the AF $F = (\{a,a',b\},\{(a,b),$ $(a',b),(b,a)\})$ depicted in Figure~\ref{fig:intro} (ignoring claims $\alpha$ and $\beta$). It can be check that  
	$\cf(F) = \allowbreak \{\emptyset,\allowbreak \{a\},\allowbreak \{a'\},\allowbreak \{b\},\allowbreak \{a,a'\}\}$, 
	$\adm(F) = \allowbreak \{\emptyset,\allowbreak \{a\},\allowbreak \{a'\},\allowbreak \{a,a'\}\}$, 
	$\naive(F) = \allowbreak \{\{b\},\allowbreak \{a,a'\}\}$, and
	$\sigma(F) = \allowbreak \{\{a, a'\}\}$ for $\sigma\in\{\comp,\allowbreak \stb,\allowbreak \pref,\allowbreak \semi,\allowbreak \stage\}$.
\end{example}

CAFs are AFs in which each argument is assigned a claim, and thus constitute a straightforward generalization of AFs \cite{DvorakW20}.

\begin{definition} \label{def:CAF}
	A claim-augmented argumentation framework (CAF) is a triple $(\Args,\Att,\cl)$ where $(\Args,\Att)$ is an AF and $\cl \colon \Args \to \Claims$ is a function that maps arguments to %
	claims. 
	The claim-function 
	can be extended to sets of arguments,
	i.e., $\cl(E)=\{\cl(a)\mid a\in E\}$. 
	A well-formed CAF (wfCAF) is a CAF $(\Args,\Att,\cl)$ in which all arguments with the same claim attack the same arguments, i.e., for all $a,b \in \Args$ with $\cl(a) = \cl(b)$ we have that $\{c \mid (a,c) \in \Att\} = \{c \mid (b,c) \in \Att\}$. 
\end{definition}

The semantics of CAFs are based on those of AFs.

\begin{definition}
	Let $F = (\Args,\Att,\cl)$ be a CAF. The claim-based variant of a semantics $\sigma$ is defined as $\sigma\!_c(F) = \{\cl(S) \mid S\in\sigma((\Args,\Att))\}$.
\end{definition}

\begin{example} \label{ex:CAF}
	Consider the CAF $F$ shown on the left in Figure~\ref{fig:intro}. Formally,  $F = (A,R,\cl)$ with $A = \{a,a',b\}$, $R = \{(a,b),(a',b),(b,a)\}$, $\cl(a) = \cl(a') = \alpha$, and $\cl(b) = \beta$.  $F$ is well-formed and the underlying AF of $F$ was investigated in Example~\ref{ex:AF}. From there we can infer that, e.g.,
	$\cf\!_c(F) = \{\emptyset,\{\alpha\},\{\beta\}\}$, 
	$\adm_c(F) = \{\emptyset, \{\alpha\}\}$, 
	$\naive_c(F) = \{\{\alpha\},\{\beta\}\}$, and
	$\stb\!_c(F) = \{\{\alpha\}\}$.
\end{example}

Well-known basic relations between different AF semantics $\sigma$ also hold for $\sigma_c$: $\stb_c(F) \subseteq \semi_c(F) \subseteq \pref_c(F) \subseteq \adm_c(F)$ as well as $\stb_c(F) \subseteq \stage_c(F) \subseteq \naive_c(F) \subseteq \cf\!_c(F)$ \cite{DvorakRW20c}.  

Note that the semantics $\sigma \in \{ \naive,\stb,\pref,\semi,\stage\}$ employ argument maximization and result in incomparable extensions on regular AFs: for all $S,T \in \sigma(F)$, $S \subseteq T$ implies $S = T$. This property is referred to as I-maximality, and is defined analogously for CAFs:

\begin{definition}
	$\sigma\!_c$ is I-maximal for a class $\mathcal{C}$ of CAFs if, for all CAFs $F \in \mathcal{C}$ and all $S,T \in \sigma\!_c(F)$, $S \subseteq T$ implies $S = T$.
\end{definition}

Table~\ref{tab:imax_CAFwfCAF} shows I-maximality properties of CAFs \cite{DvorakRW20c}, 
revealing an important property of wfCAFs compared to general CAFs: I-maximality is preserved in all semantics except $\naive_c$, implying natural behavior of these maximization-based semantics 
analogous to regular AFs; see, e.g.,  \cite{TorreVesic2017} for a general discussion of such properties.

\begin{table}
	\centering
	\begin{tabular}{l|*{5}{c}}
		& $\naive\!_c$ & $\stable\!_c$ & $\pref\!_c$ & $\semi\!_c$ & $\stage\!_c$\\
		\hline 
		$\CAF$ & \textsf{x} & \textsf{x} & \textsf{x} & \textsf{x} & \textsf{x}\\
		$\wfCAF$ & \textsf{x} & $\checkmark$ & $\checkmark$ & $\checkmark$ & $\checkmark$
	\end{tabular}
	\caption{I-maximality of CAFs.}
	\label{tab:imax_CAFwfCAF}
\end{table}  

Regarding computational complexity, we consider the following standard decision problems pertaining to CAF-semantics $\sigma\!_c$:
\begin{itemize}
	\item \emph{Credulous Acceptance} ($\Cred_{\sigma}^\CAF$): Given a 
	CAF $F$ and claim $\alpha$, is $\alpha$ contained in some $S \in \sigma\!_c(F)$?
	\item \emph{Skeptical Acceptance} ($\Skept_{\sigma}^\CAF$): Given a 
	CAF $F$ and claim $\alpha$, is $\alpha$ contained in each $S\in \sigma\!_c(F)$?
	\item \emph{Verification} ($\Ver_{\sigma}^\CAF$): Given a 
	CAF $F$ and a set of claims $S$, is $S \in \sigma\!_c(F)$?
\end{itemize}
We furthermore consider these reasoning problems restricted to wfCAFs and denote them by $\Cred_{\sigma}^\wfCAF$, $\Skept_{\sigma}^\wfCAF$, and $\Ver_{\sigma}^\wfCAF$. 
Table~\ref{tab:complexity_CAFwfCAF} shows the complexity of these problems \cite{DvorakW20, DvorakGRW21}. Here we see that the complexity of the verification problem drops by one level in the polynomial hierarchy when comparing general CAFs to wfCAFs. This is an important advantage of wfCAFs, as a lower complexity in the verification problem allows for a more efficient enumeration of claim-extensions (cf.\  \cite{DvorakW20}).

\setlength{\tabcolsep}{4pt}
\renewcommand{\arraystretch}{1.1}
\begin{table}
	\centering
	\begin{tabular}{c|*{5}{c} }
		$\sigma$ &  $\Cred_{\sigma}^\Delta$ & $\Skept_{\sigma}^\Delta$ & $\Ver_{\sigma}^\CAF$ & $\Ver_{\sigma}^\wfCAF$ \\
		\hline 
		$\cf$ & in $\P$&  trivial &  $\NP$-c & in $\P$ \\
		$\adm$ & $\NP$-c & trivial &  $\NP$-c & in $\P$ \\
		$\comp$ & $\NP$-c & $\P$-c &  $\NP$-c & in $\P$ \\
		$\naive$ & in $\P$ & $\coNP$-c & $\NP$-c & in $\P$ \\
		$\stable$ & $\NP$-c & $\coNP$-c & $\NP$-c & in P \\
		$\pref$& $\NP$-c & $\PiP{2}$-c & $\SigmaP{2}$-c & $\coNP$-c \\
		$\semi/\stage$ & $\SigmaP{2}$-c & $\PiP{2}$-c & $\SigmaP{2}$-c & $\coNP$-c \\
	\end{tabular}
	\caption{Complexity of CAFs ($\Delta\in\{\CAF,\wfCAF\}$).}
	\label{tab:complexity_CAFwfCAF}
\end{table}

\section{Preference-based CAFs} \label{sec:pcafs}

As discussed in the previous sections, wfCAFs are a natural subclass of CAFs with advantageous properties in terms of I-maximality and computational complexity. 
However, when resolving preferences among arguments the resulting CAFs are typically no longer well-formed (cf.~Example~\ref{ex:intro}).
In order to study preferences under a claim-centric view we introduce preference-based CAFs. 
These frameworks enrich the notion of wfCAFs with the concept of argument strength in terms of preferences. 
Our main goals are then to understand the effect of resolved preferences on the structure of the underlying wfCAF on the one hand, and to determine whether 
the advantages of wfCAFs are maintained on the other hand. Given this motivation, it is reasonable to consider the impact of preferences on \emph{well-formed} CAFs only. 

\begin{definition} \label{def:PCAF}
	A preference-based claim-augmented argumentation framework (PCAF) is a quadruple $F = (\Args,\Att,\cl,\succ)$ where $(\Args,\Att,\cl)$ is a well-formed CAF and $\succ$ is an asymmetric preference relation over $\Args$. 
\end{definition}

Note that 
preferences in PCAFs are not required to be transitive. While transitivity of preferences is often assumed in argumentation \cite{AmgoudV14,KaciTV18}, it cannot always be guaranteed in practice \cite{KaciEtAl2021hofa2}. 
However, we will consider the effect of transitive orderings when applicable.

If $a$ and $b$ are arguments and $a \succ b$ holds then we say that $a$ is stronger than $b$. 
But what effect should this ordering have? How should this influence, e.g., the set of admissible arguments? One possibility is to remove all attacks from weaker to stronger arguments in our PCAF, and to then determine the set of admissible arguments in the resulting CAF. This altering of attacks in a PCAF based on its preference-ordering is called a reduction. 
The literature describes four such reductions for regular AFs \cite{KaciTV18,AmgoudC02,AmgoudV14}. Following \cite{KaciTV18} we next recall these reductions.

\begin{figure*}[t]
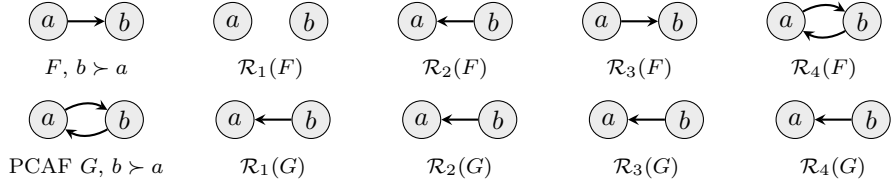

	\captionsetup[subfigure]{labelformat=empty}
	\centering
	\begin{subfigure}{0.19\textwidth}
		\centering
		\tikz{
			\node[arg] (a) at (0,0) {$a$};
			\node[arg] (b) at (1,0) {$b$};
			\draw[attack] 
			(a) edge (b);
		}
		\caption{$F$, $b\succ a$}
		\label{fig:beforeReductionSingle}
	\end{subfigure}
	\hfill
	\begin{subfigure}{0.19\textwidth}
		\centering
		\tikz{
			\node[arg] (a) at (0,0) {$a$};
			\node[arg] (b) at (1,0) {$b$};
		}
		\caption{$\red{1}(F)$}
		\label{fig:afterReduction1Single}
	\end{subfigure}
	\hfill
	\begin{subfigure}{0.19\textwidth}
		\centering
		\tikz{
			\node[arg] (a) at (0,0) {$a$};
			\node[arg] (b) at (1,0) {$b$};
			\draw[attack] 
			(b) edge (a);
		}
		\caption{$\red{2}(F)$}
		\label{fig:afterReduction2Single}
	\end{subfigure}
	\hfill
	\begin{subfigure}{0.19\textwidth}
		\centering
		\tikz{
			\node[arg] (a) at (0,0) {$a$};
			\node[arg] (b) at (1,0) {$b$};
			\draw[attack] 
			(a) edge(b);
		}
		\caption{$\red{3}(F)$}
		\label{fig:afterReduction3Single}
	\end{subfigure}
	\hfill
	\begin{subfigure}{0.19\textwidth}
		\centering
		\tikz{
			\node[arg] (a) at (0,0) {$a$};
			\node[arg] (b) at (1,0) {$b$};
			\draw[attack] 
			(a) edge [bend left] (b)
			(b) edge [bend left] (a);
		}
		\caption{$\red{4}(F)$}
		\label{fig:afterReduction4Single}
	\end{subfigure}
	\hfill
	\par\medskip
	\begin{subfigure}{0.19\textwidth}
		\centering
		\tikz{
			\node[arg] (a) at (0,0) {$a$};
			\node[arg] (b) at (1,0) {$b$};
			\draw[attack] 
			(a) edge [bend left] (b)
			(b) edge [bend left] (a);
		}
		\caption{PCAF $G$, $b \succ a$}
		\label{fig:beforeReductionDouble}
	\end{subfigure}
	\hfill
	\begin{subfigure}{0.19\textwidth}
		\centering
		\tikz{
			\node[arg] (a) at (0,0) {$a$};
			\node[arg] (b) at (1,0) {$b$};
			\draw[attack] 
			(b) edge (a);
		}
		\caption{$\red{1}(G)$}
		\label{fig:afterReduction1Double}
	\end{subfigure}
	\hfill
	\begin{subfigure}{0.19\textwidth}
		\centering
		\tikz{
			\node[arg] (a) at (0,0) {$a$};
			\node[arg] (b) at (1,0) {$b$};
			\draw[attack] 
			(b) edge (a);
		}
		\caption{$\red{2}(G)$}
		\label{fig:afterReduction2Double}
	\end{subfigure}
	\hfill
	\begin{subfigure}{0.19\textwidth}
		\centering
		\tikz{
			\node[arg] (a) at (0,0) {$a$};
			\node[arg] (b) at (1,0) {$b$};
			\draw[attack] 
			(b) edge (a);
		}
		\caption{$\red{3}(G)$}
		\label{fig:afterReduction3Double}
	\end{subfigure}
	\hfill
	\begin{subfigure}{0.19\textwidth}
		\centering
		\tikz{
			\node[arg] (a) at (0,0) {$a$};
			\node[arg] (b) at (1,0) {$b$};
			\draw[attack] 
			(b) edge (a);
		}
		\caption{$\red{4}(G)$}
		\label{fig:afterReduction4Double}
	\end{subfigure}
	\caption{Effect of the four reductions on the attack relation between two arguments.}
	\label{fig:reductionVisualization}
\end{figure*}

\begin{definition} \label{def:reductions}
	Given a PCAF $F = (\Args,\Att,\cl,\succ)$, a corresponding CAF $\red{i}(F) = (\Args,\Def,\cl)$ is constructed via Reduction~$i$, where $i \in \{1,2,3,4\}$, as follows:
	\begin{itemize}
		\item $i = 1$: $\forall a,b\in \Args:\ (a,b)\in \Def \Leftrightarrow (a,b)\in \Att,b\not \succ a$
		\item $i = 2$: $\forall a,b\in \Args:\ (a,b)\in \Def \Leftrightarrow ((a,b)\in \Att,b\not \succ a) \allowbreak \lor \allowbreak ((b,a)\in \Att, (a,b)\notin \Att, a\succ b)$
		\item $i = 3$: $\forall a,b\in \Args:\ (a,b)\in \Def \Leftrightarrow ((a,b)\in \Att,b\not \succ a) \allowbreak \lor \allowbreak ((a,b) \in \Att, (b,a) \not\in \Att)$ 
		\item $i = 4$: $\forall a,b\in \Args:\ (a,b)\in \Def \Leftrightarrow ((a,b)\in \Att,b\not \succ a) \allowbreak \lor \allowbreak ((b,a)\in \Att, (a,b)\notin \Att, a\succ b) \lor ((a,b) \in \Att, (b,a) \not\in \Att)$ 
	\end{itemize}
\end{definition}

Figure~\ref{fig:reductionVisualization} visualizes the above reductions. Intuitively, Reduction~$1$ removes attacks that contradict the preference ordering while Reduction~$2$ reverts such attacks. Reduction~$3$ removes attacks that contradict the preference ordering, but only if the weaker argument is attacked by the stronger argument also. Reduction~$4$ can be seen as a combination of Reductions~$2$ and~$3$. 
Observe that all four reductions are polynomial time computable with respect to the input PCAF.

The semantics for PCAFs can now be defined in a straightforward way: first, one of the four reductions is applied to the given PCAF; then, CAF-semantics are applied to the resulting CAF.

\begin{definition}
	Let $F$ be a PCAF and let $i \in \{1,2,3,4\}$. The preference-claim-based variant of a semantics $\sigma$ relative to Reduction~$i$ is defined as $\sigma_p^i(F) = \sigma_c(\red{i}(F))$.
\end{definition}

\begin{example} \label{ex:PCAF}
	Let $F = (A,R,\cl,\succ)$ be the PCAF where $A = \{a,a',b\}$, $R = \{(a,b),(a',b),(b,a)\}$, $\cl(a) = \cl(a') = \alpha$, $\cl(b) = \beta$, and $b \succ a'$. The underlying CAF $(A,R,\cl)$ of $F$ was examined in Example~\ref{ex:CAF}. 
	
	$\red{1}(F) = (A,R',\cl)$ with $R' = \{(a,b),(b,a)\}$.  $\red{1}(F)$ is depicted on the right in Figure~\ref{fig:intro}. It can be verified that, e.g.,
	$\adm_p^1(F) = \adm_c(\red{1}(F)) = \{\{\emptyset, \{\alpha\},\{\beta\},\{\alpha,\beta\}\}$
	and 
	$\stb_p^1(F) = \{\{\alpha\},\{\alpha,\beta\}\}$.
	
	Indeed, the choice of reduction can %
	influence the extensions of a PCAF. For example, $\red{2}(F) = (A,R'',\cl)$ with $R'' = \{(a,b),(b,a),(b,a')\}$, $\adm_p^2(F) = \{\emptyset,\{\alpha\},$ $\{\beta\}\}$, and $\stable_p^2(F) = \{\{\alpha\},\{\beta\}\}$.
\end{example}
It is easy to see that basic relations between semantics carry over from CAFs, as, if we have $\sigma_c(F) \subseteq \tau_c(F)$ for two semantics $\sigma,\tau$ and all CAFs $F$, then also $\sigma_p^i(F) \subseteq  \tau_p^i(F)$ for all PCAFs $F$.
It thus holds that for all $i \in \{1,2,3,4\}$, $\stb_p^i(F) \subseteq \semi_p^i(F) \subseteq \pref_p^i(F) \subseteq \adm_p^i(F)$ as well as $\stb_p^i(F) \subseteq \stage_p^i(F) \subseteq \naive_p^i(F) \subseteq \cf_p^i(F)$.

\begin{remark}
	In this paper we require the underlying CAF of a PCAF to be well-formed. The reason for this is that we are interested in whether the benefits of well-formed CAFs are preserved when preferences have to be taken into account. Even from a technical perspective, admitting PCAFs with a non-well-formed underlying CAF is not very interesting with respect to the questions addressed in this paper. Indeed, any CAF could be obtained from such general PCAFs, regardless of which preference reduction we are using, by simply specifying the desired CAF and an empty preference relation. Thus, such general PCAFs have the same properties regarding I-maximality and complexity as general CAFs. %
\end{remark}

\section{Resulting CAF Classes} \label{sec:characterization}

Our first step towards understanding the effect of preferences on wfCAFs is to examine the impact of resolving preferences on the structure of the underlying CAF. To this end, we consider four new CAF classes which are obtained from applying the reductions of Definition~\ref{def:reductions} to PCAFs.

\begin{definition} \label{def:images}
	$\Image{i}$ denotes the set of CAFs that can be obtained by applying Reduction~$i$ to PCAFs, i.e., $\Image{i} = \{\red{i}(F) \mid F \text{ is a PCAF} \}$. 
\end{definition}

It is easy to see that $\Image{i}$, with $i \in \{1,2,3,4\}$, contains all wfCAFs (we can simply specify the desired wfCAF and an empty preference relation). However, not all CAFs are contained in $\Image{i}$. For example, $F = (\{a,b\},\allowbreak \{(a,b),\allowbreak (b,a)\}, \cl)$ with $\cl(a) = \cl(b)$ can not be obtained from a PCAF $F'$: such $F'$ would need to contain either $(a,b)$ or $(b,a)$. But then, since the underlying CAF of a PCAF must be well-formed, $F'$ would have to contain a self-attack which can not be removed by any of the reductions. This is enough to conclude that the four new classes are located in-between wfCAFs and general CAFs:

\begin{restatable}{proposition}{fourImagesLieBetweenCAFandwfCAF}
	\label{prop:fourImagesLieBetweenCAFandwfCAF}
	Let $\setOfAllCAFs$ be the set of all CAFs and $\setOfAllwfCAFs$ the set of all wfCAFs. For all $i \in \{1,2,3,4\}$ it holds that $\setOfAllwfCAFs \subset \Image{i} \subset \setOfAllCAFs$. 
\end{restatable}

Furthermore, the new classes are all distinct from each other, i.e., we are indeed dealing with \emph{four} new CAF classes:

\begin{restatable}{proposition}{imagesAreIncomparable} \label{prop:imagesAreIncomparable}
	For all $i \in \{1,2,4\}$ and all $j \in \{1,2,3,4\}$ such that $i \neq j$ it holds that $\Image{i} \not\subseteq \Image{j}$ and $\Image{3} \subset \Image{i}$.
\end{restatable}
\begin{proofsketch}
	Figure~\ref{fig:imagesAreIncomparable} shows CAFs that are in only one of $\Image{1}$, $\Image{2}$, and $\Image{4}$. Consider the PCAF $F = (\{a,b\}, \{(a,b),(b,b)\}, \cl, \succ)$ with $\cl(a) = \cl(b) = \alpha$ and $b \succ a$. Then $\red{1}(F)$, $\red{2}(F)$, and $\red{4}(F)$ are the CAFs of Figure~\ref{fig:imagesAreIncomparable}. Since self-attacks are not removed or introduced by any  reduction, and the underlying CAF must be well-formed, $F$ is the only PCAF from which $\red{1}(F)$, $\red{2}(F)$, and $\red{4}(F)$ can be obtained. Note that $\red{3}(F)$ is simply the underlying CAF of $F$. 
	$\Image{3} \subset \Image{i}$ follows by the fact that if an attack $(a,b)$ is removed by Reduction~3 from some PCAF $G$, then $(b,a) \in G$. In this case, all reductions behave in the same way (cf.\ Definition~\ref{def:reductions} or Figure~\ref{fig:reductionVisualization}).
\end{proofsketch}
\begin{figure}[t]
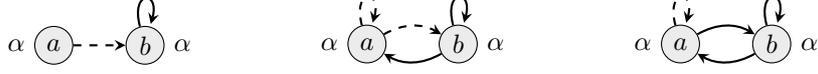

	\centering
	\begin{subfigure}{0.32\columnwidth}
		\centering
		\tikz{
			\node[arg, label={left}:$\alpha$] (a) at (0,0) {$a$};
			\node[arg, label={right}:$\alpha$] (b) at (1.2,0) {$b$};
			\draw[attack] 
			(a) edge [dashed] (b)
			(b) edge [loop above] (b);
		}
	\end{subfigure}
	\hfill
	\begin{subfigure}{0.32\columnwidth}
		\centering
		\tikz{
			\node[arg, label={left}:$\alpha$] (a) at (0,0) {$a$};
			\node[arg, label={right}:$\alpha$] (b) at (1.2,0) {$b$};
			\draw[attack] 
			(b) edge [bend left] (a)
			(a) edge [bend left, dashed] (b)
			(a) edge [dashed, loop above] (a)
			(b) edge [loop above] (b);
		}
	\end{subfigure}
	\hfill
	\begin{subfigure}{0.32\columnwidth}
		\centering
		\tikz{
			\node[arg, label={left}:$\alpha$] (a) at (0,0) {$a$};
			\node[arg, label={right}:$\alpha$] (b) at (1.2,0) {$b$};
			\draw[attack] 
			(a) edge [bend left] (b)
			(b) edge [bend left] (a)
			(b) edge [loop above] (b)
			(a) edge [dashed, loop above] (a);
		}
	\end{subfigure}
	\caption{CAFs contained only in $\Image{1}$, $\Image{2}$, and $\Image{4}$ respectively. Solid arrows are attacks, dashed arrows indicate where attacks are missing for the CAF to be well-formed.}
	\label{fig:imagesAreIncomparable}
\end{figure}
While the classes
$\Image{1}$, $\Image{2}$, and $\Image{4}$ are incomparable
we observe $\Image{3} \subset \Image{i}$ which reflects that Reduction~3 is the most conservative of the four reductions, removing attacks only when there is a counter-attack from the stronger argument.

We now know that applying preferences to wfCAFs results in four distinct CAF-classes that lie in-between wfCAFs and general CAFs. It is still unclear, however, how to determine whether some CAF %
belongs to one of these classes or not. Especially for $\Image{2}$ and $\Image{4}$ this is not straightforward, since Reductions~2 and~4 not only remove but also introduce attacks and therefore allow for many possibilities 
to obtain a particular CAF as result. 
We tackle this problem by characterizing the new classes via the so-called wf-problematic part of a CAF.

\begin{definition}
	A pair of arguments $(a,b)$ is wf-problematic in a given CAF  $F = (\Args, \Att, \cl)$ if $a,b \in \Args$, $(a,b) \not\in \Att$, and there is $a' \in \Args$ with $\cl(a') = \cl(a)$ and $(a',b) \in \Att$. $\problematicPart(F) = \{(a,b) \mid \allowbreak  (a,b) \allowbreak  \text{ is wf-problematic in } F\}$ is called the wf-problematic part of $F$. 
\end{definition}

Intuitively, the wf-problematic part of a CAF $F$ consists of those attacks that are missing for $F$ to be well-formed (cf.\ Figure~\ref{fig:imagesAreIncomparable}). 
Indeed, $F$~is a wfCAF if and only if $\problematicPart(F) = \emptyset$. The four new classes can be characterized as follows:

\begin{restatable}{proposition}{characterizations} \label{prop:characterizations}
	Let $F = (A,R,\cl)$ be a CAF. Then
	\begin{itemize}
		
		\item $F \in \Image{1}$ iff
		$(a,b) \in \problematicPart(F)$ implies $(b,a) \not\in \problematicPart(F)$;
		
		\item $F \in \Image{2}$ iff
		there are no arguments $a,a',b,b'$ in $F$ with $\cl(a) = \cl(a')$ and $\cl(b) = \cl(b')$ such that $(a,b) \in \problematicPart(F)$, $(b,a) \not\in R$, $(a',b) \in R$, and either $(b,a') \in R$ or $((a',b') \not\in R$ and $(b',a') \not\in R)$;
		
		\item $F \in \Image{3}$ iff
		$(a,b) \in \problematicPart(F)$ implies $(b,a) \in R$;
		
		\item $F \in \Image{4}$ iff
		there are no arguments $a,a',b,b'$ in $F$ with $\cl(a) = \cl(a')$ and $\cl(b) = \cl(b')$ such that $(a,b) \in \problematicPart(F)$, $(b,a) \not\in R$, $(a',b) \in R$, and either $(b,a') \not\in R$ or $((a',b') \not\in R$ and $(b',a') \not\in R)$.
	\end{itemize}
\end{restatable}
\begin{proofsketch}
	Regarding $\Image{1}$, observe that Reduction~$1$ can only delete but not introduce attacks. If $(a,b) \in \problematicPart(F)$ implies $(b,a) \not\in \problematicPart(F)$ then we can construct a PCAF $F'$ with $R' = R \cup \{(a,b) \mid (a,b) \in \problematicPart(F)\}$ and $b \succ a$ iff $(a,b) \in R' \setminus R$. Observe that $\succ$ is asymmetric. Conversely, a CAF $G$ with arguments $a,b$ such that $(a,b) \in \problematicPart(G)$ and $(b,a) \in \problematicPart(G)$ can not be obtained via Reduction~$1$ from a PCAF $G'$, since $G'$ would have to contain both the attacks $(a,b), (b,a)$ as well as the preferences $b \succ a, a \succ b$. The argument for $\Image{3}$ is similar. 
	
	For $\Image{2}$, suppose there are $a,a',b$ with $\cl(a) = \cl(a')$, $(a,b) \in \problematicPart(F)$, $(b,a) \not\in R$, and $(a',b) \in R$. Assume there is a PCAF $F' = (A,R',\cl, \allowbreak \succ)$ such that $\red{2}(F') = F$. Since Reduction~$2$ can not completely remove conflicts, $(a,b) \not\in R'$ and $(b,a) \not\in R'$. If $(b,a') \in R$, then $(a',b) \in R'$ and $(b,a') \in R'$ since Reduction~2 can not introduce symmetric attacks. But then $(A,R',\cl)$ is not well-formed. Now suppose $(b,a') \not\in R$, but there is some $b'$ with $\cl(b) = \cl(b')$, $(a',b') \not\in~R$, and $(b',a') \not\in~R$. Then also $(a',b') \not\in~R'$ and $(b',a') \not\in~R'$. But since $(a',b) \in~R$ we have either $(a',b) \in~R'$ or $(b,a') \in~R'$, which means that $(A,R',\cl)$ is not well-formed. In all other cases we can construct a PCAF $F'' = (A,R'',\cl,\succ)$ such that $\red{2}(F'') = F$: first revert all attacks $(a',b)$ in $F$ for which there is some $a$ with $\cl(a) = \cl(a')$ and $(a,b) \not\in R$, $(b,a) \not\in R$; then, add all remaining pairs $(a,b)$ that are still wf-problematic as attacks. Define $b \succ a$ iff $(a,b) \in R'' \setminus R$. It can be verified that $(A,R'',\cl)$ is well-formed, $\succ$ is asymmetric, and $\red{2}(F'') = F$. The argument for $\Image{4}$ is similar. 
\end{proofsketch}

The above characterizations give us some insights into the effect of the various reductions on wfCAFs. Indeed, the similarity between the characterizations of $\Image{1}$ and $\Image{3}$, resp.\ $\Image{2}$ and $\Image{4}$, can intuitively be explained by the fact that Reductions~1 and~3 only remove attacks, while Reductions~2 and~4 can also introduce attacks.
Furthermore, Proposition~\ref{prop:characterizations} allows us to decide in polynomial time whether a given CAF~$F$ can be obtained by applying one of the four preference reductions to a PCAF. 

But what happens if we restrict ourselves to transitive preferences? Analogously to $\Image{i}$, by $\ImageTrans{i}$ we denote the set of CAFs obtained by applying Reduction~$i$ to PCAFs with a transitive preference relation. It is clear that $\ImageTrans{i} \subseteq \Image{i}$ for all $i \in \{1,2,3,4\}$. Interestingly, %
the relationship between the classes $\ImageTrans{i}$ is different to that between $\Image{i}$ %
(Proposition~\ref{prop:imagesAreIncomparable}).
Specifically, $\ImageTrans{3}$ is not contained in the other classes. 
Intuitively, this is because, in certain PCAFs $F$, transitivity can force $a_1 \succ a_n$ via $a_1 \succ a_2 \succ \ldots \succ a_n$ such that $(a_n,a_1) \in F$ but $(a_1,a_n) \not\in F$. In this case, only Reduction~3 leaves the attacks between $a_1$ and $a_n$ unchanged.

\begin{restatable}{proposition}{imagesAreIncomparableTransitive} \label{prop:imagesAreIncomparable-transitive}
	For $i,j \in \{1,2,3,4\}$ with $i \neq j$ it holds that $\ImageTrans{i} \not\subseteq \ImageTrans{j}$.
\end{restatable}

We will not characterize all four classes $\ImageTrans{i}$. However, capturing $\ImageTrans{1}$ will prove useful when analyzing the computational complexity of PCAFs using Reduction~$1$ (see Section~\ref{sec:complexity}). Note that $\problematicPart(F)$ can be seen as a directed graph, with an edge between vertices $a$ and $b$ whenever $(a,b) \in \problematicPart(F)$. Thus, we may use notions such as paths and cycles in the wf-problematic part of a CAF.

\begin{restatable}{proposition}{characterizationFirstImageTransitive} \label{prop:characterizationImage1-transitive}
	$F \in \ImageTrans{1}$ for a CAF $F$ if and only if 
	(1)~$\problematicPart(F)$ is acyclic and 
	(2)~$(a,b) \in F$ implies that there is no path from $a$ to $b$ in $\problematicPart(F)$. 
\end{restatable}

\begin{proofsketch}
	Assume there is a cycle $(a_1,\ldots,a_n,a_1)$ in $\problematicPart(F)$. Then, since Reduction~$1$ can not introduce attacks, if there is a PCAF $F'$ such that $\red{1}(F') = F$, we have $(a_1,a_2),\ldots,(a_n,a_1) \in F'$. This implies $a_1 \succ a_n \succ a_{n-1} \succ \cdots \succ a_1$, i.e., $\succ$ is not asymmetric. Similarly, if there is a path $(a_1,\ldots,a_n)$ in $\problematicPart(F)$ we have to define $a_n \succ \cdots \succ a_1$ in $F'$. But then $(a_1,a_n) \not\in \red{1}(F')$. 
	
	If $\problematicPart(F)$ is acyclic and there is no path from $a$ to $b$ in $\problematicPart(F)$ such that $(a,b) \in F$, then we can construct a PCAF $F'$ such that $\red{1}(F') = F$ in the same way as when $\succ$ is not transitive (cf.\ proof of Proposition~\ref{prop:characterizations}). 
\end{proofsketch}

\section{I-Maximality} \label{sec:imax}

One of the advantages of wfCAFs over general CAFs is that they preserve I-maximality under most maximization-based semantics (cf.\ Table~\ref{tab:imax_CAFwfCAF}), which leads to more intuitive behavior of these semantics when considering extensions on the claim-level. We now investigate whether these advantages are preserved when preferences are introduced.

\begin{definition}
	$\sigma^i_p$ is I-maximal for a class $\mathcal{C}$ of PCAFs if, for all $F$ in $\mathcal{C}$ and all $S,T \in \sigma^i_p(F)$, $S \subseteq T$ implies $S = T$.
\end{definition}

From known properties of wfCAFs (cf.\ Table~\ref{tab:imax_CAFwfCAF}) it follows directly that $\naive_p^i$ is not I-maximal for PCAFs. It remains to investigate the I-maximality of $\pref_p^i$, $\stb_p^i$, $\semi_p^i$, and $\stage_p^i$ for PCAFs. 
For convenience, given a CAF $F = (\Args,\Att,\cl)$ and $E \subseteq A$, we sometimes write $E \in \sigma(F)$ for $E \in \sigma((A,R))$. 

\begin{restatable}{lemma}{cfDoesNotChangeForSomeReductions} \label{lemma:cfDoesNotChangeForReductions234}
	Let $F = (\Args,\Att,\cl, \succ)$ be a PCAF and ${E \subseteq \Args}$. $E \in \cf(\red{i}(F))$ if and only if $E \in \cf((\Args,\Att,\cl))$ for $i \in \{2,3,4\}$.
\end{restatable}

In other words, Reductions~$2$, $3$~and~$4$ preserve conflict-freeness. It is easy to see that this is not the case for Reduction~$1$. In fact, Reduction~$1$ has been deemed problematic for exactly this reason when applied to regular AFs \cite{AmgoudV14}, although it is still discussed and considered in the literature alongside the other reductions \cite{KaciEtAl2021hofa2}. We first consider Reduction~$3$, and show that it preserves I-maximality for some, but not all, semantics.

\begin{restatable}{proposition}{imaxThirdImage} \label{prop:prop:i-max-image3-pref}
	$\pref_p^3$, $\stb_p^3$, and $\semi_p^3$ are I-maximal for PCAFs.
\end{restatable}
\begin{proof}
	By $\stb_p^3(F) \subseteq \semi_p^3(F) \subseteq \pref_p^3(F)$ it suffices to consider $\pref_p^3$. Towards a contradiction, assume there is a PCAF $F = (\Args,\Att,\cl, \succ)$ such that $S \subset T$ for some $S,T \in \pref^3_p(F)$. Then there are $S',T' \subseteq \Args$ such that $S' \in \pref(\red{3}(F))$, $\cl(S') = S$, $T' \in \pref(\red{3}(F))$, and $\cl(T') = T$. $S' \not\subseteq T'$ since $S' \in \pref(\red{3}(F))$. Thus, there is $x \in S'$ such that $x \not\in T'$. But $\cl(x) \in T$, i.e., there is $x' \in T'$ with $\cl(x') = \cl(x)$. There are two possibilities for why $x \not\in T'$.
	
	Case 1: $T' \cup \{x\} \not\in \cf(\red{3}(F))$, i.e., there exists $y \in T'$ such that $y \not\in S'$ and either $(x,y) \in F$ or $(y,x) \in F$. In fact, $(x,y) \not\in F$: otherwise, by the well-formedness of $(\Args,\Att,\cl)$, we have $(x',y) \in F$ and, by Lemma~\ref{lemma:cfDoesNotChangeForReductions234}, $T' \not\in \cf(\red{3}(F))$. Thus, $(y,x) \in F$.
	By the definition of Reduction~$3$, $(y,x) \in \red{3}(F)$. $S'$ must defend $x$ in $\red{3}(F)$, i.e., there exists $z \in S'$ such that $(z,y) \in \red{3}(F)$. Then $(z,y) \in F$. Since $S \subset T$ there exists $z' \in T'$ such that $\cl(z') = \cl(z)$. $(z',y) \in F$ by the well-formedness of $(\Args,\Att,\cl)$. But then $T' \not\in \cf(\red{3}(F))$. Contradiction.
	
	Case 2: $x$ is not defended by $T'$, i.e., there exists $y \in \Args$ that is not attacked by $T'$ and such that $(y,x) \in \red{3}(F)$. 
	By the same argument as above, there is $z' \in T'$ with $(z',y) \in F$.
	It cannot be that $(z',y) \in \red{3}(F)$, i.e., $y \succ z'$. By the definition of Reduction~$3$, $(y,z') \in F$ and thus $(y,z') \in \red{3}(F)$. 
	But then $T' \not\in \adm(\red{3}(F))$. Contradiction.
\end{proof}

Of course, positive results regarding the I-maximality of PCAFs with arbitrary preferences, such as in the above proposition, still hold for PCAFs with transitive preference orderings. Conversely, for negative results, it suffices to show that I-maximality is not preserved on transitive orderings to obtain results for the more general case.

\begin{restatable}{proposition}{imaxThirdImageStage} \label{prop:i-max-image3-stg}
	$\stage_p^3$ is not I-maximal for PCAFs, even when considering only transitive preferences.
\end{restatable}
\begin{proofsketch}
	Let $F$ be the CAF shown on the left in Figure~\ref{fig:ImaxCafsShort}.
	Observe that $F \in \ImageTrans{3}$ since $\red{3}(F') = F$ for the PCAF $F'$ with the 
	same arguments as $F$,
	attacks $\{(a,b),(b,c),(c,a),(a',b),(b,a')\}$ and ${b \succ a'}$. 
	Moreover, it can be verified that $\stage_p^3(F') = \{\{\alpha\},\{\alpha,\gamma\},\{\beta\}\}$. 
\end{proofsketch}
\begin{figure}[t]
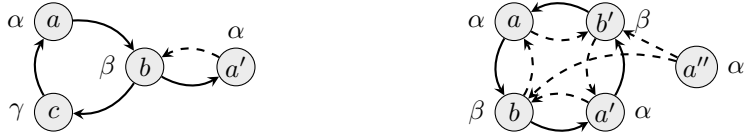

	\centering
	\begin{subfigure}{0.48\columnwidth}
		\centering
		\tikz{
			\node[arg, label={left}:$\alpha$] (a) at (0,1.2) {$a$};
			\node[arg, label={above}:$\alpha$] (aPrime) at (2.4,0.6) {$a'$};
			\node[arg, label={left}:$\beta$] (b) at (1.2,0.6) {$b$};
			\node[arg, label={left}:$\gamma$] (c) at (0,0) {$c$};
			\draw[attack] 
			(a) edge [bend left] (b)
			(b) edge [bend left] (c)
			(c) edge [bend left] (a)
			(aPrime) edge [bend right, dashed] (b)
			(b) edge [bend right] (aPrime);
		}
	\end{subfigure}
	\hfill
	\begin{subfigure}{0.48\columnwidth}
		\centering
		\tikz{
			\node[arg, label={left}:$\alpha$] (a) at (0,1.2) {$a$};
			\node[arg, label={right}:$\alpha$] (b) at (1.2,0) {$a'$};
			\node[arg, label={left}:$\beta$] (c) at (0,0) {$b$};
			\node[arg, label={right}:$\beta$] (d) at (1.2,1.2) {$b'$};
			\node[arg, label={right}:$\alpha$] (e) at (2.4,0.6) {$a''$};
			\draw[attack] 
			(a) edge [bend right] (c)
			(c) edge [bend right, dashed] (a)
			(b) edge [bend right] (d)
			(d) edge [bend right, dashed] (b)
			(c) edge [bend right] (b)
			(b) edge [bend right, dashed] (c)
			(d) edge [bend right] (a)
			(a) edge [bend right, dashed] (d)
			(e) edge [bend right, dashed] (c)
			(e) edge [dashed] (d);
		}
	\end{subfigure}
	\caption{CAFs used as counter examples for I-maximality (cf.\ Proposition~\ref{prop:i-max-image3-stg} and \ref{prop:i-max-images134}). Dashed arrows are edges in 
		$\problematicPart(F)$.
	}
	\label{fig:ImaxCafsShort}
\end{figure}

In contrast to Reduction~3, under Reductions~1, 2, and~4 we lose I-maximality for \emph{all} semantics.

\begin{restatable}{proposition}{imaxOtherImages} \label{prop:i-max-images134}
	$\sigma_p^i$, with $\sigma \in \{\pref,\stb,\semi,\stage\}$ and $i \in \{1,2,4\}$, is not I-maximal for PCAFs, even when considering only transitive preferences.
\end{restatable}
\begin{proofsketch}
	We only need to show this for $\stb_p^i$ since $\stb_p^i(F) \subseteq \semi_p^i(F) \subseteq \pref_p^i(F)$ and $\stb_p^i(F) \subseteq \stage_p^i(F)$. 
	
	For $i \in \{1,4\}$, let $F$ be the CAF shown on the right in Figure~\ref{fig:intro}. $F \in \ImageTrans{1}$ by Proposition~\ref{prop:characterizationImage1-transitive}. Furthermore, $F \in \ImageTrans{4}$ since $\red{4}(F') = F$ for $F' = (\{a,a',b\},\{(b,a)\},\cl,\succ)$ with $a \succ b$. It can be verified that $\stb_c(F) = \{\{\alpha\},\{\alpha,\beta\}\}$. 
	
	For $i = 2$, let $G$ be the CAF shown on the right in Figure~\ref{fig:ImaxCafsShort}. $G \in \ImageTrans{2}$ since $\red{2}(G') = G$ for the PCAF $G'$ with attacks $\{(b,a),(b,a'),(b',a),(b',a')\}$ and preferences $a \succ b$ and $a' \succ b'$. $\stb_c(G) = \{\{\alpha\},\{\alpha,\beta\}\}$. 
\end{proofsketch}

\begin{table}
	\centering
	\begin{tabular}{l|*{5}{c}}
		& $\naive_p^i$ & $\stable_p^i$ & $\pref_p^i$ & $\semi_p^i$ & $\stage_p^i$\\
		\hline 
		$i\!\in\!\{1,2,4\}$ & \textsf{x} & \textsf{x} & \textsf{x} & \textsf{x} & \textsf{x}\\
		$i\!=\!3$ & \textsf{x} & $\checkmark$ & $\checkmark$ & $\checkmark$ & \textsf{x}
	\end{tabular}
	\caption{I-maximality of PCAFs. Results also hold when considering only PCAFs with transitive preferences.}
	\label{tab:imax_PCAF}
\end{table}  

Table~\ref{tab:imax_PCAF} 
summarizes our I-maximality results. Reduction~$3$ manages to preserve I-maximality in most cases. It is also the most conservative of the reductions, preserving conflict-freeness and not adding new attacks. Interestingly, the other three reductions lose I-maximality for \emph{all} semantics.

\section{Computational Complexity} \label{sec:complexity}

Given a preference Reduction~$i \in \{1,2,3,4\}$, we define $\Cred_{\sigma,i}^\PCAF$, $\Skept_{\sigma,i}^\PCAF$, and $\Ver_{\sigma,i}^\PCAF$ analogously to $\Cred_{\sigma}^\CAF$, $\Skept_{\sigma}^\CAF$, and $\Ver_{\sigma}^\CAF$, except that we take a PCAF instead of a CAF as input and appeal to the $\sigma_p^i$ semantics instead of the $\sigma_c$ semantics. Membership results for PCAFs can be inferred from results for general CAFs (recall that the preference reductions from PCAFs to the respective CAF class can be done in polynomial time), and hardness results from results for wfCAFs.  Thus, the complexity of credulous and skeptical acceptance %
follows immediately, i.e., for $i \in \{1,2,3,4\}$ and $\sigma \in \{\cf,\adm,\comp,\naive,\stable,\pref,\semi,\stage\}$,
$\Cred_{\sigma,i}^\PCAF$ and $\Skept_{\sigma,i}^\PCAF$
have the same complexity as  $\Cred_{\sigma}^\wfCAF$ and $\Skept_{\sigma}^\wfCAF$ respectively (cf.\ Table~\ref{tab:complexity_CAFwfCAF}).

However, the complexity of $\Ver_{\sigma,i}^\PCAF$ does not follow from known results. We consider Reduction~$1$ first.

\begin{restatable}{proposition}{verificationConflictfreeNaiveFirstImage} \label{prop:verificationConflictfreeNaive1}
	$\Ver_{\sigma,1}^\PCAF$ is $\NP$-complete for $\sigma \in \{\cf,$ $\naive\}$, even when considering only transitive preferences.
\end{restatable}
\begin{proofsketch}
	$\NP$-membership follows from known results for general CAFs. \NP-hardness: let $\varphi$ be an arbitrary instance of 3-\SAT\ given as a set $\{C_1,\ldots,C_m\}$ of clauses over variables $X$. We construct a PCAF $F = (\Args,\Att,\cl,\succ)$ and a set of claims $S = \{1,\ldots,m\} \cup X$ as follows: 
	\begin{itemize}
		\item $\Args = V \cup \overline{V} \cup H$ where $V = \{x_i \mid x \in C_i, 1 \leq i \leq m\}$, $\overline{V} = \{\overline{x}_i \mid \neg x \in C_i, 1 \leq i \leq m\}$, and \\
		$H = \{x_T, x_F \mid x \in X\}$;
		\item $\Att = \{(x_T,x_i),(x_F,x_i)  \mid x_i \in V \}\ \cup$\\ 
		\phantom{$\Att =\ $}%
		$\{(x_T,\overline{x}_i),(x_F,\overline{x}_i) \mid \overline{x}_i \in \overline{V} \}$;
		
		\item $\cl(x_i) = \cl(\overline{x}_i) = i$ for all $x_i,\overline{x}_i \in V \cup \overline{V}$, 
		$\cl(x_T) = \cl(x_F) = x$ for all $x \in X$;
		\item $x_i \succ x_T$ for all $x_i \in V$ and 
		$\overline{x}_i \succ x_F$ for all $\overline{x}_i \in \overline{V}$.
	\end{itemize}
	Figure~\ref{fig:reductionVerficationImage1Cf} illustrates the above construction. It can be verified that $\varphi$ is satisfiable if and only if  $S \in \cf_p^1(F)$. The same can be shown for $\naive_p^1$.
	Informally, the set $S$ forces us to have, for each $x\in X$, $x_T$ or $x_F$ in $S$ thus simulating a guess for an interpretation. Due to the removed attacks all corresponding occurrences $x_i$ (resp.\ $\overline{x}_i$) can be added to $S$ without conflict. Now it amounts to check whether these occurrences cover all $i$, i.e., make all clauses true under the actual guess.
\end{proofsketch}
\begin{figure}[t]
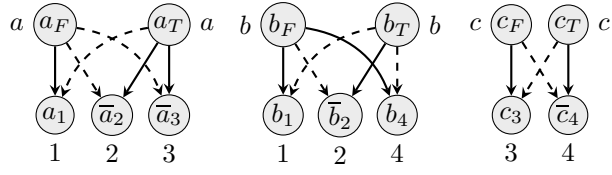

	\centering
	\tikz{
		\node[arg, label={left}:$a$] (aF) at (0,1.2) {$a_F$};
		\node[arg, label={right}:$a$] (aT) at (1.5,1.2) {$a_T$};
		\node[arg, label={left}:$b$] (bF) at (3,1.2) {$b_F$};
		\node[arg, label={right}:$b$] (bT) at (4.5,1.2) {$b_T$};
		\node[arg, label={left}:$c$] (cF) at (6,1.2) {$c_F$};
		\node[arg, label={right}:$c$] (cT) at (6.75,1.2) {$c_T$};
		\node[arg, label={below}:$1$] (a1) at (0,0) {$a_1$};
		\node[arg, label={below}:$2$] (a2) at (0.75,0) {$\overline{a}_2$};
		\node[arg, label={below}:$3$] (a3) at (1.5,0) {$\overline{a}_3$};
		\node[arg, label={below}:$1$] (b1) at (3,0) {$b_1$};
		\node[arg, label={below}:$2$] (b2) at (3.75,0) {$\overline{b}_2$};
		\node[arg, label={below}:$4$] (b4) at (4.5,0) {$b_4$};
		\node[arg, label={below}:$3$] (c3) at (6,0) {$c_3$};
		\node[arg, label={below}:$4$] (c4) at (6.75,0) {$\overline{c}_4$};
		\draw[attack] 
		(aF) edge (a1)
		(aF) edge [dashed] (a2)
		(aF) edge [bend left, dashed] (a3)
		(aT) edge [bend right, dashed] (a1)
		(aT) edge (a2) 
		(aT) edge (a3)
		(bF) edge (b1) 
		(bF) edge [dashed] (b2) 
		(bF) edge [bend left] (b4) 
		(bT) edge [bend right, dashed] (b1)
		(bT) edge (b2)
		(bT) edge [dashed] (b4)
		(cF) edge (c3)
		(cF) edge [dashed] (c4)
		(cT) edge (c4)
		(cT) edge [dashed] (c3);
	}
	\caption{Reduction of 3-\SAT-instance $C_1 = \{a,b,c\}$, $C_2 = \{\neg a, \neg b\}$, $C_3 = \{\neg a, c\}$, $C_4 = \{b, \neg c\}$, to an instance $(F,S)$ of $\Ver_{\cf,1}^\PCAF$ (cf.\ Proof of Proposition~\ref{prop:verificationConflictfreeNaive1}). Dashed arrows are attacks deleted in $\red{1}(F)$, i.e., they are edges in $\problematicPart(\red{1}(F))$.
	}
	\label{fig:reductionVerficationImage1Cf}
\end{figure}

Note that the ``trick'' in above construction to guess an interpretation does not work 
for admissible-based semantics, since the occurrences of $x_i$ resp.\ $\overline{x}_i$ in $S$ would remain undefended. Indeed, we need a more involved construction next.

\begin{restatable}{proposition}{verificationStableAdmissibleCompleteFirstImage} \label{prop:verificationStableAdmissibleComplete1}
	$\Ver_{\sigma,1}^\PCAF$ is $\NP$-complete for $\sigma \in \{\stb,$ $\adm, \comp\}$, even when considering only transitive preferences.
\end{restatable}
\begin{proofsketch}
	We show \NP-hardness. Let $\varphi$ be a 3-\SAT-instance given as a set $\{C_1,\ldots,C_m\}$ of clauses over variables $X$. For convenience, we directly construct a CAF $F = (\Args,\Att,\cl) $ with $F \in \ImageTrans{1}$ instead of providing a PCAF $F'$ such that $\red{1}(F') = F$. This is legitimate, as, by our characterization of $\ImageTrans{1}$ (see Proposition~\ref{prop:characterizationImage1-transitive}), we can obtain $F'$ by simply adding all edges in $\problematicPart(F)$ to $R$ and defining $\succ$ accordingly. 
	\begin{itemize}
		\item $\Args = V \cup \overline{V} \cup H$ where $V = \{x_i \mid x \in C_i, 1 \leq i \leq m\}$, $\overline{V} = \{\overline{x}_i \mid \neg x \in C_i, 1 \leq i \leq m\}$, and\\ 
		$H = \{x^k_{i,j}, \hat{x}^k_{i,j} \mid 1 \leq k \leq 4, x_i \in V, \overline{x}_j \in \overline{V}\}$;
		
		\item $\Att = \{(x_i, x^1_{i,j}), \allowbreak (x^1_{i,j}, x^2_{i,j}), \allowbreak (x^2_{i,j}, \overline{x}_j), \allowbreak (\overline{x}_j, x^3_{i,j}),$ \\ 
		\phantom{$\Att = \{$}%
		$(x^3_{i,j}, x^4_{i,j}), \allowbreak (x^4_{i,j}, x_i) \mid x_i \in V, \overline{x}_j \in \overline{V}\}$;
		
		\item $\cl(x_i) = \cl(\overline{x}_i) = i$ for all $x_i,\overline{x}_i$, \\
		$\cl(x^k_{i,j}) = \cl(\hat{x}^k_{i,j}) = x^k_{i,j}$ for all $x^k_{i,j}, \hat{x}^k_{i,j}$.
		
	\end{itemize}
	For verification consider the set $S = \{1,\ldots,m\} \cup \{\cl(a) \mid a \in H \}$.
	Figure~\ref{fig:reductionVerficationImage1Stb} illustrates the above construction. It can be verified that 
	(1) 
	$F \in \ImageTrans{1}$;  
	(2) 
	$\varphi$ is satisfiable iff  $S \in \stb_c(F)$. Likewise for $\adm_c$ and $\comp_c$. 
	Intuitively, for each $x_i, \overline{x}_j$, the helper arguments $x^k_{i,j}$ and the corresponding cycle ensures that only one of $x_i, \overline{x}_j$ can be chosen. Note that $x_i$ and $\overline{x}_j$ must not attack each other directly because of well-formedness in the original CAF.
\end{proofsketch}
\begin{figure}[t]
	\centering
	\footnotesize
	\tikz[yscale=0.9, label distance=-0.5mm]{
		\node[arg, label={above}:$1$] (a1) at (1,4.25) {$a_1$};
		\node[arg, label={right}:$2$] (a2) at (5,4.25) {$\overline{a}_2$};
		\node[arg, label={above}:$a^1_{1,2}$] (h1) at (2,5) {$a^1_{1,2}$};
		\node[arg, label={above}:$a^2_{1,2}$] (h2) at (4,5) {$a^2_{1,2}$};
		\node[arg, label={above}:$a^3_{1,2}$] (h3) at (4,3.5) {$a^3_{1,2}$};
		\node[arg, label={above}:$a^4_{1,2}$] (h4) at (2,3.5) {$a^4_{1,2}$};
		\node[arg, label={right}:$2$] (b2) at (5,1.75) {$b_2$};
		\node[arg, label={below}:$3$] (b3) at (1,1.75) {$\overline{b}_3$};
		\node[arg, label={below}:$b^1_{2,3}$] (g1) at (4,2.5) {$b^1_{2,3}$};
		\node[arg, label={below}:$b^2_{2,3}$] (g2) at (2,2.5) {$b^2_{2,3}$};
		\node[arg, label={below}:$b^3_{2,3}$] (g3) at (2,1) {$b^3_{2,3}$};
		\node[arg, label={below}:$b^4_{2,3}$] (g4) at (4,1) {$b^4_{2,3}$};
		\node[arg, label={right}:$3$] (c3) at (3,0) {$c_3$};
		\node[arg, label={right}:$a^1_{1,2}$] (d1) at (6,5.5) {$\hat{a}^1_{1,2}$};
		\node[arg, label={right}:$a^2_{1,2}$] (d2) at (6,4.5) {$\hat{a}^2_{1,2}$};
		\node[arg, label={left}:$a^3_{1,2}$] (d3) at (0,3.5) {$\hat{a}^3_{1,2}$};
		\node[arg, label={left}:$a^4_{1,2}$] (d4) at (0,4.5) {$\hat{a}^4_{1,2}$};
		\node[arg, label={left}:$b^1_{2,3}$] (e1) at (0,2.5) {$\hat{b}^1_{2,3}$};
		\node[arg, label={left}:$b^2_{2,3}$] (e2) at (0,1.5) {$\hat{b}^2_{2,3}$};
		\node[arg, label={right}:$b^3_{2,3}$] (e3) at (6,0.5) {$\hat{b}^3_{2,3}$};
		\node[arg, label={right}:$b^4_{2,3}$] (e4) at (6,1.5) {$\hat{b}^4_{2,3}$};
		\draw[attack] 
		(a1) edge (h1)
		(h1) edge (h2)
		(h2) edge (a2)
		(a2) edge (h3)
		(h3) edge (h4)
		(h4) edge (a1)
		(b2) edge (g1)
		(g1) edge (g2)
		(g2) edge (b3)
		(b3) edge (g3)
		(g3) edge (g4)
		(g4) edge (b2)
		(d1) edge [dashed] (h2)
		(d2) edge [dashed, bend right] (a2)
		(d3) edge [dashed] (h4)
		(d4) edge [dashed] (a1)
		(e1) edge [dashed] (g2)
		(e2) edge [dashed] (b3)
		(e3) edge [dashed] (g4)
		(e4) edge [dashed, bend left] (b2)
		(a2) edge [dashed, bend left] (g1)
		(b2) edge [dashed, bend right] (h3)
		(c3) edge [dashed, bend right] (g3);
	}
	\caption{Reduction of 3-\SAT-instance $C_1 = \{a\}$, $C_2 = \{\neg a, b\}$, $C_3 = \{\neg b, c\}$, to an instance $(F',S)$ of $\Ver_{\stable,1}^\PCAF$ (cf.\ Proof of Proposition~\ref{prop:verificationStableAdmissibleComplete1}). Dashed arrows are attacks deleted from $F'$, i.e., they are edges in $\problematicPart(\red{1}(F'))$.}
	\label{fig:reductionVerficationImage1Stb}
\end{figure}

In fact, when applying Reduction~$1$, we lose the advantages of wfCAFs for \emph{all} investigated semantics, since also for the remaining semantics verification remains harder than in the case of wfCAFs.

\begin{restatable}{proposition}{verificationPrefSemiStageFirstImage} \label{prop:verificationSemiStablePref1}
	$\Ver_{\sigma,1}^\PCAF$ is $\SigmaP{2}$-complete for $\sigma \in \{\pref,$ $\semi, \stage\}$, even when considering only transitive preferences.
\end{restatable}
The proposition can be proven by adapting the standard translation for skeptical acceptance of preferred-semantics %
\cite[Reduction~3.7]{DvorakD17}. 

We now turn our attention to Reductions~$2$, $3$, and $4$. Since these reductions do not remove conflicts between arguments, it is easy to see that verification for conflict-free and naive semantics remains tractable. 

\begin{restatable}{proposition}{verificationConflictfreeNaiveImagesTwoThreeFour} \label{prop:verificationConflictfreeNaive2to4}
	$\Ver_{\sigma,i \in \{2,3,4\}}^\PCAF$ is in $\P$ for $\sigma \in \{\cf, \naive\}$. 
\end{restatable}
\begin{proofsketch}
	By Lemma~\ref{lemma:cfDoesNotChangeForReductions234}, given a PCAF $F$, it suffices to test whether $C$ is conflict-free (resp.\ naive) in the underlying CAF of $F$. This can be done in polynomial time for wfCAFs (cf.\ Table~\ref{tab:complexity_CAFwfCAF}). 
\end{proofsketch}

As it turns out, with Reductions~$2$, $3$, and $4$ we retain the benefits of wfCAFs over general CAFs for almost all investigated semantics with respect to computational complexity. In short, verification for wfCAFs is easier than on general CAFs because, given a wfCAF $F$ and a set of claims $C$, a set of arguments $S$ can be constructed in polynomial time such that $S$ is the unique maximal admissible set in $F$ with claim $\cl(S) = C$ \cite{DvorakW20}. Making use of the fact that Reductions~$2$, $3$, and~$4$ do not alter conflicts between arguments, we can construct such a maximal set of arguments also for PCAFs: given a PCAF $F$ and set $C$ of claims, 
we define the set $E_0(C)$ containing all arguments of $F$ with a claim in $C$;
the set $E_1^i(C)$ is obtained from $E_0(C)$ by removing all arguments attacked by $E_0(C)$ in the underlying CAF of $F$; 
finally, the set $E_*^i(C)$ is obtained by repeatedly removing all arguments not defended by $E_1^i(C)$ in $\red{i}(F)$ until a fixed point is reached.
Recall that $S^+_{(\Args,\Att)} = \{ a \mid (b,a)\in \Att, b\in S\}$ denotes the arguments attacked by $S$ in $(\Args,\Att)$. 

\begin{definition}%
	Given a PCAF $F=(\Args,\Att,\cl, \succ)$, a set of claims $C$, and $i \in \{2,3,4\}$, we define
	\begin{align*}
	E_0(C) =& \{ a\in \Args \mid \cl(a)\in C\};\\
	E_1^i(C) =& E_0(C) \setminus E_0(C)^+_{(\Args,\Att)};\\
	E_{k}^i(C) =& \{ x \in E_{k-1}^i(C) \mid x \text{ is defended by } E_{k-1}^i(C) \\ & \text{ in } \red{i}(F) \} \text{ for } k \geq 2; \\
	E_*^i(C) =& E_k^i \text{ for $k \geq 2$ such that } E_k^i(C) = E_{k-1}^i(C).
	\end{align*} 
\end{definition}

The above definition is based on \cite[Definition~5]{DvorakW20}, but with the crucial differences that undefended arguments are (a) computed w.r.t.\ $\red{i}(F)$ and (b) are iteratively removed until a fixed point is reached. The sets $E_1^i(C)$ and $E_*^i(C)$ have useful properties.

For conflict-free based semantics we observe that the conflicts are not affected by the reductions 
and thus one can use existing results for well-formed CAFs~\cite[Lemma~$1$]{DvorakW20} to obtain that $E_1^i(C)$ is the unique candidate for the maximal conflict-free set of arguments that realizes the claim set $C$.

\begin{restatable}{lemma}{conflictFreeLemma} \label{lemma:Im234ConflictFree}
	Let $F$ be a PCAF, $C$ be a set of claims and $i\! \in\! \{2,3,4\}$.
	We have that $C\! \in\! \cf_p^i(F)$ iff $\cl(E_1^i(C))\! =\! C$. 
	Moreover, if $C\! \in\! \cf_p^i(F)$ then $E_1^i(C)$ is the unique maximal conflict-free set $S$ in $\red{i}(F)$ such that $\cl(S)\! =\! C$.
\end{restatable}

For admissible semantics we are looking for a conflict-free set that defends all its arguments.
Thus we start from the conflict-free set $E_1^i(C)$. Notice that arguments that are not in $E_1^i(C)$ cannot be contained in any admissible set $S$ with $\cl(S)\! =\! C$. 
We can then obtain the maximal admissible set realizing $C$ in  $\red{i}(F)$ by iteratively removing arguments that are not defended by the current set of arguments. Once we reach a fixed-point we have an admissible set, but need to check whether we still cover all the claims of $C$.

\begin{restatable}{lemma}{admissibleLemma} \label{lemma:Im234Admissible}
	Let $F$ be a PCAF, $C$ be a set of claims and $i\! \in\! \{2,3,4\}$.
	We have that $C\! \in\! \adm_p^i(F)$ iff $\cl(E_*^i(C))\! =\! C$.
	Moreover, if $C\! \in\! \adm_p^i(F)$ then $E_*^i(C)$ is the unique maximal admissible set $S$ in $\red{i}(F)$ such that $\cl(S)\! =\! C$.
\end{restatable}

By computing the maximal conflict-free (resp.\ admissible) extensions 
$E_1^i(C)$ (resp.\ $E_*^i(C)$) 
for a set of claims $C$, the verification problem becomes easier for most semantics.

\begin{restatable}{proposition}{verificationStableAdmissibleCompleteImagesTwoThreeFour} \label{prop:verificationAdm2to4}
	$\Ver_{\sigma,i \in \{2,3,4\}}^\PCAF$ is in $\P$ for $\sigma \in \{\adm, \stable\}$. It is $\coNP$-complete for $\sigma \in \{\pref, \semi, \stage\}$, even when considering only transitive preferences.
\end{restatable}
\begin{proofsketch}
	Let $F = (\Args,\Att,\cl, \succ)$ be a PCAF, let $C$ be a set of claims, and let $i \in \{2,3,4\}$. We can compute $\red{i}(F)$, $E_1^i(C)$, and $E_*^i(C)$ in polynomial time. 
	
	For $\adm$, by Lemma~\ref{lemma:Im234Admissible}, it suffices to test whether $\cl(E_*^i(C)) = C$. 
	For $\stb$, we first check whether $C \in \adm_p^i(F)$. If not, ${C \not\in \stb_p^i(F)}$. If yes, then, by Lemma~\ref{lemma:Im234Admissible}, ${\cl(E_*^i(C)) = C}$. We can check in polynomial time if $E_*^i(C) \in \stb(\red{i}(F))$. If yes, we are done. If no, then there is an argument $x$ that is not in $E_*^i(C)$ but is also not attacked by $E_*^i(C)$ in $\red{i}(F)$. Moreover, there can be no other $S \in \stb(\red{i}(F))$ with $\cl(S) = C$ since for any such $S$ we would have $S \subseteq E_*^i(C)$, which would imply that $S$ does not attack $x$ and that $x \not \in S$. 
	
	The arguments for $\sigma \in \{\pref, \semi, \stage\}$ are similar, but some checks require $\coNP$-time. 
\end{proofsketch}

Regarding complete semantics, only Reduction~$3$ retains the benefits of wfCAFs. Here, 
the fact that Reductions~$2$ and~$4$ can introduce new attacks is enough to see an increase in complexity. 

\begin{figure}[t]
	\centering
	\tikz{
		\node[arg, label={below}:$a$] (a1) at (1.8,0) {$a$};
		\node[arg, label={below}:$a$] (a2) at (2.7,0) {$\overline{a}$};
		\node[arg, label={above}:$d_{a}$] (da1) at (0,1) {$d_a$};
		\node[arg, label={below}:$d'_{a}$] (da12) at (0.9,0.5) {$d'_a$};
		\node[arg, label={above}:$d'_{\overline a}$] (da22) at (0.9,1.15) {$d'_{\overline a}$};
		
		\node[arg, label={below}:$b$] (b1) at (4.5,0) {$b$};
		\node[arg, label={below}:$b$] (b2) at (5.4,0) {$\overline{b}$};
		\node[arg, label={above}:$d_{b}$] (db1) at (7.2,1) {$d_b$};
		\node[arg, label={above}:$d'_{b}$] (db12) at (6.3,1.15) {$d'_b$};
		\node[arg, label={below}:$d'_{\overline b}$] (db22) at (6.3,0.5) {$d'_{\overline b}$};
		
		\node[arg, label={above}:$c_1$] (c1) at (2.25,1.25) {$c_1$};
		\node[arg, label={above}:$c_2$] (c2) at (4.95,1.25) {$c_2$};
		\node[arg, label={left}:$\varphi$] (phi) at (3.6,2.25) {$\varphi$};
		\draw[attack] 
		
		(a1) edge[color=gray, ultra thick] (c1)
		(c1) edge[color=gray, ultra thick] (a1)
		(b1) edge[color=gray, ultra thick]  (c1)
		(c1) edge[color=gray, ultra thick]  (b1)
		(b2) edge[color=gray, ultra thick] (c2)
		(c2) edge[color=gray, ultra thick] (b2)
		(a2) edge[color=gray, ultra thick]  (c2)
		(c2) edge[color=gray, ultra thick]  (a2)
		(c1) edge[loop left] (c1)
		(c2) edge[loop left] (c2)
		(c1) edge (phi)
		(c2) edge (phi)
		
		(a1) edge[color=gray, ultra thick] (da12)
		(da12) edge[color=gray, ultra thick] (a1)
		(a2) edge[color=gray, ultra thick] (da22)
		(da22) edge[color=gray, ultra thick] (a2)
		(da12) edge (da1)
		(da22) edge (da1)
		
		(b1) edge[color=gray, ultra thick]  (db12)
		(db12) edge[color=gray, ultra thick]  (b1)
		(b2) edge[color=gray, ultra thick] (db22)
		(db22) edge[color=gray, ultra thick] (b2)
		(db12) edge (db1)
		(db22) edge (db1)
		;
	}
	\caption{$\red{4}(F)$ from the proof of Proposition~\ref{prop:verificationComp2to4}, $\varphi =  ((a \lor b) \allowbreak \land \allowbreak (\neg a \lor \neg b))$. Symmetric attacks drawn in gray and thick have been introduced by Reduction~$4$.}
	\label{fig:reductionVerficationImage4Complete}
\end{figure}

\begin{restatable}{proposition}{verificationCompleteImagesTwoThreeFour} \label{prop:verificationComp2to4}
	$\Ver_{\comp,3}^\PCAF$ is in $\P$. $\Ver_{\comp,i \in \{2,4\}}^\PCAF$ is $\NP$-complete, even for transitive preferences.
\end{restatable}
\begin{proofsketch}
	$\P$-membership for $\Ver_{\comp,3}^\PCAF$ is similar to the proof of Proposition~\ref{prop:verificationAdm2to4}. We demonstrate $\NP$-hardness of $\Ver_{\comp,4}^\PCAF$. 
	Let $\varphi$ be an arbitrary instance of 3-\SAT\ given as a set $C$ of clauses over variables $X$
	and let $\overline{X} = \{\overline x \mid x \in X\}$. 
	We construct a PCAF $F = (\Args,\Att,\cl, \succ)$ as well as a set of claims $S = X \cup \{\varphi\}$: 
	\begin{itemize}
		\item $\Args = \{\varphi\} \cup C \cup X \cup \overline{X} \cup \{d_x \mid x\in X\} \cup \{d'_x \mid x\in X \cup \overline X\}$;
		\item $\Att =  \{(c,\varphi)\mid c\in C\} \cup \{(c,c)\mid c\in C\}\ \cup$\\ 
			\phantom{$\Att =\ $}%
			$\{(c,x) \mid x \in c, c \in C\}\ \cup \{(c,\overline{x}) \mid \neg x \in c, c \in C\}\ \cup$\\
			\phantom{$\Att =\ $}%
			$\{(d'_x,x) \mid x\! \in\! X \cup \overline{X}\} \cup \{(d'_x,d_x),(d'_{\overline{x}},d_x)\mid x\! \in\! X \}$;
		\item $\cl(x)=\cl(\overline{x})= x$ for $x \in X$,\\ $\cl(v) = v$ otherwise;
		\item $x \succ c$, $x \succ d'_x$ for all $x \in X \cup \overline{X}$ and all $c \in C$. 
	\end{itemize}
	Figure~\ref{fig:reductionVerficationImage4Complete} illustrates the above construction. 
	It can be verified that $\varphi$ is satisfiable iff $S \in \comp_c(\red{4}(F))$.
\end{proofsketch}

Table~\ref{table:complexityResults} summarizes our complexity results. Reduction~3 preserves the lower complexity of wfCAFs for all investigated semantics, while Reductions~2 and~4 preserve the lower complexity for all but complete semantics. Reduction~1 does not preserve the advantages of wfCAFs, and rather exhibits the full complexity as general CAFs.
Notice that the lower complexity of the verification problem is crucial for enumerating extensions. In particular, the improved enumeration algorithm for wfCAFs~\cite{DvorakW20} is based on the polynomial time verification 
of claim-sets 
and thus extends to PCAFs under~Reductions 2--4.

\begin{table}
	\centering
	\begin{tabular}{c|ccc}
		$\sigma$ & $i = 1$ & $i \in \{2,4\}$ & $i = 3$\\ 
		\hline
		$\cf/\adm/\naive/\stb$ & $\NP$-c & in $\P$ & in $\P$  \\
		$ \comp$ & $\NP$-c & $\NP$-c & in $\P$ \\
		$\pref/\semi/\stage$ & $\SigmaP{2}$-c & $\coNP$-c & $\coNP$-c 
	\end{tabular}
	\caption{Complexity of $\Ver_{\sigma,i}^\PCAF$. Results also hold when considering only PCAFs with transitive preferences.
	}
	\label{table:complexityResults}
\end{table}

\section{Conclusion} \label{sec:conclusion}

Many approaches to structured argumentation 
(i) assume that arguments with the same claims attack the same arguments
and (ii) take preferences into account.
Investigations on claim-augmented argumentation frameworks (CAFs) 
so far only consider (i), showing that the resulting subclass of well-formed
CAFs (wfCAFs) has several desired properties. The research question of this paper is to analyze
whether these properties carry over when preferences are taken into account, since 
the incorporation of preferences can violate the syntactical restriction of wfCAFs.

To this end, we introduced preference-based claim-augmented argumentation frameworks (PCAFs)
and investigated the impact of the four preference reductions commonly used in abstract argumentation when applied to PCAFs.
In particular, we examined and characterized CAF-classes that result from applying 
these %
reductions to PCAFs, and furthermore investigated the fundamental properties of I-maximality and computational complexity for PCAFs. 
Preserving I-maximality is desirable since it implies intuitive behavior of maximization-based semantics, %
while the complexity of the verification problem %
is crucial for the enumeration of claim-extensions.
Insights in terms of both semantical and computational properties provide necessary foundations towards a practical realization of this particular argumentation paradigm (we refer to, e.g., \cite{BaumeisterJNNR21,FazzingaFF20}, for a similar research endeavor in terms of incomplete AFs).

Our results show that 
(1) Reduction~$3$, the most conservative of the four reductions, exhibits the same properties as wfCAFs regarding computational complexity while also preserving I-maximality for most of the semantics; 
(2) Reductions~$2$ and~$4$ retain the advantages of wfCAFs regarding complexity for all but complete semantics, but do not preserve I-maximality for any investigated semantics; %
(3) under Reduction~$1$, neither complexity properties nor I-maximality are preserved. The above results hold even if we restrict ourselves to transitive preferences.
It is worth noting that Reduction~$3$ behaves favorably on regular AFs as well, fulfilling many principles for preference-based semantics~\cite[Table~1]{KaciTV18}. 

A possible direction for future work is to lift the preference ordering over arguments to sets of arguments and select extensions in this way. This has been investigated for regular AFs in combination with Reduction~$2$ \cite{AmgoudV14}. 
Another direction is to extend our studies to alternative semantics for CAFs 
\cite{DvorakRW20c, DvorakGRW21}, where
subset-maximization is handled on the claim-level instead of on the argument-level.

\bibliographystyle{apalike}
\bibliography{argumentation.bib}

\clearpage 

\appendix

\section{Proofs}

Here we gather full proofs for our results. 

\subsection{Proofs for Section~\ref{sec:characterization}}

\fourImagesLieBetweenCAFandwfCAF* 
\begin{proof}
	Let $i \in \{1,2,3,4\}$. $\setOfAllwfCAFs \subseteq \Image{i}$ follows from the fact that any $(\Args, \Att, \cl) \in \setOfAllwfCAFs$ can be obtained via Reduction~$i$ from the PCAF $(\Args, \Att, \cl, \emptyset)$. 
	
	$\setOfAllwfCAFs\! \subset\! \Image{i}$: consider the PCAF $F = (\{a,b\},\allowbreak\{(a,a),\allowbreak(a,b),\allowbreak(b,a),\allowbreak(b,b)\}, \cl,\allowbreak \succ)$ with $\cl(a) = \cl(b)$, and $b \succ a$. For all $i \in \{1,2,3,4\}$ we have $\red{i}(F) = (\{a,b\}, \{(a,a),(b,a),(b,b)\}, \cl)$, i.e., the resulting CAF $\red{i}(F)$ is not well-formed. 
	
	$\Image{i} \subset \setOfAllCAFs$: Towards a contradiction, assume there is a PCAF $F = (\Args, \Att, \cl, \succ)$ such that $\red{i}(F) = (\Args, \Def, \cl)$ with $(a,b), (b,a) \in \Def$ but $(a,a), (b,b) \not\in \Def$ for some $a,b \in \Args$ with $\cl(a) = \cl(b)$. This means that either $(a,b) \in \Att$ or $(b,a) \in \Att$, since none of four reductions can introduce the attacks $(a,b)$ and $(b,a)$ at the same time. By symmetry, we only look at the case that $(a,b) \in \Att$. Then, since $(\Args, \Att, \cl)$ is well-formed and since $\cl(a) = \cl(b)$, $(b,b) \in \Att$. But $\succ$ is non-reflexive, i.e., $(b,b)$ is not removed by Reduction~$i$ and therefore $(b,b) \in \Def$. Contradiction.
\end{proof}

\imagesAreIncomparable*
\begin{proof}
	We show the various statements separately.
	\begin{itemize}
		\item $\Image{1} \not\subseteq \Image{j}$ with $j \in \{2,3,4\}$: let $F$ be the CAF shown in Figure~\ref{fig:cafOnlyInImage1Appendix}. $F$ is in $\Image{1}$ as it can be obtained by applying Reduction~$1$ to the PCAF $(\Args, \Att, \cl, \succ)$ with $\Att = \{(a,b),(b,b)\}$ and $b \succ a$. Towards a contradiction, assume there is a PCAF $F'$ such that $\red{j}(F') = F$. Since self-attacks can not be removed by any of the four reductions, $(b,b) \in F'$. Since the underlying CAF of $F'$ must be well-formed, also $(a,b) \in F'$. But then, by the definition of Reduction~$j$, either $(a,b) \in \red{j}(F')$ or $(b,a) \in \red{j}(F')$. Contradiction.
		
		\item $\Image{2} \not\subseteq \Image{j}$ with $j \in \{1,3,4\}$: let $F$ be the CAF shown in Figure~\ref{fig:cafOnlyInImage2Appendix}. $F$ is in $\Image{2}$ as it can be obtained by applying Reduction~$2$ to the PCAF $(\Args, \Att, \cl, \succ)$ with $\Att = \{(a,b),(b,b)\}$ and $b \succ a$. Towards a contradiction, assume there is a PCAF $F'$ such that $\red{j}(F') = F$. Then  $(b,b) \in F'$ and therefore also $(a,b) \in F'$. But $(b,a) \not\in F'$, since $(a,a) \not\in F$ and therefore also $(a,a) \not\in F'$. But Reductions~$1$ and~$3$ can not introduce $(b,a)$ in this case, while Reduction~$4$ can not introduce $(b,a)$ without retaining $(a,b)$.
		
		\item $\Image{3} \subset \Image{j}$ with $j \in \{1,2,4\}$: let $F$ be any CAF in $\Image{3}$. Then there is a PCAF $F' = (A,R',\cl,\succ)$ such that $\red{3}(F') = F$. If $(a,b) \in F'$ and $(a,b) \in F$ we can assume that $b \not\succ a$ without loss of generality. If $(a,b) \in F'$ but $(a,b) \not\in F$, then, by definition of Reduction~$3$, $(b,a) \in F'$ and $b \succ a$. In this case, Reduction~$j$ functions in the same way as Reduction~3 (cf.\ Definition~\ref{def:reductions} and Figure~\ref{fig:reductionVisualization}), i.e., $\red{j}(F') = F$. This proves $\Image{3} \subseteq \Image{j}$. $\Image{3} \subset \Image{j}$ follows from $\Image{j} \not\subseteq \Image{3}$.
		
		\item $\Image{4} \not\subseteq \Image{j}$ with $j \in \{1,2,3\}$: let $F$ be the CAF shown in Figure~\ref{fig:cafOnlyInImage4Appendix}. $F$ is in $\Image{4}$ as it can be obtained by applying Reduction~$4$ to the PCAF $(\Args, \Att, \cl, \succ)$ with $\Att = \{(a,b),(b,b)\}$ and $b \succ a$. Towards a contradiction, assume there is a PCAF $F'$ such that $\red{j}(F') = F$. Then  $(b,b) \in F'$ and therefore also $(a,b) \in F'$. But $(b,a) \not\in F'$, since $(a,a) \not\in F'$. But Reduction~$1$, $2$~and~$3$ can not introduce $(b,a)$, at least not without deleting $(a,b)$. \qedhere
	\end{itemize}
\end{proof}
\begin{figure}[ht]
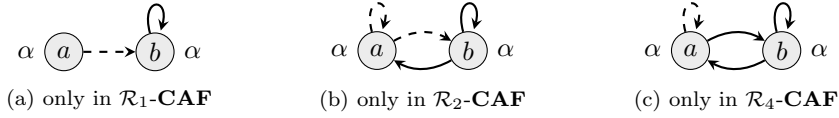

	\centering
	\begin{subfigure}{0.32\columnwidth}
		\centering
		\tikz{
			\node[arg, label={left}:$\alpha$] (a) at (0,0) {$a$};
			\node[arg, label={right}:$\alpha$] (b) at (1.2,0) {$b$};
			\draw[attack] 
			(a) edge [dashed] (b)
			(b) edge [loop above] (b);
		}
		\caption{only in $\Image{1}$}
		\label{fig:cafOnlyInImage1Appendix}
	\end{subfigure}
	\hfill
	\begin{subfigure}{0.32\columnwidth}
		\centering
		\tikz{
			\node[arg, label={left}:$\alpha$] (a) at (0,0) {$a$};
			\node[arg, label={right}:$\alpha$] (b) at (1.2,0) {$b$};
			\draw[attack] 
			(b) edge [bend left] (a)
			(a) edge [bend left, dashed] (b)
			(a) edge [dashed, loop above] (a)
			(b) edge [loop above] (b);
		}
		\caption{only in $\Image{2}$}
		\label{fig:cafOnlyInImage2Appendix}
	\end{subfigure}
	\hfill
	\begin{subfigure}{0.32\columnwidth}
		\centering
		\tikz{
			\node[arg, label={left}:$\alpha$] (a) at (0,0) {$a$};
			\node[arg, label={right}:$\alpha$] (b) at (1.2,0) {$b$};
			\draw[attack] 
			(a) edge [bend left] (b)
			(b) edge [bend left] (a)
			(b) edge [loop above] (b)
			(a) edge [dashed, loop above] (a);
		}
		\caption{only in $\Image{4}$}
		\label{fig:cafOnlyInImage4Appendix}
	\end{subfigure}
	\caption{CAFs that are contained only in $\Image{1}$, $\Image{2}$, and $\Image{4}$ respectively. Dashed arrows are edges in the wf-problematic part of the respective CAF.}
	\label{fig:imagesAreIncomparableAppendix}
\end{figure}

\characterizations*
\begin{proof}
	See below (Lemmas~\ref{prop:characterizationFirstImage}, \ref{prop:characterizationSecondImage}, \ref{prop:characterizationThirdImage}, and \ref{prop:characterizationFourthImage}).
\end{proof}

\begin{restatable}{lemma}{characterizationFirstImage} \label{prop:characterizationFirstImage}
	Let $F = (A,R,\cl)$ be a CAF. $F \in \Image{1}$ iff
	$(a,b) \in \problematicPart(F)$ implies $(b,a) \not\in \problematicPart(F)$.
\end{restatable}
\begin{proof}
	``$\implies$'': By contrapositive. Suppose there is $(a,b) \in \problematicPart(F)$ such that $(b,a) \in \problematicPart(F)$. Towards a contradiction, assume $F \in \Image{1}$. Then there is a PCAF $F' = (A,R',\cl,\succ)$ such that $\red{1}(F') = F$. Since Reduction~$1$ can only delete but not introduce attacks, and since the underlying CAF of $F'$ must be well-formed, $(a,b) \in R'$ and $(b,a) \in R'$. However, then also $(b \succ a)$ and $(a \succ b)$ which means that $F'$ is not asymmetric. Contradiction.
	
	``$\impliedby$'': Suppose that $(a,b) \in \problematicPart(F)$ implies $(b,a) \not\in \problematicPart(F)$. Then $\red{1}(F') = F$ for the PCAF $F' = (A,R',\cl,\succ)$ with $R' = R \cup \{(a,b) \mid (a,b) \in \problematicPart(F)\}$ as well as $a \succ b \iff (b,a) \in R' \setminus R$. The underlying CAF of $F'$ is well-formed since $\problematicPart((A,R',\cl)) = \emptyset$. Furthermore, $\succ$ is asymmetric since $(a,b) \in \problematicPart(F)$ implies $(b,a) \not\in \problematicPart(F)$ and by construction of $F'$. 
\end{proof}

\begin{restatable}{lemma}{characterizationSecondImage} \label{prop:characterizationSecondImage}
	Let $F = (A,R,\cl)$ be a CAF. $F \in \Image{2}$ iff
	there are no arguments $a,a',b,b'$ in $F$ with $\cl(a) = \cl(a')$ and $\cl(b) = \cl(b')$ such that $(a,b) \in \problematicPart(F)$, $(b,a) \not\in R$, $(a',b) \in R$, and either $(b,a') \in R$ or $((a',b') \not\in R$ and $(b',a') \not\in R)$.
\end{restatable}
\begin{proof}
	``$\implies$'': By contrapositive. Suppose that there are $a,a',b,b' \in A$ with $\cl(a') = \cl(a)$ and $\cl(b') = \cl(b)$ such that $(a,b) \in \problematicPart(F)$,  $(b,a) \not\in R$, $(a',b) \in R$, and either $(b,a') \in R$ or $((a',b') \not\in R$ and $(b',a') \not\in R)$. Towards a contradiction, assume that $F \in \Image{2}$. Then there must be a PCAF $F' = (A,R',\cl,\succ)$ such that $\red{2}(F') = F$. Reduction~$2$ can not completely remove conflicts between arguments. Since there is no conflict between $a$ and $b$ in $F$ there can be no conflict in $F'$ either, i.e., $(a,b) \not\in R'$ and $(b,a) \not\in R'$. Therefore, since the underlying CAF $(A,R',\cl)$ of $F'$ must be well-formed, $(a',b) \not \in R'$. Since $(a',b) \in R$ it must be that $(b,a') \in R'$ and $a' \succ b$. Then $(b,a') \not\in \red{2}(F')$. Furthermore, by the well-formedness of $(A,R',\cl)$, we have that $(b',a') \in R'$ and therefore either $(a',b') \in \red{2}(F')$ or $(b',a') \in \red{2}(F')$. Contradiction to $\red{2}(F') = F$.
	
	``$\impliedby$'': Our underlying assumption is that there are no arguments $a,a',b,b'$ in $F$ with $\cl(a) = \cl(a')$ and $\cl(b) = \cl(b')$ such that $(a,b) \in \problematicPart(F)$, $(b,a) \not\in R$, $(a',b) \in R$, and either $(b,a') \in R$ or $((a',b') \not\in R$ and $(b',a') \not\in R)$. We will construct a PCAF $F'' = (A,R'',\cl,\succ)$ such that $\red{2}(F'') = F$. 
	
	But first, as an intermediate step, we construct the CAF $F' = (A,R',\cl)$. We say that $(b,a)$ is forced in $F$ if $(a,b) \in R$ and if there is an argument $a'$ with $\cl(a') = \cl(a)$ such that $(a',b) \not\in R$ and $(b,a')\not\in R$. Observe that if $(b,a)$ is forced in $F$, then $(a,b)$ cannot be forced in $F$ by our underlying assumption. Furthermore, if $(b,a)$ is forced in $F$, then $(b,a) \not\in R$, again by our underlying assumption. We construct $R' = (R \cup \{(b,a) \mid (b,a) \text{ is forced in } F\}) \setminus \{(a,b) \mid (b,a) \text{ is forced in } F\}$. Note that $(a,b) \in \problematicPart(F')$ implies $(b,a) \in R'$ for all arguments $a,b$: towards a contradiction, assume otherwise. Then there is some $(a,b) \in \problematicPart(F')$ such that $(b,a) \not\in R'$. Then $(a,b) \not\in R$ and $(b,a) \not\in R$ by construction of $R'$. Furthermore, since $(a,b) \in \problematicPart(F')$, there must be some $a'$ with $\cl(a') = \cl(a)$ and $(a',b) \in R'$. It cannot be that $(a',b) \in R$, otherwise $(b,a')$ would be forced in $F$ and $(a',b) \not\in R'$. Thus, $(b,a') \in R$ and $(a',b)$ was added to $R'$ because it is forced in $F$. But this is only possible if there is some $b'$ with $\cl(b') = \cl(b)$ and $(a',b') \not\in R$ and $(b',a') \not\in R$. This contradicts our underlying assumption: $(b',a') \in \problematicPart(F)$, $(a',b') \not\in R$, $(b,a') \in R$, $(a,b) \not\in R$, and $(b,a) \not\in R$.
	
	Now we construct $R'' = R' \cup \{(a,b) \mid (a,b) \in \problematicPart(F')\}$. Furthermore, $b \succ a \iff (a,b) \in R'' \setminus R$. Finally, $F'' = (A,R'',\cl,\succ)$. The underlying CAF of $F''$ is well-formed since $\problematicPart((A,R'',\cl)) = \emptyset$ by construction. Moreover, $\succ$ is asymmetric since if $(a,b) \in R''$ and $(b,a) \in R''$ then, by construction of $R'$ and $R''$, either $(a,b) \in R$ or $(b,a) \in R$. Lastly, we show that $\red{2}(F'') = F$: 
	if $(a,b) \in R'' \setminus R$, then we defined $b \succ a$ and thus $(a,b) \not\in \red{2}(F'')$. 
	If $(a,b) \in R \setminus R''$, then $(b,a)$ was forced in $F$, i.e., $(b,a) \not\in R$ but $(b,a) \in R'$ and therefore also $(b,a) \in R''$. Thus, we define $a \succ b$ which means that $(a,b) \in \red{2}(F'')$.
\end{proof}

\begin{restatable}{lemma}{characterizationThirdImage} \label{prop:characterizationThirdImage}
	Let $F = (A,R,\cl)$ be a CAF. $F \in \Image{3}$ iff
	$(a,b) \in \problematicPart(F)$ implies $(b,a) \in R$.
\end{restatable}
\begin{proof}
	``$\implies$'': By contrapositive. Suppose there is $(a,b) \in \problematicPart(F)$ such that $(b,a) \not\in R$. Towards a contradiction, assume $F \in \Image{3}$. Then there is a PCAF $F' = (A,R',\cl,\succ)$ such that $\red{3}(F') = F$. 
	Since Reduction~$3$ can only delete but not introduce attacks, and since $(A,R',\cl)$ must be well-formed, $(a,b) \in R'$. However, Reduction~$3$ can not completely remove conflicts between arguments, i.e., either $(a,b) \in \red{3}(F')$ or $(b,a) \in \red{3}(F')$. Contradiction.
	
	``$\impliedby$'': Suppose $(a,b) \in \problematicPart(F)$ implies $(b,a) \in R$. Then $\red{3}(F') = F$ for the PCAF $F' = (A,R',\cl,\succ)$ with $R' = R \cup \{(a,b) \mid (a,b) \in \problematicPart(F)\}$ as well as ${a \succ b} \iff (b,a) \in R' \setminus R$. $(A,R',\cl)$ is well-formed since $\problematicPart((A,R',\cl)) = \emptyset$. Furthermore, $\succ$ is asymmetric by construction. 
\end{proof}

\begin{restatable}{lemma}{characterizationFourthImage} \label{prop:characterizationFourthImage}
	Let $F = (A,R,\cl)$ be a CAF. $F \in \Image{4}$ iff
	there are no arguments $a,a',b,b'$ in $F$ with $\cl(a) = \cl(a')$ and $\cl(b) = \cl(b')$ such that $(a,b) \in \problematicPart(F)$, $(b,a) \not\in R$, $(a',b) \in R$, and either $(b,a') \not\in R$ or $((a',b') \not\in R$ and $(b',a') \not\in R)$.
\end{restatable}
\begin{proof}
	Similar to the proof of Lemma~\ref{prop:characterizationSecondImage}:
	
	``$\implies$'': By contrapositive. Suppose that there are $a,a',b,b' \in A$ with $\cl(a') = \cl(a)$ and $\cl(b') = \cl(b)$ such that $(a,b) \in \problematicPart(F)$, $(b,a) \not\in R$, $(a',b) \in R$, and either $(b,a') \not\in R$ or $((a',b') \not\in R$ and $(b',a') \not\in R)$. Towards a contradiction, assume that $F \in \Image{4}$. Then there must be a PCAF $F' = (A,R',\cl,\succ)$ such that $\red{4}(F') = F$. Reduction~$4$ can not completely remove conflicts between arguments. Since there is no conflict between $a$ and $b$ in $F$ there can be no conflict in $F'$ either, i.e., $(a,b) \not\in R'$ and $(b,a) \not\in R'$. Therefore, since the underlying CAF of $F'$ must be well-formed, $(a',b) \not \in R'$. The only way to obtain $(a',b) \in R$ from $(a',b) \not \in R'$ via Reduction~$4$ is to have $(b,a') \in R'$ and $a' \succ b$. Then $(b,a') \in \red{4}(F')$. Furthermore, by the well-formedness of $(A,R',\cl)$, we have that $(b',a') \in R'$ and therefore either $(a',b') \in \red{4}(F')$ or $(b',a') \in \red{4}(F')$. Contradiction to $\red{4}(F') = F$.
	
	``$\impliedby$'': Our underlying assumption is that there are no arguments $a,a',b,b'$ in $F$ with $\cl(a) = \cl(a')$ and $\cl(b) = \cl(b')$ such that $(a,b) \in \problematicPart(F)$, $(b,a) \not\in R$, $(a',b) \in R$, and either $(b,a') \not\in R$ or $((a',b') \not\in R$ and $(b',a') \not\in R)$. We will construct a PCAF $F'' = (A,R'',\cl,\succ)$ such that $\red{4}(F'') = F$. 
	
	But first, as an intermediate step, we construct the CAF $F' = (A,R',\cl)$. We say that $(b,a)$ is forced in $F$ if $(a,b) \in R$, $(b,a) \in R$, and if there is an argument $a'$ with $\cl(a') = \cl(a)$ such that $(a',b) \not\in R$ and $(b,a')\not\in R$. Observe that if $(b,a)$ is forced in $F$, then $(a,b)$ cannot be forced in $F$ by our underlying assumption. We construct $R' = R \setminus \{(a,b) \mid (b,a) \text{ is forced in } F\}$. Note that $(a,b) \in \problematicPart(F')$ implies $(b,a) \in R'$ for all arguments $a,b$: 
	towards a contradiction, assume otherwise. Then there is some $(a,b) \in \problematicPart(F')$ such that $(b,a) \not\in R'$. Then $(a,b) \not\in R$ and $(b,a) \not\in R$ by construction of $R'$. Furthermore, there must be some $a'$ with $\cl(a') = \cl(a)$ and $(a',b) \in R'$. It cannot be that $(a',b) \in R$ and $(b,a') \in R$, otherwise $(b,a')$ would be forced in $F$ and $(a',b) \not\in R'$. Thus, $(a',b) \in R$ and $(b,a') \not\in R$ by construction of $F'$. But this contradicts our underlying assumption: $(a,b) \in \problematicPart(F)$, $(b,a) \not\in R$, $(a',b) \in R$, and $(b,a') \not\in R$. 
	
	Now we construct $R'' = R' \cup \{(a,b) \mid (a,b) \in \problematicPart(F')\}$. 
	Furthermore, $b \succ a \iff (a,b) \in R'' \setminus R$ or $(b,a) \in R \setminus R''$. Finally, $F'' = (A,R'',\cl,\succ)$. The underlying CAF of $F''$ is well-formed since $\problematicPart((A,R'',\cl)) = \emptyset$ by construction. 
	Moreover, $\succ$ is asymmetric: if $b \succ a$, there are two cases. 
	\begin{enumerate}
		\item $(a,b) \in R'' \setminus R$. Clearly, $(a,b) \not\in R \setminus R''$. Moreover, $(a,b) \in R'' \setminus R$ implies $(b,a) \in R$ since we did not add attacks to $R''$ if there was no conflict between these attacks in $R$. Thus, $(b,a) \not\in R'' \setminus R$. We can conclude $a \not\succ b$.
		\item $(b,a) \in R \setminus R''$. Clearly, $(b,a) \not\in R'' \setminus R$. Moreover, $(b,a) \in R \setminus R''$ implies $(a,b) \in R''$, since we never completely removed conflicts when constructing $R''$ from $R$. Thus, $(a,b) \not\in R \setminus R''$. We can conclude $a \not\succ b$.
	\end{enumerate}
	Lastly, we show that $\red{4}(F'') = F$: 
	if $(a,b) \in R'' \setminus R$, then we defined $b \succ a$. As above, $(a,b) \in R'' \setminus R$ implies $(b,a) \in R$. The only possible reason for why we added $(a,b)$ to $R''$ is because $(a,b) \in \problematicPart(F')$. As previously discussed, this means that $(b,a) \in R'$ and therefore also $(b,a) \in R''$. Thus, $(a,b) \not\in \red{4}(F'')$.
	If $(a,b) \in R \setminus R''$, then $a \succ b$. As above, this implies $(b,a) \in R''$, and therefore $(a,b) \in \red{4}(F'')$.
\end{proof}

\imagesAreIncomparableTransitive*
\begin{proof}
	For $i \in \{1,2,4\}$, showing $\ImageTrans{i} \not\subseteq \ImageTrans{j}$ can be done in the same way as showing $\Image{i} \not\subseteq \Image{j}$ in the proof of Proposition~\ref{prop:imagesAreIncomparable} (note that the preference relations of the PCAFs associated with the CAFs of Figure~\ref{fig:imagesAreIncomparableAppendix} are transitive). It remains to show $\ImageTrans{3} \not\subseteq \ImageTrans{j}$ with $j \in \{1,2,4\}$. 
	
	Let $F$ be the CAF shown in Figure~\ref{fig:cafOnlyInImage3Trans}. $F$ is in $\ImageTrans{3}$: a PCAF that reduces to $F$ must have $c \succ b$ and $b \succ a$, which forces $c \succ a$ by transitivity. However, the attack $(a,c)$ is not deleted by Reduction~3 if there is no attack $(c,a)$.
	
	$F$ is not in $\ImageTrans{1}$ since a PCAF that reduces to $F$ would need to have $c \succ b$, $b \succ a$, and therefore also $c \succ a$. But Reduction~1 would delete the attack $(a,c)$.
		
	Towards a contradiction, assume there is a PCAF $F'$ such that $\red{2}(F') = F$. First, we show that $(a,b) \in F'$, $(b,a) \in F'$, $(b,c) \in F'$, and $(c,b) \in F'$.
	\begin{itemize}
		\item Assume $(a,b) \not\in F'$. Then two things most hold. Firstly, it must be that $(b,a) \in F'$, otherwise $(b,a) \not\in F$. Secondly, $(a',b) \not\in F'$, otherwise the underlying CAF of $F'$ would not be well-formed. This means that $(a',b)$ must have been introduced into $F$ by applying Reduction~$2$, i.e., by reversing $(b,a')$. Therefore, $(b,a') \in F'$. But then also $(b',a') \in F'$, otherwise the underlying CAF of $F'$ is not well-formed. But then, by the definition of Reduction~$2$, either $(b',a') \in F$ or $(a',b') \in F$, which is not the case. Contradiction. 
		\item Assume $(b,a) \not\in F'$. Then, since the underlying CAF of $F'$ must be well-formed, $(b',a) \not\in F'$. This means $(a,b') \in F'$, otherwise we cannot obtain $F$ from $F'$ via Reduction~$2$. This means that $(a',b') \in F'$, which is not possible since neither $(a',b') \in F$ nor $(b',a') \in F'$.
		\item Assume $(b,c) \not\in F'$. Then two things most hold. Firstly, it must be that $(c,b) \in F'$, otherwise $(c,b) \not\in F$. Secondly, $(b',c) \not\in F'$, otherwise the underlying CAF of $F'$ would not be well-formed. This means that $(b',c)$ must have been introduced into $F$ by applying Reduction~$2$, i.e., by reversing $(c,b')$. Therefore, $(c,b') \in F'$. But then also $(c',b') \in F'$, otherwise the underlying CAF of $F'$ is not well-formed. But then, by the definition of Reduction~$2$, either $(c',b') \in F$ or $(b',c') \in F$, which is not the case. Contradiction. 
		\item Assume $(c,b) \not\in F'$. Then, since the underlying CAF of $F'$ must be well-formed, $(c',b) \not\in F'$. This means $(b,c') \in F'$, otherwise we cannot obtain $F$ from $F'$ via Reduction~$2$. This means that $(b',c') \in F'$, which is not possible since neither $(b',c') \in F$ nor $(c',b') \in F'$.
	\end{itemize}
	Since $(a,b) \in F'$, $(b,a) \in F'$, $(b,c) \in F'$, and $(c,b) \in F'$, the only way to obtain $F$ form $F'$ via Reduction~$2$ is to set $c \succ b$ and $b \succ a$. But then $c \succ a$ which means that $(a,c) \not\in F$. Contradiction, i.e., $F \not\in \ImageTrans{2}$.
	
	Now assume there is a PCAF $G'$ such that $\red{4}(G') = F$. It must be that $(a,b) \in G'$ since we can not obtain $(a',b) \in \red{4}(G')$ and $(b,a') \not\in \red{4}(G')$ without $(a',b) \in G'$. Analogously, it must be that $(b,c) \in G'$. Then in order to have $\red{4}(G') = F$ we need to set $c \succ b$ and $b \succ a$. But then $c \succ a$ which means that it can not be that $(a,c) \in \red{4}(G')$ and $(c,a) \not\in \red{4}(G')$. Contradiction, i.e., $F \not\in \ImageTrans{4}$. 
\end{proof}
\begin{figure}[ht]
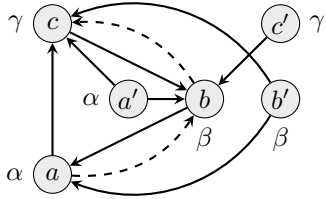

	\centering
	\tikz[rotate=90]{
		\node[arg, label={left}:$\alpha$] (a) at (0,3) {$a$};
		\node[arg, label={left}:$\alpha$] (b) at (1,2) {$a'$};
		\node[arg, label={below}:$\beta$] (c) at (1,1) {$b$};
		\node[arg, label={below}:$\beta$] (d) at (1,0) {$b'$};
		\node[arg, label={left}:$\gamma$] (e) at (2,3) {$c$};
		\node[arg, label={right}:$\gamma$] (f) at (2,0) {$c'$};
		\draw[attack] 
		(a) edge (e)
		(b) edge (c)
		(b) edge (e)
		(c) edge (a)
		(d) edge [bend angle=40, bend left] (a)
		(d) edge [bend angle=40, bend right] (e)
		(e) edge (c)
		(a) edge [bend right, dashed] (c)
		(c) edge [bend right, dashed] (e)
		(f) edge (c);
	}
	\caption{A CAF which shows that $\ImageTrans{3} \not\subseteq \ImageTrans{j}$ for $j \in \{1,2,4\}$. Dashed arrows are edges in the wf-problematic part.}
	\label{fig:cafOnlyInImage3Trans}
\end{figure}

\characterizationFirstImageTransitive*
\begin{proof}
	Let $F = (\Args, \Att, \cl)$.
	\begin{itemize}
		\item Suppose $\problematicPart(F)$ is acyclic and there is no $(a,b) \in F$ with a path from $a$ to $b$ in $\problematicPart(F)$. Construct the PCAF $F' = (\Args, \Att', \cl, \succ)$ with $\Att' = \Att \cup \{(a,b) \mid (a,b) \in \problematicPart(F)\}$ and $b \succ a$ iff there is a path from $a$ to $b$ in $\problematicPart(F)$. $(\Args, \Att', \cl)$ is well-formed by construction. $\succ$ is transitive because if there is a path from $a$ to $b$ and from $b$ to $c$, then there is also a path from $a$ to $c$. 
		$\succ$ is asymmetric because otherwise there would be a path from $a$ to $b$ and from $b$ to $a$, which again would mean that there is a cycle. It remains to show that $\red{1}(F') = F$. Let $(a,b)$ be any attack in $F'$. We distinguish two cases:
		\begin{itemize}
			\item $(a,b) \in F$. Then, since there is no path from $a$ to $b$ in $\problematicPart(F)$, $b \not\succ a$. Therefore, $(a,b) \in \red{1}(F')$.
			\item $(a,b) \not\in F$. Then, by construction, $(a,b) \in \problematicPart(F)$ and therefore $b \succ a$. Thus, $(a,b)$ is removed from $F'$ by Reduction~$1$, i.e., $(a,b) \not\in \red{1}(F')$.  
		\end{itemize}
		Note also that, by construction of $F'$, there can be no $(a,b) \in F$ such that $(a,b) \not\in F'$.
		
		\item Suppose $\problematicPart(F)$ is cyclic. Then there are arguments $x_1,\ldots,x_n \in F$ such that $x_1 = x_n$ and $(x_i, x_{i+1}) \in \problematicPart(F)$ for all $1 \leq i < n$. Towards a contradiction, assume there is a PCAF $F' = (\Args, \Att', \cl, \succ)$ such that $\red{1}(F') = F$. Then $(x_i, x_{i+1}) \in F'$ for all $1 \leq i < n$, otherwise $(\Args, \Att', \cl)$ would not be well-formed. In order to have $\red{1}(F') = F$ we must have $x_{i+1} \succ x_i$ for all $1 \leq i < n$. But then, by transitivity and since $x_1 = x_n$ we obtain $x_1 \succ x_1$, which is in contradiction to $\succ$ being asymmetric.
		
		On the other hand, suppose there is an attack $(a,b) \in F$ with a path from $a$ to $b$ in $\problematicPart(F)$. Let us denote this path as $x_1,\ldots,x_n$ with $x_1 = a$ and $x_n = b$. By the same argument as above, if there was a PCAF $F' = (\Args, \Att', \cl, \succ)$ such that $\red{1}(F') = F$, then $x_n \succ x_1$, i.e., $b \succ a$. But then $(a,b) \not\in \red{1}(F')$. Contradiction.  \qedhere
	\end{itemize}
\end{proof}

\subsection{Proofs for Section~\ref{sec:imax}}

\cfDoesNotChangeForSomeReductions*
\begin{proof}
	Let $F = (\Args,\Att,\cl, \succ)$ be a PCAF and let $i \in \{2,3,4\}$. Observe that if $(a,b) \in F$ then either $(a,b) \in \red{i}(F)$ or $(b,a) \in \red{i}(F)$. In other words, if there is an attack between two arguments in a PCAF, then there will still be at least one attack between those arguments after Reduction~$i$ has been applied. Conversely, if there is an attack $(a,b) \in \red{i}(F)$, then it must be that either $(a,b) \in F$ or $(b,a) \in F$.
	Thus the argument set $E$ is conflict free in $\red{i}(F)$ iff it is conflict-free in $(\Args,\Att,\cl)$.
\end{proof}

\imaxThirdImage*
\begin{proof}
	We show this for $\pref_p^3$. The other results follow from $\stb_p^3(F) \allowbreak \subseteq \allowbreak \semi_p^3(F) \allowbreak\subseteq \allowbreak\pref_p^3(F)$. Towards a contradiction, assume there is a PCAF $F = (\Args,\Att,\cl, \allowbreak \succ)$ such that $S \subset T$ for some $S,T \in \pref^3_p(F)$. Then there must be $S' \subseteq \Args$ such that $S' \in \pref(\red{3}(F))$ and $\cl(S') = S$, as well as $T' \subseteq \Args$ with $T' \in \pref(\red{3}(F))$ and $\cl(T') = T$. Observe that $S' \not\subseteq T'$, otherwise $S' \not\in \pref(\red{3}(F))$. Thus, there is $x \in S'$ (with $\cl(x) \in S$) such that $x \not\in T'$. However, $\cl(x) \in T$ since $S \subset T$, i.e., there is some $x' \in T'$ such that $\cl(x') = \cl(x)$. There are two possibilities for why $x$ is not in $T'$:
	\begin{enumerate}
		\item $T' \cup \{x\} \not\in \cf(\red{3}(F))$. By Lemma~\ref{lemma:cfDoesNotChangeForReductions234}, $T' \cup \{x\} \not\in \cf((\Args,\Att,\cl))$. Therefore, there is some $y \in T'$ such that $y \not\in S'$ and either $(x,y) \in F$ or $(y,x) \in F$. Actually, it cannot be that $(x,y) \in F$, otherwise, by the well-formedness of $(\Args,\Att,\cl)$, we would have $(x',y) \in F$ which, also by Lemma~\ref{lemma:cfDoesNotChangeForReductions234}, would mean that $T' \not\in \cf(\red{3}(F))$. Thus, $(y,x) \in F$. Since $(x,y) \not\in F$, and by the definition of Reduction~$3$, $(y,x) \in \red{3}(F)$. $S'$ must defend $x$ from $y$ in $\red{3}(F)$, i.e., there is some $z \in S'$ such that $(z,y) \in \red{3}(F)$. Therefore, also $(z,y) \in F$. Since we have that $S \subset T$ there is some $z' \in T'$ such that $\cl(z') = \cl(z)$. $(z',y) \in F$ by the well-formedness of $(\Args,\Att,\cl)$. But then, by Lemma~\ref{lemma:cfDoesNotChangeForReductions234}, $T' \not\in \cf(\red{3}(F))$. Contradiction.
		\item $x$ is not defended by $T'$. Then there is some $y \in \Args$ such that $(y,x) \in \red{3}(F)$ and such that $y$ is not attacked by any argument in $T'$. But $S'$ must defend $x$ against $y$ in $\red{3}(F)$, i.e., there is $z \in S'$ such that $(z,y) \in \red{3}(F)$. Then also $(z,y) \in F$. Since $S \subset T$ there is some $z' \in T'$ such that $\cl(z') = \cl(z)$. $(z',y) \in F$ by the well-formedness of $(\Args,\Att,\cl)$. It cannot be that $(z',y) \in \red{3}(F)$, i.e., $y \succ z'$. But then, by the definition of Reduction~$3$, we must have $(y,z') \in F$ and also $(y,z') \in \red{3}(F)$, which means that $T'$ is attacked by $y$ but not defended against it, i.e., $T' \not\in \adm(\red{3}(F))$. Contradiction. \qedhere 
	\end{enumerate}  
\end{proof}

\imaxThirdImageStage*
\begin{proof}
	Let $F$ be the CAF shown in Figure~\ref{fig:ImaxImage3}. Clearly, $F \in \ImageTrans{3}$. 
	We can see that $\naive(F) = \{\{a,a'\},\{a',c\},\{b\}\}$. Furthermore, $\{a,a'\}^\oplus_F = \{a,a',b\}$, $\{a',c\}^\oplus_F = \{a,a',c\}$, and $\{b\}^\oplus_F = \{a',b,c\}$. Since the ranges of all three naive extensions are incomparable, we have that $\stage(F) = \naive(F)$ and thus $\stage_c(F) = \{\{\alpha\},\{\alpha,\gamma\},\{\beta\}\}$.
\end{proof}

\imaxOtherImages*
\begin{proof}
	We show this for $\stb_p^i$. The other results follow by $\stb_p^i(F) \subseteq \semi_p^i(F) \subseteq \pref_p^i(F)$ and $\stb_p^i(F) \subseteq \stage_p^i(F)$. 
	
	For $i \in \{1,4\}$, let $F$ be the CAF shown in Figure~\ref{fig:ImaxImage14}. $F \in \ImageTrans{1}$ by Proposition~\ref{prop:characterizationImage1-transitive}. $F \in \ImageTrans{4}$ since $\red{4}(F') = F$ for $F' = (\Args,\Att,\cl,\succ)$ with $\Args = \{a,a',b\}$, $\Att = \{(b,a)\}$, $\cl(a) = \cl(a') = \alpha$, $\cl(b) = \beta$, and $a \succ b$. As required, the underlying CAF of $F'$ is well-formed. It can be verified that $\stb(F) = \{\{a,a'\},\{a',b\}\}$ and thus $\stb_c(F) = \{\{\alpha\},\{\alpha,\beta\}\}$.
	
	For $i = 2$, let $G$ be the CAF of Figure~\ref{fig:ImaxImage2}. $G \in \ImageTrans{2}$ since $\red{2}(G') = G$ for the PCAF $G' = (\Args',\Att',\cl',\succ)$ with $\Att' = \{(b,a),\allowbreak(b,a'),\allowbreak(b',a),\allowbreak(b',a')\}$, $a \succ b$, and $a' \succ b'$. As required, the underlying CAF of $G'$ is well-formed. It can be verified that $\stb(G) = \{\{a,a',a''\},\{a'',b,b'\}\}$ and thus $\stb_c(G) = \{\{\alpha\},\{\alpha,\beta\}\}$. 
\end{proof}

\begin{figure}[ht]
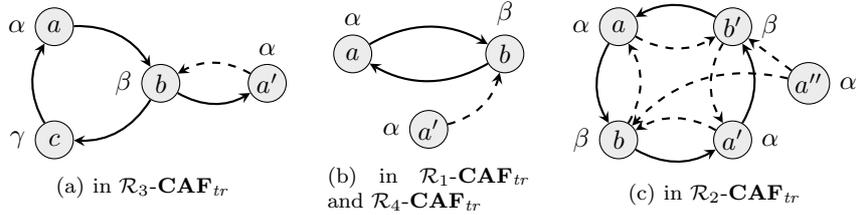

	\centering
	\begin{subfigure}{0.38\textwidth}
		\centering
		\tikz{
			\node[arg, label={left}:$\alpha$] (a) at (0,1.5) {$a$};
			\node[arg, label={above}:$\alpha$] (aPrime) at (2.8,0.75) {$a'$};
			\node[arg, label={left}:$\beta$] (b) at (1.4,0.75) {$b$};
			\node[arg, label={left}:$\gamma$] (c) at (0,0) {$c$};
			\draw[attack] 
			(a) edge [bend left] (b)
			(b) edge [bend left] (c)
			(c) edge [bend left] (a)
			(aPrime) edge [bend right, dashed] (b)
			(b) edge [bend right] (aPrime);
		}
		\caption{in $\ImageTrans{3}$}
		\label{fig:ImaxImage3}
	\end{subfigure}
	\hfill
	\begin{subfigure}{0.22\textwidth}
		\centering
		\tikz{
			\node[arg, label={above}:$\alpha$] (a) at (0,1) {$a$};
			\node[arg, label={above}:$\beta$] (b) at (2,1) {$b$};
			\node[arg, label={left}:$\alpha$] (c) at (1,0) {$a'$};
			\draw[attack] 
			(a) edge [bend left] (b)
			(b) edge [bend left] (a)
			(c) edge [dashed, bend right] (b);
		}
		\caption{in $\ImageTrans{1}$ and $\ImageTrans{4}$}
		\label{fig:ImaxImage14}
	\end{subfigure}
	\hfill
	\begin{subfigure}{0.38\textwidth}
		\centering
		\tikz{
			\node[arg, label={left}:$\alpha$] (a) at (0,1.5) {$a$};
			\node[arg, label={right}:$\alpha$] (b) at (1.5,0) {$a'$};
			\node[arg, label={left}:$\beta$] (c) at (0,0) {$b$};
			\node[arg, label={right}:$\beta$] (d) at (1.5,1.5) {$b'$};
			\node[arg, label={right}:$\alpha$] (e) at (2.5,0.75) {$a''$};
			\draw[attack] 
			(a) edge [bend right] (c)
			(c) edge [bend right, dashed] (a)
			(b) edge [bend right] (d)
			(d) edge [bend right, dashed] (b)
			(c) edge [bend right] (b)
			(b) edge [bend right, dashed] (c)
			(d) edge [bend right] (a)
			(a) edge [bend right, dashed] (d)
			(e) edge [bend right, dashed] (c)
			(e) edge [dashed] (d);
		}
		\caption{in $\ImageTrans{2}$}
		\label{fig:ImaxImage2}
	\end{subfigure}
	\caption{CAFs used as counter examples for I-maximality of some semantics. Dashed arrows are edges in the respective wf-problematic part.}
	\label{fig:ImaxCafs}
\end{figure}

\subsection{Proofs for Section~\ref{sec:complexity}}

\verificationConflictfreeNaiveFirstImage*
\begin{proof}
	$\NP$-membership follows from results for general CAFs (cf.\ Table~\ref{tab:complexity_CAFwfCAF}). We now show \NP-hardness. Let $\varphi$ be an arbitrary instance of 3-\SAT\ given as a set $\{C_1,\ldots,C_m\}$ of clauses over variables $X$. Without loss of generality, we can assume that every variable appears both positively and negatively in $\varphi$. We construct a PCAF $F = (\Args,\Att,\cl, \succ)$ as well as a set of claims $S$: 
	\begin{itemize}
		\item $\Args = V \cup \overline{V} \cup H$ where $V = \{x_i \mid x \in C_i, 1 \leq i \leq m\}$, $\overline{V} = \{\overline{x}_i \mid \neg x \in C_i, 1 \leq i \leq m\}$, and \\
		$H = \{x_T, x_F \mid x \in X\}$;
		\item $\Att = \{(x_T,x_i),(x_F,x_i)  \mid x_i \in V \}\ \cup$\\ 
		\phantom{$\Att =\ $}%
		$\{(x_T,\overline{x}_i),(x_F,\overline{x}_i) \mid \overline{x}_i \in \overline{V} \}$;
		
		\item $\cl(x_i) = \cl(\overline{x}_i) = i$ for all $x_i,\overline{x}_i \in V \cup \overline{V}$, 
		$\cl(x_T) = \cl(x_F) = x$ for all $x \in X$;
		\item $x_i \succ x_T$ for all $x_i \in V$ and 
		$\overline{x}_i \succ x_F$ for all $\overline{x}_i \in \overline{V}$;
		\item $S = \{1,\ldots,m\} \cup X$. 
	\end{itemize}
	Figure~\ref{fig:reductionVerficationImage1CfAppendix} illustrates the above construction (and is in fact identical to Figure~\ref{fig:reductionVerficationImage1Cf} from the main text). Let $F' = \red{1}(F) = (\Args,\Att',\cl)$. 
	
	Assume $\varphi$ is satisfiable. Then there is an interpretation $I$ such that $I \models \varphi$. Let $E = \{x_T \in H \mid x \in I\} \cup \{x_F \in H \mid x \not\in I\} \cup \{x_i \in V \mid x \in I\} \cup \{\overline{x}_i \in \overline{V} \mid x \not\in I\}$. It can be easily verified that $E$ is conflict free in $(\Args,\Att')$ and that $\cl(E) = S$. Towards a contradiction, assume that $E$ is not a naive extension, i.e., that there is a conflict-free extension $D$ for $(\Args,\Att')$ such that $D \supset E$. Consider any $x$ such that $x_T \in E$ (the case that $x_F \in E$ is analogous). Then also $x_T \in D$. Since $D$ must be conflict-free, $\overline{x}_j \not\in D$ for all corresponding $\overline{x}_j$. Furthermore, for all corresponding $x_i$, we have $x_i \in E$ and thus $x_i \in D$. But then $x_F \not\in D$. Since for all $x \in X$ either $x_T \in E$ or $x_F \in E$, this means that $D \subseteq E$. Contradiction, i.e., $E$ is naive in $(\Args,\Att')$.
	
	Assume $S$ is conflict-free in $(\Args,\Att',\cl)$. Then there is some $E \subseteq \Args$ such that $E$ is conflict-free in $(\Args,\Att')$ and $\cl(E) = S$. Let $x$ be any variable in $X$. Since $x \in \cl(E)$ it must be that either $x_T \in E$ or $x_F \in E$. Thus, for all $i,j$, we have $x_i \in E \implies \overline{x}_j \not\in E$ and $\overline{x}_i \in E \implies x_j \not\in E$ (otherwise, we would need both $x_T \not\in E$ and $x_F \not\in E$ for $E$ to be conflict-free). Furthermore, for any $i \in \{1,\ldots,m\}$, there must be some $x$ such that $x_i \in E$ or $\overline{x}_i \in E$.  Let $I = \{x \mid x_i \in E \text{ for some } i \}$. 
	Then for every $i$ there is some $x$ such that either $x \in C_i$ and $x \in I$ or $\neg x \in C_i$ and $x \not\in I$. Thus, $I$ satisfies all clauses in $C$ which means that $\varphi$ is satisfiable. The proof works likewise if we assume $S$ to be naive since $\naive_c(G) \subseteq \cf_c(G)$ for any CAF~$G$.
\end{proof}
\begin{figure}[ht]
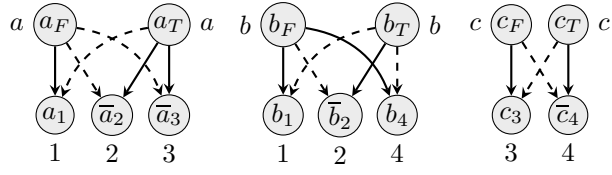

	\centering
	\tikz{
		\node[arg, label={left}:$a$] (aF) at (0,1.2) {$a_F$};
		\node[arg, label={right}:$a$] (aT) at (1.5,1.2) {$a_T$};
		\node[arg, label={left}:$b$] (bF) at (3,1.2) {$b_F$};
		\node[arg, label={right}:$b$] (bT) at (4.5,1.2) {$b_T$};
		\node[arg, label={left}:$c$] (cF) at (6,1.2) {$c_F$};
		\node[arg, label={right}:$c$] (cT) at (6.75,1.2) {$c_T$};
		\node[arg, label={below}:$1$] (a1) at (0,0) {$a_1$};
		\node[arg, label={below}:$2$] (a2) at (0.75,0) {$\overline{a}_2$};
		\node[arg, label={below}:$3$] (a3) at (1.5,0) {$\overline{a}_3$};
		\node[arg, label={below}:$1$] (b1) at (3,0) {$b_1$};
		\node[arg, label={below}:$2$] (b2) at (3.75,0) {$\overline{b}_2$};
		\node[arg, label={below}:$4$] (b4) at (4.5,0) {$b_4$};
		\node[arg, label={below}:$3$] (c3) at (6,0) {$c_3$};
		\node[arg, label={below}:$4$] (c4) at (6.75,0) {$\overline{c}_4$};
		\draw[attack] 
		(aF) edge (a1)
		(aF) edge [dashed] (a2)
		(aF) edge [bend left, dashed] (a3)
		(aT) edge [bend right, dashed] (a1)
		(aT) edge (a2) 
		(aT) edge (a3)
		(bF) edge (b1) 
		(bF) edge [dashed] (b2) 
		(bF) edge [bend left] (b4) 
		(bT) edge [bend right, dashed] (b1)
		(bT) edge (b2)
		(bT) edge [dashed] (b4)
		(cF) edge (c3)
		(cF) edge [dashed] (c4)
		(cT) edge (c4)
		(cT) edge [dashed] (c3);
	}
	\caption{Reduction of 3-\SAT-instance $C_1 = \{a,b,c\}$, $C_2 = \{\neg a, \neg b\}$, $C_3 = \{\neg a, c\}$, $C_4 = \{b, \neg c\}$, to an instance $(F,S)$ of $\Ver_{\cf,1}^\PCAF$. Dashed arrows are attacks deleted in $\red{1}(F)$, i.e., they are edges in $\problematicPart(\red{1}(F))$.
	}
	\label{fig:reductionVerficationImage1CfAppendix}
\end{figure}

\verificationStableAdmissibleCompleteFirstImage*
\begin{proof}
	$\NP$-membership follows from results for general CAFs (cf.\ Table~\ref{tab:complexity_CAFwfCAF}). We show \NP-hardness. Let $\varphi$ be an arbitrary 3-\SAT-instance given as a set $\{C_1,\ldots,C_m\}$ of clauses over variables $X$. For convenience, we directly construct a CAF $F = (\Args,\Att,\cl) $ with $F \in \ImageTrans{1}$ instead of providing a PCAF $F'$ such that $\red{1}(F') = F$. This is legitimate, as, by our characterization of $\ImageTrans{1}$ (see Proposition~\ref{prop:characterizationImage1-transitive}), we can obtain $F'$ by simply adding all edges in $\problematicPart(F)$ to $R$ and defining $\succ$ accordingly. We also construct a set of claims $S$. 
	\begin{itemize}
		\item $\Args = V \cup \overline{V} \cup H$ where $V = \{x_i \mid x \in C_i, 1 \leq i \leq m\}$, $\overline{V} = \{\overline{x}_i \mid \neg x \in C_i, 1 \leq i \leq m\}$, and\\ 
		$H = \{x^k_{i,j}, \hat{x}^k_{i,j} \mid 1 \leq k \leq 4, x_i \in V, \overline{x}_j \in \overline{V}\}$;
		
		\item $\Att = \{(x_i, x^1_{i,j}), \allowbreak (x^1_{i,j}, x^2_{i,j}), \allowbreak (x^2_{i,j}, \overline{x}_j), \allowbreak (\overline{x}_j, x^3_{i,j}),$ \\ 
		\phantom{$\Att = \{$}%
		$(x^3_{i,j}, x^4_{i,j}), \allowbreak (x^4_{i,j}, x_i) \mid x_i \in V, \overline{x}_j \in \overline{V}\}$;
		
		\item $\cl(x_i) = \cl(\overline{x}_i) = i$ for all $x_i,\overline{x}_i$, \\
		$\cl(x^k_{i,j}) = \cl(\hat{x}^k_{i,j}) = x^k_{i,j}$ for all $x^k_{i,j}, \hat{x}^k_{i,j}$;
		
		\item $S = \{1,\ldots,m\} \cup \{\cl(a) \mid a \in H \}$. 
	\end{itemize}
	Figure~\ref{fig:reductionVerficationImage1StbAppendix} illustrates the above construction (and is the same as Figure~\ref{fig:reductionVerficationImage1Stb} in the main text). In general, every $\hat{x}^k_{i,j}$ only has outgoing edges in the wf-problematic part, and no incoming or outgoing attacks in $\Att$. Every $x^k_{i,j}$ only has incoming edges in the wf-problematic part. Finally, there can be no edges in the wf-problematic part between any $x_i$ (or $\overline{x}_i$) and any other $x_j$ (or $\overline{x}_j$). From this, and by the above construction, we can infer that $(\Args,\Att,\cl)$ fulfills all of the conditions to be in $\ImageTrans{1}$ (cf.\ Proposition~\ref{prop:characterizationImage1-transitive}). It remains to show the correctness of the above construction.   
	
	Assume $\varphi$ is satisfiable. Then there is an interpretation $I$ such that $I \models \varphi$. Let $E = \{x_i \in V \mid x \in I\} \cup \{\overline{x}_i \in \overline{V} \mid x \not\in I\} \cup \{x^2_{i,j},x^3_{i,j} \mid x_i,\overline{x}_j \in \Args, x \in I\} \cup \{x^1_{i,j},x^4_{i,j} \mid x_i,\overline{x}_j \in \Args, x \not\in I\} \cup \{\hat{x}^k_{i,j} \mid \hat{x}^k_{i,j} \in \Args\}$. It can be verified that $E$ is stable in $(\Args,\Att)$ and that $\cl(E) = S$. Since $\stable_c(F) \subseteq \comp_c(F) \subseteq \adm_c(F)$ for any CAF $F$, $S$ is also admissible and complete in $(\Args,\Att,\cl)$.
	
	Assume $S$ is admissible in $(\Args,\Att,\cl)$. Then there is some $E \subseteq \Args$ such that $E$ is admissible in $(\Args,\Att)$ and $\cl(E) = S$. Thus, for any $i \in \{1,\ldots,m\}$, there must be some $x$ such that $x_i \in E$ or $\overline{x}_i \in E$. Consider the case that $x_i \in E$. Since $E$ is admissible, $x^1_{i,j} \not\in E$ for any $j$ such that $\overline{x}_j \in \Args$. This further means that $\overline{x}_j \not\in E$ for any $\overline{x}_j \in \Args$, since we would need $x^1_{i,j} \in E$ to defend $\overline{x}_j$ from the attack by $x^2_{i,j}$. Likewise, if $\overline{x}_i \in E$, then $x_j \not\in E$ for all $x_j \in \Args$. Let $I = \{x \mid x_i \in E \text{ for some } i \}$. 
	Then for every $i$ there is some $x$ such that either $x \in C_i$ and $x \in I$ or $\neg x \in C_i$ and $x \not\in I$. Thus, $I$ satisfies all clauses in $C$ which means that $\varphi$ is satisfiable. The proof works likewise if we assume $S$ to be stable or complete since $\stable_c(F) \subseteq \comp_c(F) \subseteq \adm_c(F)$ for any CAF $F$.
\end{proof}
\begin{figure}[ht]
	\centering
	\footnotesize
	\tikz[yscale=0.9, label distance=-0.5mm]{
		\node[arg, label={above}:$1$] (a1) at (1,4.25) {$a_1$};
		\node[arg, label={right}:$2$] (a2) at (5,4.25) {$\overline{a}_2$};
		\node[arg, label={above}:$a^1_{1,2}$] (h1) at (2,5) {$a^1_{1,2}$};
		\node[arg, label={above}:$a^2_{1,2}$] (h2) at (4,5) {$a^2_{1,2}$};
		\node[arg, label={above}:$a^3_{1,2}$] (h3) at (4,3.5) {$a^3_{1,2}$};
		\node[arg, label={above}:$a^4_{1,2}$] (h4) at (2,3.5) {$a^4_{1,2}$};
		\node[arg, label={right}:$2$] (b2) at (5,1.75) {$b_2$};
		\node[arg, label={below}:$3$] (b3) at (1,1.75) {$\overline{b}_3$};
		\node[arg, label={below}:$b^1_{2,3}$] (g1) at (4,2.5) {$b^1_{2,3}$};
		\node[arg, label={below}:$b^2_{2,3}$] (g2) at (2,2.5) {$b^2_{2,3}$};
		\node[arg, label={below}:$b^3_{2,3}$] (g3) at (2,1) {$b^3_{2,3}$};
		\node[arg, label={below}:$b^4_{2,3}$] (g4) at (4,1) {$b^4_{2,3}$};
		\node[arg, label={right}:$3$] (c3) at (3,0) {$c_3$};
		\node[arg, label={right}:$a^1_{1,2}$] (d1) at (6,5.5) {$\hat{a}^1_{1,2}$};
		\node[arg, label={right}:$a^2_{1,2}$] (d2) at (6,4.5) {$\hat{a}^2_{1,2}$};
		\node[arg, label={left}:$a^3_{1,2}$] (d3) at (0,3.5) {$\hat{a}^3_{1,2}$};
		\node[arg, label={left}:$a^4_{1,2}$] (d4) at (0,4.5) {$\hat{a}^4_{1,2}$};
		\node[arg, label={left}:$b^1_{2,3}$] (e1) at (0,2.5) {$\hat{b}^1_{2,3}$};
		\node[arg, label={left}:$b^2_{2,3}$] (e2) at (0,1.5) {$\hat{b}^2_{2,3}$};
		\node[arg, label={right}:$b^3_{2,3}$] (e3) at (6,0.5) {$\hat{b}^3_{2,3}$};
		\node[arg, label={right}:$b^4_{2,3}$] (e4) at (6,1.5) {$\hat{b}^4_{2,3}$};
		\draw[attack] 
		(a1) edge (h1)
		(h1) edge (h2)
		(h2) edge (a2)
		(a2) edge (h3)
		(h3) edge (h4)
		(h4) edge (a1)
		(b2) edge (g1)
		(g1) edge (g2)
		(g2) edge (b3)
		(b3) edge (g3)
		(g3) edge (g4)
		(g4) edge (b2)
		(d1) edge [dashed] (h2)
		(d2) edge [dashed, bend right] (a2)
		(d3) edge [dashed] (h4)
		(d4) edge [dashed] (a1)
		(e1) edge [dashed] (g2)
		(e2) edge [dashed] (b3)
		(e3) edge [dashed] (g4)
		(e4) edge [dashed, bend left] (b2)
		(a2) edge [dashed, bend left] (g1)
		(b2) edge [dashed, bend right] (h3)
		(c3) edge [dashed, bend right] (g3);
	}
	\caption{Reduction of 3-\SAT-instance $C_1 = \{a\}$, $C_2 = \{\neg a, b\}$, $C_3 = \{\neg b, c\}$, to an instance $(F',S)$ of $\Ver_{\stable,1}^\PCAF$. Dashed arrows are attacks deleted in $\red{1}(F')$, i.e., they are edges in $\problematicPart(\red{1}(F'))$.}
	\label{fig:reductionVerficationImage1StbAppendix}
\end{figure}

\verificationPrefSemiStageFirstImage*
\begin{proof}
	$\SigmaP{2}$-membership follows from results for general CAFs (cf.\ Table~\ref{tab:complexity_CAFwfCAF}). We show hardness separately below (Lemmas~\ref{prop:verificationPrefFirstImage}, \ref{prop:verificationSemiFirstImage}, and~\ref{prop:verificationStageFirstImage}).
\end{proof}

\begin{restatable}{lemma}{verificationPrefFirstImage} \label{prop:verificationPrefFirstImage}
	$\Ver_{\pref,1}^\PCAF$ is $\SigmaP{2}$-hard, even for transitive preferences.
\end{restatable}
\begin{proof}
	We provide a reduction from $\text{QBF}^2_\forall$ to the complementary problem of $\Ver_{\pref,1}^\PCAF$. This is an adaptation of \cite[Reduction~3.7]{DvorakD17}. Let $\Phi = \forall Y \exists Z \varphi(Y,Z)$ be an arbitrary instance of $\text{QBF}^2_\forall$. Let $X = Y \cup Z$. We assume $\varphi$ to be in CNF and denote the set of clauses in $\varphi$ by $C$. Without loss of generality, we can assume that every clause contains either $z$ or $\neg z$ for at least one $z \in Z$. We construct a CAF $(\Args,\Att,\cl)$ as well as a set of claims $S$: 
	\begin{itemize}
		\item $\Args = \{\varphi, \overline{\varphi}\} \cup C \cup X \cup \overline{X} \cup Y^*\cup \overline Y^*$, where \\
		$\overline{X} = \{\overline x \mid x \in X\}$,
		$Y^*=\{y^*\mid y\in Y\}$, and\\ $\overline Y^*=\{\overline y^*\mid y\in Y\}$;
		\item $\Att = \{(x,\overline x),(\overline x,x)\mid x \in X\}\ \cup$\\
		\phantom{$\Att =\ $}%
		$\{(x,c) \mid x \in c, c \in C\} \cup \{(\overline{x},c) \mid \neg x \in c, c \in C\}\ \cup $\\
		\phantom{$\Att =\ $}%
		$\{(c,c)\mid c\in C\} \cup \{(c,\varphi) \mid c \in C\}\ \cup$\\ 
		\phantom{$\Att =\ $}%
		$\{(\varphi,\overline{\varphi}),(\overline{\varphi},\overline{\varphi})\} \cup \{(\overline{\varphi},z),(\overline{\varphi},\overline{z}) \mid z \in Z \}$;
		\item $\cl(v^*)=v$ for $v\in Y^*\cup \overline Y^*$ and $\cl(v) = v$ else;
		\item $S = Y\cup \overline Y$. 
	\end{itemize}
	Figure~\ref{fig:reductionVerficationImage1Prf} illustrates the above construction. First of all, we can see that $(\Args,\Att,\cl)$ is in $\ImageTrans{1}$ since all paths in $\problematicPart(F)=\{(v^*,v)\mid v\in Y\cup \overline Y\}$ are of length $1$ (only arguments in $Y^*\cup \overline Y^*$ have outgoing edges in $\problematicPart(F)$). It remains to verify the correctness of the reduction, i.e., we will show that $S\notin \pref_c(F)$ iff $\Phi$ is valid.
	
	First assume $\Phi$ is valid, that is, for all $Y'\subseteq Y$, there is $Z'\subseteq Z$ such that $M=Y'\cup Z'$ is a model of $\varphi$. Towards a contradiction, assume $S\in \pref_c(F)$. Then there must be some $E \subseteq \Args$ such that $E \in \pref((\Args,\Att))$ and $\cl(E)=S$. Clearly, $Y^*\cup \overline Y^*\subseteq E$; moreover, $E$ contains either $y$ or $\overline y$ for each $y\in Y$ (otherwise $E\cup \{y\}$ is a proper admissible superset of $E$, contradiction to $E$ being preferred in $(\Args,\Att)$). Let $Y'=E\cap Y$. By assumption, there is $Z'\subseteq Z$ such that $M=Y'\cup Z'$ is a model of $\varphi$. We show that $D=Y^*\cup \overline Y^*\cup M\cup \{\overline x\mid x\notin M\}\cup \{\varphi\}$ is admissible: 
	$D$ is conflict-free by construction. $D$ is admissible since $\varphi$ defends $v\in Z'\cup \{\overline z\mid z\notin Z'\}$ against $\overline \varphi$; moreover, each argument $v\in X$ defends itself against $\overline v$ and vice verca; also, $M\cup \{\overline x\mid x\notin M\}$ defends $\varphi$ against each attack from clause-arguments $c\in C$ since $\varphi$ is satisfiable: For each clause $c\in C$, there is either $v\in M$ with $v\in C$ or $\neg v\in C $ for some $v\notin M$. In the first case, $(v,c)\in \Att$ and $v\in E$, in the latter, $(\overline v,c)\in \Att$ and $\overline v\in E$. 
	Now observe that $E$ and $D$ both contain $Y'\cup \{\overline y\mid y\notin Y'\}\cup Y^*\cup \overline Y^*$ (by construction of $D$), that is, $E \subset D$ and therefore $E \not\in \pref((\Args,\Att))$. Contradiction.
	
	Now assume $S\notin \pref_c(F)$. Consider an arbitrary subset $Y'\subseteq Y$. We show that there is some $Z'\subseteq Z$ such that $Y'\cup Z'$ is a model of $\varphi$. Let $E=Y'\cup \{\overline y\mid y\notin Y'\}\cup Y^*\cup \overline Y^*$ (and observe that $\cl(E)=S$). By assumption $S\notin \pref_c(F)$ we have that $E$ is not preferred in $(\Args,\Att)$, that is, there  is some $D\in \adm((\Args,\Att))$ with $D \supset E$. 
	Thus, there is some $v\in \Args\setminus E$ such that $v\in D$. It follows that $\varphi\in D$: Clearly, $v\in \{\varphi\}\cup Z\cup \overline Z$ since each remaining argument is either self-attacking or attacked by $E$ (and thus also by $D$). In case $v=\varphi$, we are done; in case $v\in Z\cup \overline Z$, we have $\varphi\in D$ by admissibility of $D$ (observe that $\varphi$ is the only attacker of $\overline \varphi$). Consequently, $D$ defends $\varphi$ against each attack from each clause-argument $c\in C$.
	Now, let $Z'=Z\cap D$. We show that $M=Y'\cup Z'$ is a model of $\varphi$. Consider some arbitrary clause $c\in C$. Since $\varphi\in D$, there is some $v\in D$ such that $(v,c)\in \Att$ by admissibility of $D$. In case $v\in X$, we have $v\in M$ and $v\in c$ by construction of $F$; similarly, in case $v\in \overline X$ we have $v\notin M$ and $\neg v\in c$. Thus $c$ is satisfied by $M$. Since $c$ was chosen arbitrary it follows that $\varphi$ is satisfiable.
	We have shown that $\Phi$ is valid. 
\end{proof}
\begin{figure}[ht]
	\centering
	\tikz{
		\node[arg, label={left}:$a$] (a1) at (1.5,0) {$a$};
		\node[arg, label={right}:$\overline{a}$] (a2) at (2.5,0) {$\overline{a}$};
		\node[arg, label={left}:$b$] (b1) at (4,0) {$b$};
		\node[arg, label={right}:$\overline{b}$] (b2) at (5,0) {$\overline{b}$};
		\node[arg, label={left}:$a$] (a1p) at (0,1) {$a^*$};
		\node[arg, label={left}:$\overline{a}$] (a2p) at (0,0) {$\overline a^*$};
		\node[arg, label={above left}:$c_1$] (c1) at (2,1.5) {$c_1$};
		\node[arg, label={above right}:$c_2$] (c2) at (4.5,1.5) {$c_2$};
		\node[arg, label={left}:$\varphi$] (phi) at (3.25,3) {$\varphi$};
		\node[arg, label={right}:$\overline{\varphi}$] (negphi) at (6,3) {$\overline{\varphi}$};
		\draw[attack] 
		(a1) edge (a2)
		(a2) edge (a1)
		(b1) edge (b2)
		(b2) edge (b1)
		(a1) edge (c1)
		(b1) edge (c1)
		(b1) edge (c2)
		(a2) edge (c2)
		(c1) edge[loop left] (c1)
		(c2) edge[loop left] (c2)
		(c1) edge (phi)
		(c2) edge (phi)
		(phi) edge (negphi)
		(negphi) edge [loop above] (negphi)
		(negphi) edge [bend angle=15, bend left] (b1)
		(negphi) edge [bend angle=15, bend left] (b2)
		;
	}
	\caption{Reduction of the $\text{QBF}^2_\forall$ instance $\forall a \exists b ((a \lor b) \allowbreak \land \allowbreak (\neg a \lor b))$ to an instance of $\Ver_{\pref,1}^\PCAF$.}
	\label{fig:reductionVerficationImage1Prf}
\end{figure}

\begin{restatable}{lemma}{verificationSemiFirstImage} \label{prop:verificationSemiFirstImage}
	$\Ver_{\semi,1}^\PCAF$ is $\SigmaP{2}$-hard, even for transitive preferences.
\end{restatable}
\begin{proof}
	We provide a reduction from $\text{QBF}^2_\forall$ to the complementary problem: 
	For an instance $\Phi = \forall Y \exists Z \varphi(Y,Z)$ of $\text{QBF}^2_\forall$ where $\varphi$ is given by a set of clauses $C$ over atoms in $X = Y \cup Z$, we construct a CAF $F = (\Args,\Att,\cl)$ as well as a set of claims $S$: 
	\begin{itemize}
		\item $\Args = \{\varphi, \overline{\varphi}\} \cup C \cup X \cup \overline{X} \cup Y^*\cup \overline Y^*\cup \{d_v\mid v\in Y\cup \overline Y\}$, where $\overline{X} = \{\overline x \mid x \in X\}$, $Y^*=\{y^*\mid y\in Y\}$, and $\overline Y^*=\{\overline y^*\mid y\in Y\}$;
		
		\item $\Att = \{(x,\overline x),(\overline x,x)\mid x \in X\} \cup \{(c,c), (c,\varphi)\mid c\in C\}\ \cup$\\ 
		\phantom{$\Att =\ $}%
		$\{(x,c) \mid x \in c, c \in C\} \cup \{(\overline{x},c) \mid \neg x \in c, c \in C\}\ \cup$\\ 
		\phantom{$\Att =\ $}%
		$\{(d_v,d_v),(v,d_v)\mid v\in Y\cup \overline Y\}\ \cup$\\
		\phantom{$\Att =\ $}%
		$\{(\varphi,\overline{\varphi}),(\overline{\varphi},\overline{\varphi})\} \cup \{(\overline{\varphi},z),(\overline{\varphi},\overline{z}) \mid z \in Z \}$;
		\item $\cl(v^*)=v$ for $v\in Y^*\cup \overline Y^*$ and $\cl(v) = v$ else;
		\item $S = Y\cup \overline Y$. 
	\end{itemize}
	
	Figure~\ref{fig:reductionVerficationImage1Semi} illustrates the above construction. The constructed CAF is contained in $\ImageTrans{1}$ since all paths in $\problematicPart(F)=\{(v^*,v)\mid v\in Y\cup \overline Y\}$ are of length $1$ (only arguments in $Y^*\cup \overline Y^*$ have outgoing edges in $\problematicPart(F)$).
	It remains to verify the correctness of the reduction, i.e., we will show that $S\notin \semi\!_c(F)$ iff $\Phi$ is valid.
	
	First assume $\Phi$ is valid, that is, for all $Y'\subseteq Y$, there is $Z'\subseteq Z$ such that $M=Y'\cup Z'$ is a model of $\varphi$. Towards a contradiction, assume $S\in \semi\!_c(F)$ and consider a $\semi$-realization $E$ of $S$, i.e., $\cl(E)=S$ and $E\in \semi((\Args,\Att))$. Clearly, $Y^*\cup \overline Y^*\subseteq E$; moreover, $E$ contains either $y$ or $\overline y$ for each $y\in Y$ (otherwise $E\cup \{y\}$ is a proper admissible superset of $E$, contradiction to $E$ being semi-stable in $(\Args,\Att)$). Let $Y'=E\cap Y$. By assumption, there is $Z'\subseteq Z$ such that $M=Y'\cup Z'$ is a model of $\varphi$. We show that $D=Y^*\cup \overline Y^*\cup M\cup \{\overline x\mid x\notin M\}\cup \{\varphi\}$ is admissible: %
	$D$ is conflict-free by construction. $D$ is admissible since $\varphi$ defends $v\in Z'\cup \{\overline z\mid z\notin Z'\}$ against $\overline \varphi$; moreover, each argument $v\in X$ defends itself against $\overline v$ and vice verca; also, $M\cup \{\overline x\mid x\notin M\}$ defends $\varphi$ against each attack from clause-arguments $c\in C$ since $\varphi$ is satisfiable: For each clause $c\in C$, there is either $v\in M$ with $v\in C$ or $\neg v\in C $ for some $v\notin M$. In the first case, $(v,c)\in \Att$ and $v\in E$, in the latter, $(\overline v,c)\in \Att$ and $\overline v\in E$. 
	Now observe that $E$ and $D$ both contain $Y'\cup \{\overline y\mid y\notin Y'\}\cup Y^*\cup \overline Y^*$ (by construction of $D$), that is, $D$ properly extends $E$. Since $\varphi$ is not contained in $E\cup E^+$ ($\varphi\notin E$ since $\cl(E)=S$ and $\varphi\notin S$ and $\varphi\notin E^+$ since the only attacker of $\varphi$ are clause-arguments $c\in C$), it follows that $E\cup E^+\subset D\cup D^+$, that is, we have found an admissible set $D$ having range larger than $E$.
	Consequently, $S\notin \semi\!_c(F)$. 
	
	Now assume $S\notin \semi\!_c(F)$. Consider an arbitrary subset $Y'\subseteq Y$. We show that there is some $Z'\subseteq Z$ such that $Y'\cup Z'$ is a model of $\varphi$. Let $E=Y'\cup \{\overline y\mid y\notin Y'\}\cup Y^*\cup \overline Y^*$ (and observe that $\cl(E)=S$). By assumption $S\notin \semi\!_c(F)$ we have that $E$ is not semi-stable in $(\Args,\Att)$, that is, there  is some $D\in \adm((\Args,\Att))$ with $D\cup D^+\supset E\cup E^+$. 
	In particular, we have $D\supset E$ since each argument $d_v\in E^+$, $v\in Y\cup \overline Y$, has precisely one non-self-attacking attacker (namely the argument $v$). Thus, $Y'\cup \{\overline y\mid y\notin Y'\}\subseteq D$; moreover, $D$ contains each argument $v\in Y^*\cup \overline Y^*$ since each such argument is unattacked. 
	It follows that $\varphi\in D$: Since $E\subset D$, there is some $v\in \Args\setminus E$ such that $v\in D$. Clearly, $v\in \{\varphi\}\cup Z\cup \overline Z$ since each remaining argument is either self-attacking or attacked by $E$ (and thus also by $D$). In case $v=\varphi$, we are done; in case $v\in Z\cup \overline Z$, we have $\varphi\in D$ by admissibility of $D$ (observe that $\varphi$ is the only attacker of $\overline \varphi$). Consequently, $D$ defends $\varphi$ against each attack from each clause-argument $c\in C$.
	Now, let $Z'=Z\cap D$. We show that $M=Y'\cup Z'$ is a model of $\varphi$. Consider some arbitrary clause $c\in C$. Since $\varphi\in D$, there is some $v\in D$ such that $(v,c)\in \Att$ by admissibility of $D$. In case $v\in X$, we have $v\in M$ and $v\in c$ by construction of $F$; similarly, in case $v\in \overline X$ we have $v\notin M$ and $\neg v\in c$. Thus, $c$ is satisfied by $M$. Since $c$ was chosen arbitrarily it follows that $\varphi$ is satisfiable.
	We have shown that $\Phi$ is valid. 
\end{proof}
\begin{figure}[ht]
	\centering
	\tikz{
		\node[arg, label={left}:$a$] (a1) at (1.5,1) {$a$};
		\node[arg, label={right}:$\overline a$] (a2) at (2.5,1) {$\overline{a}$};
		\node[arg, label={left}:$d_{a}$] (d1) at (1.5,0) {$d_a$};
		\node[arg, label={right}:$d_{\overline a}$] (d2) at (2.5,0) {$d_{\overline a}$};
		\node[arg, label={left}:$b$] (b1) at (4,1) {$b$};
		\node[arg, label={right}:$\overline b$] (b2) at (5,1) {$\overline{b}$};
		\node[arg, label={left}:$a$] (a1p) at (0,2) {$a^*$};
		\node[arg, label={left}:$\overline a$] (a2p) at (0,1) {$\overline a^*$};
		\node[arg, label={above left}:$c_1$] (c1) at (2,2.5) {$c_1$};
		\node[arg, label={above right}:$c_2$] (c2) at (4.5,2.5) {$c_2$};
		\node[arg, label={left}:$\varphi$] (phi) at (3.25,4) {$\varphi$};
		\node[arg, label={right}:$\overline{\varphi}$] (negphi) at (6,4) {$\overline{\varphi}$};
		\draw[attack] 
		(a1) edge (d1)
		(d1) edge[loop below] (d1)
		(a2) edge (d2)
		(d2) edge[loop below] (d2)
		(a1) edge (a2)
		(a2) edge (a1)
		(b1) edge (b2)
		(b2) edge (b1)
		(a1) edge (c1)
		(b1) edge (c1)
		(b1) edge (c2)
		(a2) edge (c2)
		(c1) edge[loop left] (c1)
		(c2) edge[loop left] (c2)
		(c1) edge (phi)
		(c2) edge (phi)
		(phi) edge (negphi)
		(negphi) edge [loop above] (negphi)
		(negphi) edge [bend angle=15, bend left] (b1)
		(negphi) edge [bend angle=15, bend left] (b2)
		;
	}
	\caption{Reduction of the $\text{QBF}^2_\forall$ instance $\forall a \exists b ((a \lor b) \allowbreak \land \allowbreak (\neg a \lor b))$ to an instance of $\Ver_{\semi,1}^\PCAF$.}
	\label{fig:reductionVerficationImage1Semi}
\end{figure}

\begin{restatable}{lemma}{verificationStageFirstImage} \label{prop:verificationStageFirstImage}
	$\Ver_{\stage,1}^\PCAF$ is $\SigmaP{2}$-hard, even for transitive preferences.
\end{restatable}
\begin{proof}
	We provide a reduction from $\text{QBF}^2_\forall$ to the complementary problem: 
	For an instance $\Phi = \forall Y \exists Z \varphi(Y,Z)$ of $\text{QBF}^2_\forall$ where $\varphi$ is given by a set of clauses $C$ over atoms in $X = Y \cup Z$, we construct a CAF $F = (\Args,\Att,\cl)$ as well as a set of claims $S$: 
	\begin{itemize}
		\item $\Args = \{\overline{\varphi}\} \cup C \cup X \cup \overline{X} \cup Y^*\cup \overline Y^*\cup \{d_v\mid v\in Y\cup \overline Y\}$, where $\overline{X} = \{\overline x \mid x \in X\}$, $Y^*=\{y^*\mid y\in Y\}$, and $\overline Y^*=\{\overline y^*\mid y\in Y\}$;
		
		\item $\Att = \{(x,\overline x),(\overline x,x)\mid x \in X\}\ \cup$ \\
		\phantom{$\Att =\ $}%
		$\{(x,c) \mid x \in c, c \in C\} \cup \{(\overline{x},c) \mid \neg x \in c, c \in C\}\ \cup$
		\phantom{$\Att =\ $}%
		$\{(c,c)\mid c\in C\}\cup \{(d_v,d_v),(v,d_v)\mid v\in Y\cup \overline Y\}\ \cup $\\
		\phantom{$\Att =\ $}%
		$\{(\overline\varphi, c)\mid c\in C\}\cup  \{(z,\overline{\varphi}),(\overline z, \overline{\varphi}) \mid z \in Z \}$;
		\item $\cl(v^*)=v$ for $v\in Y^*\cup \overline Y^*$ and $\cl(v) = v$ else;
		\item $S = Y\cup \overline Y\cup \{\overline\varphi\}$. 
	\end{itemize}
	
	Figure~\ref{fig:reductionVerficationImage1Stage} illustrates the above construction. The constructed CAF is contained in $\ImageTrans{1}$ since all paths in $\problematicPart(F)=\{(v^*,v)\mid v\in Y\cup \overline Y\}$ are of length $1$ (only arguments in $Y^*\cup \overline Y^*$ have outgoing edges in $\problematicPart(F)$).
	It remains to verify the correctness of the reduction, i.e., we will show that $S\notin \stage\!_c(F)$ iff $\Phi$ is valid. The proof proceeds similar as the proof of Proposition~\ref{prop:verificationSemiFirstImage}.
	
	First assume $\Phi$ is valid. Towards a contradiction, assume $S\in\stage\!_c(F)$ and consider a $\stage$-realization $E$ of $S$. Clearly, $Y^*\cup \overline Y^*\subseteq E$; moreover, $E$ contains either $y$ or $\overline y$ for each $y\in Y$ (otherwise $E\cup \{y\}$ is a proper conflict-free superset of $E$, contradiction to $E$ being stage in $(\Args,\Att)$). Let $Y'=E\cap Y$. By assumption, there is $Z'\subseteq Z$ such that $M=Y'\cup Z'$ is a model of $\varphi$. Now, let $D=Y^*\cup \overline Y^*\cup M\cup \{\overline x\mid x\notin M\}$; clearly, $D$ is conflict-free by construction. 
	We will prove that $D\cup D^+\supset E\cup E^+$: First observe that $v\in D\cap E$ for every $v\in \Args$ with $\cl(v)\in Y\cup \overline Y$ ($Y^*\cup \overline Y^*\cup Y'\subseteq E\cap D$ by definition, and $\{\overline y\mid y\notin Y'\}\subseteq E$ since $E$ contains either $y$ or $\overline y$ for each $y\in Y$ as outlined above; notice that each other argument having claim $cl\in Y\cup \overline Y$ is attacked by $Y'\cup \{\overline y\mid y\notin Y'\}$ and thus cannot be contained in $D$ or in $E$). Moreover, we have that each $c\in C$ is attacked by $D$ since $M$ is a model of $\varphi$: Consider an arbitrary clause $c\in C$, then there is either $x\in M$ with $x\in c$ (in this case, $(x,c)\in \Att$), or there is some $x\in X$ with $x\notin M$ and $\neg x\in c$ (then $(\overline x,c)\in \Att$ and $\overline x\in D$). In both cases, we have that $D$ attacks $c$. Furthermore, $\overline\varphi\in D^+$ since for each $z\in Z$, either $z\in D$ or $\overline z\in D$ by definition of $D$. Thus $D\cup D^+\supset E\cup E^+$, contradiction to the assumption.
	
	Now assume $S\notin \stage\!_c(F)$ and consider an arbitrary subset $Y'\subseteq Y$. We show that there is $Z'\subseteq Z$ such that $Y'\cup Z'$ is a model of $\varphi$.  Let $E=Y'\cup \{\overline y\mid y\notin Y'\}\cup Y^*\cup \overline Y^*$ (and observe that $\cl(E)=S\setminus \{\overline\varphi\}$). By assumption we have that $E\cup \{\overline \varphi\}$ is not stage in $(\Args,\Att)$, that is, there  is some $D\in \cf((\Args,\Att))$ with $D\cup D^+\supset E\cup \{\overline\varphi\}\cup (E\cup \{\overline\varphi\})^+$. 
	We can assume that $D\supset E$ since each argument $d_v\in E^+$, $v\in Y\cup \overline Y$, has precisely one non-self-attacking attacker (namely the argument $v$) thus $Y'\cup \{\overline y\mid y\notin Y'\}\subseteq D$; moreover, $D$ contains each argument $v\in Y^*\cup \overline Y^*$ since each such argument is unattacked. 
	Moreover observe that $\overline\varphi\notin D$ since each set of the form $Y'\cup \{\overline y\mid y\notin Y'\}\cup Y^*\cup \overline Y^*\cup \{\overline\varphi\}$ is naive in $(\Args,\Att)$ (i.e., it is a maximal conflict-free set since $\overline\varphi$ attacks each literal over atoms in $Z$ as well as every clause-argument $c\in C$) and thus cannot be extended any further. It follows that there is some argument in $E\cup \{\overline\varphi\}$ which is contained in the range of $D$ rather than in $D$ - since $D\supset E$ as outlined above we conclude $\overline\varphi\notin D$ (and thus $\overline\varphi\in D^+$).
	Also, we have that $D$ attacks each clause-argument $c\in C$ since $c\in (E\cup \{\overline\varphi\})^+$, and, moreover, since each such argument is self-attacking.
	Now, let $Z'=Z\cap D$. We show that $M=Y'\cup Z'$ is a model of $\varphi$. Consider some arbitrary clause $c\in C$. Then there is some argument $v\in D$ such that $(v,c)\in \Att$. As outlined above, $v\neq \overline\varphi$ since $\overline\varphi$ is not contained in $D$. Consequently, we have $v\in X\cup \overline X$. In case $v\in X$ we have $v\in M\cap c$, in case $v\in \overline X$ we have $\neg v\in c$ by definition of $\Att$ and $v\notin M$ - in every case, the clause $c$ is satisfied by $M$. As $c$ was chosen arbitrary it follows that $\varphi$ is satisfiable and thus $\Phi$ is valid. 
\end{proof}
\begin{figure}[ht]
	\centering
	\tikz{
		\node[arg, label={left}:$a$] (a1) at (1.5,1) {$a$};
		\node[arg, label={right}:$\overline a$] (a2) at (2.5,1) {$\overline{a}$};
		\node[arg, label={left}:$d_{a}$] (d1) at (1.5,0) {$d_a$};
		\node[arg, label={right}:$d_{\overline a}$] (d2) at (2.5,0) {$d_{\overline a}$};
		\node[arg, label={left}:$b$] (b1) at (4,1) {$b$};
		\node[arg, label={right}:$\overline b$] (b2) at (5,1) {$\overline{b}$};
		\node[arg, label={left}:$a$] (a1p) at (0,2) {$a^*$};
		\node[arg, label={left}:$\overline a$] (a2p) at (0,1) {$\overline a^*$};
		\node[arg, label={above}:$c_1$] (c1) at (2,2.5) {$c_1$};
		\node[arg, label={above}:$c_2$] (c2) at (4.5,2.5) {$c_2$};
		\node[arg, label={right}:$\overline{\varphi}$] (negphi) at (6,4) {$\overline{\varphi}$};
		\draw[attack] 
		(a1) edge (d1)
		(d1) edge[loop below] (d1)
		(a2) edge (d2)
		(d2) edge[loop below] (d2)
		(a1) edge (a2)
		(a2) edge (a1)
		(b1) edge (b2)
		(b2) edge (b1)
		(a1) edge (c1)
		(b1) edge (c1)
		(b1) edge (c2)
		(a2) edge (c2)
		(c1) edge[loop left] (c1)
		(c2) edge[loop left] (c2)
		(negphi) edge (c1)
		(negphi) edge (c2)
		(b1) edge [bend angle=15, bend right] (negphi)
		(b2) edge [bend angle=15, bend right] (negphi)
		;
	}
	\caption{Reduction of the $\text{QBF}^2_\forall$ instance $\forall a \exists b ((a \lor b) \allowbreak \land \allowbreak (\neg a \lor b))$ to an instance of $\Ver_{\stage,1}^\PCAF$.}
	\label{fig:reductionVerficationImage1Stage}
\end{figure}

\verificationConflictfreeNaiveImagesTwoThreeFour*
\begin{proof}
	Let $F = (\Args, \Att, \cl, \succ)$ be a PCAF, $C$ a set of claims, and $i \in \{2,3,4\}$. To check whether $C \in \cf_p^i(F)$, by Lemma~\ref{lemma:cfDoesNotChangeForReductions234}, it suffices to check whether $C \in \cf\!_c((\Args, \Att, \cl))$ which can be done in polynomial time on well-formed CAFs (cf.\ Table~\ref{tab:complexity_CAFwfCAF}). Analogously, to check $C \in \naive_p^i(F)$ it suffices to check if $C \in \naive_c((\Args, \Att, \cl))$, which can also be done in polynomial time on well-formed CAFs.
\end{proof}

\conflictFreeLemma*
\begin{proof}
	Let $F = (\Args, \Att, \cl, \succ)$ be a PCAF. From \cite[Lemma~$1$]{DvorakW20} we know that $C \in \cf\!_c((\Args, \Att, \cl))$ iff $\cl(E_1^i(C)) = C$, as well as that, if $C \in \cf\!_c((\Args, \Att, \cl))$ then $E_1^i(C)$ is the unique maximal conflict-free set $S$ in $(\Args, \Att, \cl)$ with $\cl(S) = C$. From this and our Lemma~\ref{lemma:cfDoesNotChangeForReductions234}, our result follows immediately.
\end{proof}

\admissibleLemma*
\begin{proof}
	Let $F = (\Args,\Att,\cl,\succ)$ be a PCAF, $C$ a set of claims, and $i \in \{2,3,4\}$.
	
	Assume $\cl(E_*^i(C)) = C$. By construction, $E_*^i(C) \in \adm(\red{i}(F))$, and thus $C \in \adm_p^i(F)$. 
	
	Now assume $C \in \adm_p^i(F)$. Then there exists $S \subseteq \Args$ such that $\cl(S) = C$ and $S \in \adm(\red{i}(F))$. Furthermore, $C \in \cf_p^i(F)$ and therefore, by Lemma~\ref{lemma:Im234ConflictFree}, $S \subseteq E_1^i(C)$. By construction, $E_*^i(C) \subseteq E_1^i(C)$. Moreover, any $x \in S$ is defended by $S$ in $\red{i}(F)$ and therefore also by $E_1^i(C)$. Thus, by definition, $x \in E_2^i(C)$. By the same argument, if $x \in S$ and $x \in E_k^i(C)$ then $x \in E_{k+1}^i(C)$. We can conclude that $S \subseteq E_*^i(C) \subseteq E_1^i(C)$ and thus $\cl(E_*^i(C)) = C$. 
	By the above we have that $E_*^i(C)$ is admissible and each $S \subseteq \Args$ such that $\cl(S) = C$ is a subset of  $E_*^i(C)$. In other words $E_*^i(C)$ is the unique maximal admissible set $S$ in $\red{i}(F)$ such that $\cl(S) = C$.
\end{proof}

\verificationStableAdmissibleCompleteImagesTwoThreeFour*

We split the proof into the following two Lemmas:

\begin{lemma}
	$\Ver_{\sigma,i \in \{2,3,4\}}^\PCAF$ is in $\P$ for $\sigma \in \{\adm, \stable\}$. 
\end{lemma}
\begin{proof}
	Let $F = (\Args,\Att,\cl, \succ)$ be a PCAF, let $C$ be a set of claims, and let $i \in \{2,3,4\}$. We can compute $\red{i}(F)$, $E_1^i(C)$, and $E_*^i(C)$ in polynomial time. 
	
	For $\adm$, by Lemma~\ref{lemma:Im234Admissible}, it suffices to test whether $\cl(E_*^i(C)) = C$. 
	
	For $\stb$, we first check whether $C \in \adm_p^i(F)$. If not, ${C \not\in \stb_p^i(F)}$. If yes, then, by Lemma~\ref{lemma:Im234Admissible}, ${\cl(E_*^i(C)) = C}$. We can check in polynomial time if $E_*^i(C) \in \stb(\red{i}(F))$. If yes, we are done. If no, then there is an argument $x$ that is not in $E_*^i(C)$ but is also not attacked by $E_*^i(C)$ in $\red{i}(F)$. Moreover, there can be no other $S \in \stb(\red{i}(F))$ with $\cl(S) = C$ since for any such $S$ we would have $S \subseteq E_*^i(C)$, which would imply that $S$ does not attack $x$ and that $x \not \in S$. 
\end{proof}

\begin{lemma}
	$\Ver_{\sigma,i \in \{2,3,4\}}^\PCAF$ is $\coNP$-complete for $\sigma \in \{\pref, \semi, \stage\}$, even for transitive preferences.
\end{lemma}
\begin{proof}
	Let $F = (\Args,\Att,\cl, \succ)$ be a PCAF, let $C$ be a set of claims, and let $i \in \{2,3,4\}$.
	
	$\coNP$-hardness for $\sigma \in \{\pref, \semi, \stage\}$ follows from known properties of wfCAFs (see Table~\ref{tab:complexity_CAFwfCAF}). For membership regarding $\pref$, we first check if $C \in \adm_p^i(F)$. If not, $C \not\in \pref_p^i(F)$. If yes, then $\cl(E_*^i(C)) = C$. We can check in $\coNP$ if $E_*^i(C) \in \pref(\red{i}(F))$. If yes, we are done. If no, then, by Lemma~\ref{lemma:Im234Admissible}, there can be no other $S \in \pref(\red{i}(F))$ with $\cl(S) = C$, since $S$ would have to be admissible in $\red{i}(F)$ and therefore $S \subseteq E_*^i(C)$.
	
	For $\semi$, we also first check if $C \in \adm_p^i(F)$. If not, $C \not\in \semi_p^i(F)$. If yes, then $\cl(E_*^i(C)) = C$. We can check in $\coNP$ if $E_*^i(C) \in \semi(\red{i}(F))$. If yes, we are done. If no, then, by Lemma~\ref{lemma:Im234Admissible}, there can be no other $S \in \pref(\red{i}(F))$ with $\cl(S) = C$, since $S$ would have to be admissible in $\red{i}(F)$
	and therefore $S \subseteq E_*^i(C)$.
	As $\semi(\red{i}(F)) \subseteq \pref(\red{i}(F))$ we obtain that there is no other $S \in \semi(\red{i}(F))$ with $\cl(S) = C$.
	
	For $\stage$, we first check if $C \in \naive_p^i(F)$. If not, $C \not\in \stage_p^i(F)$. If yes, then $\cl(E_1^i(C)) = C$. We can check in $\coNP$ if $E_1^i(C) \in \stage(\red{i}(F))$. If yes, we are done. If no, then, by Lemma~\ref{lemma:Im234ConflictFree}, there can be no other $S \in \naive(\red{i}(F))$ with $\cl(S) = C$, since $S$ would have to be conflict-free in $\red{i}(F)$and therefore $S \subseteq E_1^i(C)$.
	As $\stage(\red{i}(F)) \subseteq \naive(\red{i}(F))$ we obtain that there is no other $S \in \stage(\red{i}(F))$ with $\cl(S) = C$.
\end{proof}

\verificationCompleteImagesTwoThreeFour*
\begin{proof}
	$\P$-membership of $\Ver_{\comp,3}^\PCAF$ is shown below (Lemma~\ref{prop:verificationCompleteThirdImage}). 
	$\NP$-membership follows from results for general CAFs (cf.\ Table~\ref{tab:complexity_CAFwfCAF}). We show $\NP$-hardness of $\Ver_{\comp,2}^\PCAF$ and $\Ver_{\comp,4}^\PCAF$ separately below (Lemmas~\ref{prop:verificationCompleteSecondImage} and~\ref{prop:verificationCompleteFourthImage}).
\end{proof}

\begin{lemma} \label{prop:verificationCompleteThirdImage}
	$\Ver_{\comp,3}^\PCAF \in \P$.
\end{lemma}
\begin{proof}	
	Let $F = (\Args,\Att,\cl, \succ)$ be a PCAF and let $C$ be a set of claims.
	For $\comp$ and $i = 3$, we first check if $C \in \adm_p^i(F)$. If not, ${C \not\in \comp_p^i(F)}$. If yes, then $\cl(E_*^i(C)) = C$. We can check in polynomial time if $E_*^i(C) \in \comp(\red{i}(F))$. If no, then $E_*^i(C)$ defends some $x \not\in E_*^i(C)$ in $\red{i}(F)$. Towards a contradiction, assume there is some $S \subseteq \Args$ such that $S \in \comp(\red{i}(F))$ and $\cl(S) = C$. By Lemma~\ref{lemma:Im234Admissible}, $S \subseteq E_*^i(C)$, which implies $x \not\in S$. Then $S$ can not defend $x$ in $\red{i}(F)$, i.e., there must be $y$ and $z$ such that $y \in E_*^i(C)$, $y \not\in S$, $(z,x) \in \red{i}(F)$, and $(y,z) \in \red{i}(F)$. Then also $(y,z) \in F$ by the definition of Reduction~$3$. But there must also be some $y' \in S$ with $\cl(y') = \cl(y)$, and since the underlying CAF of $F$ is well-formed there must be $(y',z) \in F$. Since there cannot be $(y',z) \in \red{i}(F)$, otherwise $S$ would defend $x$, it has to be that $z \succ y'$. For Reduction~$3$ this further requires $(z,y') \in F$. Crucially, $(z,y') \in \red{i}(F)$. But then $S$ must be defended from $z$, i.e., there must be some $w \in S$ such that $(w,z) \in \red{i}(F)$. But this means that $S$ defends $x$, i.e., $S$ is not complete. Contradiction.
\end{proof}

\begin{restatable}{lemma}{verificationCompleteSecondImage} \label{prop:verificationCompleteSecondImage}
	$\Ver_{\comp,2}^\PCAF$ is $\NP$-hard, even for transitive preferences.
\end{restatable}
\begin{proof}
	Let $\varphi$ be an arbitrary instance of 3-\SAT\ given as a set $C$ of clauses over variables $X$ 
	and let $\overline{X} = \{\overline x \mid x \in X\}$.
	We construct a PCAF $F = (\Args,\Att,\cl, \succ)$ as well as a set of claims $S$: 
	\begin{itemize}
		\item $\Args = \{\varphi\} \cup C \cup X \cup \overline{X} \cup \{d_x \mid x\in X\}\ \cup $\\ 
		\phantom{$\Args =\ $}%
		$\{d_x^j \mid x\in X \cup \overline X, 1 \leq j \leq 4\}$;
		\item $\Att = \{(c,\varphi)\mid c\in C\} \cup \{(c,c)\mid c\in C\}\ \cup$ \\
		\phantom{$\Att =\ $}%
		$\{(c,x) \mid x \in c, c \in C\}\ \cup$\\ 
		\phantom{$\Att =\ $}%
		$\{(c,\overline{x}) \mid \neg x \in c, c \in C\}\ \cup $\\
		\phantom{$\Att =\ $}%
		$\{(d_x^1,x),(d_x^1,d_x^2),(d_x^3,d_x^2),(d_x^3,x),(d_x^4,x)  \mid$\\ 
		\phantom{$\Att =\ $}%
		$x \in X \cup \overline{X}\}\ \cup$ \\
		\phantom{$\Att =\ $}%
		$\{(d_x^4,d_x),(d_{\overline{x}}^4,d_x)\mid x \in X \}$;
		\item $\cl(x)=\cl(\overline{x})=\cl(d_x^2)=\cl(d_{\overline{x}}^2)= x$ for $x \in X$, $\cl(v) = v$ else;
		\item $x \succ c$, $x \succ d_x^1$, $x \succ d_x^4$, $d_x^2 \succ d_x^3$ for all $x \in X \cup \overline{X}$ and all $c \in C$; 
		\item $S = X \cup \{\varphi\}$. 
	\end{itemize}
	Figure~\ref{fig:reductionVerficationImage2CompleteAppendix} illustrates the above construction. It remains to show that $\varphi$ is satisfiable if and only if $S \in \comp_c(\red{2}(F))$.
	
	\begin{figure}[ht]
		\centering
		\tikz{
			\node[arg, label={left}:$a$] (a1) at (1.8,0) {$a$};
			\node[arg, label={right}:$a$] (a2) at (2.7,0) {$\overline{a}$};
			\node[arg, label={above}:$d_{a}$] (da1) at (0,0.9) {$d_a$};
			\node[arg, label={below}:$d_{a}^4$] (da14) at (0.9,0.5) {$d_a^4$};
			\node[arg, label={above}:$d_{\overline a}^4$] (da24) at (0.9,1.15) {$d_{\overline a}^4$};
			\node[arg, label={left}:$d_{a}^1$] (da11) at (1.8,-1) {$d_a^1$};
			\node[arg, label={left}:$a$] (da12) at (1.8,-2) {$d_a^2$};
			\node[arg, label={left}:$d_{a}^3$] (da13) at (1.8,-3) {$d_a^3$};
			\node[arg, label={right}:$d_{\overline a}^1$] (da21) at (2.7,-1) {$d_{\overline a}^1$};
			\node[arg, label={right}:$a$] (da22) at (2.7,-2) {$d_{\overline a}^2$};
			\node[arg, label={right}:$d_{\overline a}^3$] (da23) at (2.7,-3) {$d_{\overline a}^3$};
			
			\node[arg, label={left}:$b$] (b1) at (4.5,0) {$b$};
			\node[arg, label={right}:$b$] (b2) at (5.4,0) {$\overline{b}$};
			\node[arg, label={above}:$d_{b}$] (db1) at (7.2,1) {$d_b$};
			\node[arg, label={above}:$d_{b}^4$] (db14) at (6.3,1.15) {$d_b^4$};
			\node[arg, label={below}:$d_{\overline b}^4$] (db24) at (6.3,0.5) {$d_{\overline b}^4$};
			\node[arg, label={left}:$d_{b}^1$] (db11) at (4.5,-1) {$d_b^1$};
			\node[arg, label={left}:$b$] (db12) at (4.5,-2) {$d_b^2$};
			\node[arg, label={left}:$d_{b}^3$] (db13) at (4.5,-3) {$d_b^3$};
			\node[arg, label={right}:$d_{\overline b}^1$] (db21) at (5.4,-1) {$d_{\overline b}^1$};
			\node[arg, label={right}:$b$] (db22) at (5.4,-2) {$d_{\overline b}^2$};
			\node[arg, label={right}:$d_{\overline b}^3$] (db23) at (5.4,-3) {$d_{\overline b}^3$};
			
			\node[arg, label={above}:$c_1$] (c1) at (2.25,1.5) {$c_1$};
			\node[arg, label={above}:$c_2$] (c2) at (4.95,1.5) {$c_2$};
			\node[arg, label={left}:$\varphi$] (phi) at (3.6,3) {$\varphi$};
			\draw[attack] 
			
			(a1) edge[color=gray, ultra thick] (c1)
			(b1) edge[color=gray, ultra thick]  (c1)
			(b2) edge[color=gray, ultra thick] (c2)
			(a2) edge[color=gray, ultra thick]  (c2)
			(c1) edge[loop left] (c1)
			(c2) edge[loop left] (c2)
			(c1) edge (phi)
			(c2) edge (phi)
			
			(a1) edge[color=gray, ultra thick] (da14)
			(a2) edge[color=gray, ultra thick] (da24)
			(da14) edge (da1)
			(da24) edge (da1)
			(a1) edge [color=gray, ultra thick] (da11) 
			(da11) edge (da12) 
			(da12) edge [color=gray, ultra thick] (da13) 
			(da13) edge [bend angle=55,bend left] (a1) 
			(a2) edge[color=gray, ultra thick] (da21) 
			(da21) edge (da22) 
			(da22) edge [color=gray, ultra thick] (da23) 
			(da23) edge [bend angle=55,bend right] (a2) 
			
			(b1) edge[color=gray, ultra thick]  (db14)
			(b2) edge[color=gray, ultra thick] (db24)
			(db14) edge (db1)
			(db24) edge (db1)
			(b1) edge[color=gray, ultra thick] (db11) 
			(db11) edge (db12) 
			(db12) edge [color=gray, ultra thick]  (db13) 
			(db13) edge [bend angle=55,bend left] (b1) 
			(b2) edge [color=gray, ultra thick] (db21) 
			(db21) edge (db22) 
			(db22)edge [color=gray, ultra thick]  (db23) 
			(db23) edge [bend angle=55,bend right] (b2) 
			;
		}
		\caption{$\red{2}(F)$ from the proof of Lemma~\ref{prop:verificationCompleteSecondImage}, with $\varphi =  ((a \lor b) \allowbreak \land \allowbreak (\neg a \lor \neg b))$. Attacks drawn in gray/thick have been reversed by Reduction~$2$.}
		\label{fig:reductionVerficationImage2CompleteAppendix}
	\end{figure}
	
	Assume $\varphi$ is satisfiable. Then there is an interpretation $I$ such that $I \models \varphi$. Let $E = \{x,d_x^2 \mid x \in X, x \in I\} \cup \{\overline{x},d_{\overline{x}}^2 \mid x \in X, x \not\in I\} \cup \{\varphi\}$. Clearly, $\cl(E) = S$. Furthermore, $E$ defends $\varphi$ in $\red{2}(F)$ since each clause is satisfied by $I$, and thus each clause argument $c_j$ is attacked by some $x$ (or $\overline{x}$) in $E$. For each variable $x$, if $x \in E$, then $x$ defends $d_x^2$ and $d_x^2$ defends $x$. Moreover, if $x \in E$, then $\overline{x} \not\in E$ and none of $d_{x}$, $\overline{x}$, or $d_{\overline{x}}^j$ with $1 \leq j \leq 4$ is defended by $E$. Analogously for the case that $\overline{x} \in E$. Thus, $E$ is admissible, and contains all arguments it defends, i.e., $E \in \comp(\red{2}(F))$.
	
	Assume $S \in \comp_c(\red{2}(F))$. Then there is $E \subseteq A$ such that $\cl(E) = S$ and $E \in \comp(\red{2}(F))$. For each $x \in X$, at least one of $x, \overline{x}, d_x^2, d_{\overline{x}}^2$ must be contained in $E$. In fact, if $x \in E$, then also $d_x^2 \in E$ and vice versa. Analogous for $\overline{x}$ and $d_{\overline{x}}^2$. However, it can not be that $x \in E$ and $\overline{x} \in E$, otherwise $d_x$ would be defended by $E$ and we would have $\cl(E) \neq S$. Thus, for each $x \in X$, there is either $x \in E$ or $\overline{x} \in E$, but not both. Furthermore, $E$ defends $\varphi$, i.e., $E$ attacks all clause arguments $c_j$. Therefore, $I \models \varphi$ for $I = X \cap E$. %
\end{proof}

\begin{restatable}{lemma}{verificationCompleteFourthImage} \label{prop:verificationCompleteFourthImage}
	$\Ver_{\comp,4}^\PCAF$ is $\NP$-hard, even for transitive preferences.
\end{restatable}
\begin{proof}
	This proof is similar to that of Lemma~\ref{prop:verificationCompleteSecondImage}. Let $\varphi$ be an arbitrary instance of 3-\SAT\ given as a set $C$ of clauses over variables $X$
	and let $\overline{X} = \{\overline x \mid x \in X\}$. 
	We construct a PCAF $F = (\Args,\Att,\cl, \succ)$ as well as a set of claims $S$: 
	\begin{itemize}
		\item $\Args = \{\varphi\} \cup C \cup X \cup \overline{X} \cup \{d_x \mid x\in X\} \cup \{d_x^1 \mid x\in X \cup \overline X\}$;
		\item $\Att =  \{(c,\varphi)\mid c\in C\} \cup \{(c,c)\mid c\in C\}\ \cup$\\ 
		\phantom{$\Att =\ $}%
		$\{(c,x) \mid x \in c, c \in C\}\ \cup$\\
		\phantom{$\Att =\ $}%
		$\{(c,\overline{x}) \mid \neg x \in c, c \in C\}\ \cup$\\
		\phantom{$\Att =\ $}%
		$\{(d_x^1,x) \mid x \in X \cup \overline{X}\}\ \cup$\\
		\phantom{$\Att =\ $}%
		$\{(d_x^1,d_x),(d_{\overline{x}}^1,d_x)\mid x \in X \}$;
		\item $\cl(x)=\cl(\overline{x})= x$ for $x \in X$,\\
		$\cl(v) = v$ otherwise;
		\item $x \succ c$, $x \succ d_x^1$ for all $x \in X \cup \overline{X}$ and all $c \in C$; 
		\item $S = X \cup \{\varphi\}$. 
	\end{itemize}
	Figure~\ref{fig:reductionVerficationImage4CompleteAppendix} illustrates the above construction. 
	It remains to show that $\varphi$ is satisfiable if and only if $S \in \comp_c(\red{4}(F))$.
	
	\begin{figure}[ht]
		\centering
		\tikz{
			\node[arg, label={below}:$a$] (a1) at (1.8,0) {$a$};
			\node[arg, label={below}:$a$] (a2) at (2.7,0) {$\overline{a}$};
			\node[arg, label={above}:$d_{a}$] (da1) at (0,1) {$d_a$};
			\node[arg, label={below}:$d_{a}^1$] (da12) at (0.9,0.5) {$d_a^1$};
			\node[arg, label={above}:$d_{\overline a}^1$] (da22) at (0.9,1.15) {$d_{\overline a}^1$};
			
			\node[arg, label={below}:$b$] (b1) at (4.5,0) {$b$};
			\node[arg, label={below}:$b$] (b2) at (5.4,0) {$\overline{b}$};
			\node[arg, label={above}:$d_{b}$] (db1) at (7.2,1) {$d_b$};
			\node[arg, label={above}:$d_{b}^1$] (db12) at (6.3,1.15) {$d_b^1$};
			\node[arg, label={below}:$d_{\overline b}^1$] (db22) at (6.3,0.5) {$d_{\overline b}^1$};

			\node[arg, label={above}:$c_1$] (c1) at (2.25,1.5) {$c_1$};
			\node[arg, label={above}:$c_2$] (c2) at (4.95,1.5) {$c_2$};
			\node[arg, label={left}:$\varphi$] (phi) at (3.6,3) {$\varphi$};
			\draw[attack] 
			
			(a1) edge[color=gray, ultra thick] (c1)
			(c1) edge[color=gray, ultra thick] (a1)
			(b1) edge[color=gray, ultra thick]  (c1)
			(c1) edge[color=gray, ultra thick]  (b1)
			(b2) edge[color=gray, ultra thick] (c2)
			(c2) edge[color=gray, ultra thick] (b2)
			(a2) edge[color=gray, ultra thick]  (c2)
			(c2) edge[color=gray, ultra thick]  (a2)
			(c1) edge[loop left] (c1)
			(c2) edge[loop left] (c2)
			(c1) edge (phi)
			(c2) edge (phi)
			
			(a1) edge[color=gray, ultra thick] (da12)
			(da12) edge[color=gray, ultra thick] (a1)
			(a2) edge[color=gray, ultra thick] (da22)
			(da22) edge[color=gray, ultra thick] (a2)
			(da12) edge (da1)
			(da22) edge (da1)
			
			(b1) edge[color=gray, ultra thick]  (db12)
			(db12) edge[color=gray, ultra thick]  (b1)
			(b2) edge[color=gray, ultra thick] (db22)
			(db22) edge[color=gray, ultra thick] (b2)
			(db12) edge (db1)
			(db22) edge (db1)
			;
		}
		\caption{$\red{4}(F)$ from the proof of Lemma~\ref{prop:verificationCompleteFourthImage}, with $\varphi =  ((a \lor b) \allowbreak \land \allowbreak (\neg a \lor \neg b))$. Symmetric attacks drawn in gray/thick have been introduced by Reduction~$4$.}
		\label{fig:reductionVerficationImage4CompleteAppendix}
	\end{figure}
	
	Assume $\varphi$ is satisfiable. Then there is an interpretation $I$ such that $I \models \varphi$. Let $E = \{x\mid x \in X, x \in I\} \cup \{\overline{x}\mid x \in X, x \not\in I\} \cup \{\varphi\}$. Clearly, $\cl(E) = S$. Furthermore, $E$ defends $\varphi$ in $\red{4}(F)$ since each clause is satisfied by $I$, and thus each clause argument $c_j$ is attacked by some $x$ (or $\overline{x}$) in $E$. Each variable $x \in X$ clearly defends itself. Moreover, if $x \in E$, then $\overline{x} \not\in E$ and none of $d_x$, $\overline{x}$, or $d_{\overline{x}}^1$ is defended by $E$. Analogously for the case that $\overline{x} \in E$. Thus, $E$ is admissible, and contains all arguments it defends, i.e., $E \in \comp(\red{4}(F))$.
	
	Assume $S \in \comp_c(\red{4}(F))$. Then there is $E \subseteq A$ such that $\cl(E) = S$ and $E \in \comp(\red{4}(F))$. For each $x \in X$, at least one of $x, \overline{x}$ must be contained in $E$. In fact, it can not be that $x \in E$ and $\overline{x} \in E$, otherwise $d_x$ would be defended by $E$ and we would have $\cl(E) \neq S$. Thus, for each $x \in X$, there is either $x \in E$ or $\overline{x} \in E$, but not both. Furthermore, $E$ defends $\varphi$, i.e., $E$ attacks all clause arguments $c_j$. Thus, $I \models \varphi$ for $I = X \cap E$. %
\end{proof}

\section{Efficient enumeration of claim-extensions}
In this section, we briefly discuss how our findings can be used for efficient enumeration of claim-extensions. 
For CAFs obtainable by Reduction 2, 3, and 4, we present an fixed-parameter tractable (FPT) algorithm for enumerating extensions that scales exponential with the number $k$ of different claims in the given CAF but only polynomial in its size $n$.

\begin{proposition}
For $\sigma\in\{\cf,\allowbreak\adm,\allowbreak\naive,\allowbreak\stb,\allowbreak\pref,\allowbreak\semi,\allowbreak\stage\}$, $i\in\{2,3,4\}$, and for $\sigma=\comp$ in case $i=3$,
there is a polynomial $poly(\cdot)$ such that
enumerating all $\sigma$-extensions of a CAF $(A,R,\cl)\in \Image{i}$ can be done in time in 
$\mathcal{O}(4^k\cdot poly(n))$, with $|\Args|= n$, $|\cl(\Args)|= k$.
\end{proposition}
\begin{proof}	
First recall that for CAFs $F \in \Image{i}$, $i\in \{2,3,4\}$ for the semantics $\sigma\in\{\cf,\adm,\naive,\stb\}$ as well as for complete semantics for CAFs in $\Image{3}$, the verification problem is in $\P$ (cf.\ Table~\ref{table:complexityResults}). 
Thus iterating through all $2^k$ many sets $C\subseteq \cl(\Args)$ and checking whether $C\in \sigma(F)$ can be done in time $\mathcal{O}(2^k\cdot poly(n))$.

For the remaining semantics, the algorithm builds heavily on the existence and polynomial-time computability of unique maximal realizations (i.e., sets of arguments $E\subseteq \Args$ with $\cl(E)=C$ for a given claim-set $C$) for conflict-free and admissible claim-sets:
Using Lemmata~\ref{lemma:Im234ConflictFree} and~\ref{lemma:Im234Admissible}, 
it suffices to compute $E^i_1(C)$ resp.\ $E^i_*(C)$ to verify conflict-freeness resp.\ admissibility of a claim-set $C$; moreover, $E^i_1(C)$ resp.\ $E^i_*(C)$ is the unique maximal conflict-free resp.\ admissible realization of $C$ in $F$.
Computing all conflict-free resp.\ admissible claim-sets along with their unique $\subseteq$-maximal realizations lies thus in $\mathcal{O}(2^k\cdot poly(n))$  as it suffices to loop through all sets $C\subseteq \cl(\Args)$. 

For $\sigma=\pref$, we compute all preferred argument-extensions of the underyling AF $(\Args,\Att)$. This requires %
two loops over all unique maximal admissible realizations of the admissible claim-sets of $F$; the overall runtime lies thus in $\mathcal{O}((2^k)^2\cdot poly(n))=\mathcal{O}(4^k\cdot poly(n))$.
We note that for $i=3$, we obtain a slightly improved runtime by exploiting I-maximality: When iterating through all $C\subseteq \cl(\Args)$, consider larger sets first; if an admissible set $C$ is found, exclude all subsets from further inspection. This procedure runs in time in $\mathcal{O}(2^k\cdot poly(n))$.

For $\sigma\in\{\stage,\semi\}$ we compute the range for the extensions by adding all attacked arguments to $E_1^i(C)$ resp.\ $E^i_*(C)$. Finally, we eliminate all sets for which the range is not $\subseteq$-maximal (analogous to the case for preferred semantics). This algorithm runs in $\mathcal{O}(4^k\cdot poly(n))$.
\end{proof}

As a consequence, we obtain FPT procedures for deciding credulous and skeptical acceptance as well as for verification of extensions for the respective CAF classes.

\begin{proposition}
For $\sigma\in\{\cf,\adm,\naive,\stb,\pref,\semi,$ $\stage\}$, $i\!\in\!\{2,3,4\}$, and for $\sigma\!=\!\comp$ in case $i\!=\!3$,
$\Cred_{\sigma,i}^\PCAF$, $\Skept_{\sigma,i}^\PCAF$, and $\Ver_{\sigma,i}^\PCAF$ can be solved in time in $\mathcal{O}(4^k\cdot poly(n))$ for CAFs $(\Args,\Att,\cl)$ with $|\cl(\Args)|\leq k$.
\end{proposition}

\end{document}